\documentclass[ejsv2,noshowframe,preprint]{imsart}

\RequirePackage[numbers]{natbib}
\RequirePackage[colorlinks,citecolor=blue,urlcolor=blue]{hyperref}
\RequirePackage{graphicx}

\startlocaldefs

\usepackage{mathrsfs}

\usepackage{accents}

\usepackage{multirow}
\usepackage[normalem]{ulem}
\usepackage[linesnumbered,lined,boxed]{algorithm2e}

\makeatletter
\newcommand{\mylabel}[2]{#2\def\@currentlabel{#2}\label{#1}}
\makeatother
\usepackage{amsmath,amsfonts,amsbsy,latexsym,tikz,verbatim}
\usepackage{hyperref}
\usepackage{graphicx,color}
\usepackage{caption}
\usepackage{subcaption}
\graphicspath{} 

\usepackage{enumerate}

\newcommand{\Cov}[0]{\text{Cov}}
\newcommand{\Var}[0]{\text{Var}}

\newcommand{\Pro}{\mathbb{P}}
\newcommand{\Exp}{\mathbb{E}}

\renewcommand{\leq}{\leqslant} 
\renewcommand{\geq}{\geqslant}

\def\qed{ \hfill $\blacksquare$}  

\newcommand{\cA}{\mathcal{A}}\newcommand{\cB}{\mathcal{B}}\newcommand{\cC}{\mathcal{C}}
\newcommand{\cE}{\mathcal{E}}\newcommand{\cF}{\mathcal{F}}
\newcommand{\cG}{\mathcal{G}}\newcommand{\cH}{\mathcal{H}}\newcommand{\cI}{\mathcal{I}}

\newcommand{\cQ}{\mathcal{Q}}
\newcommand{\cS}{\mathcal{S}}\newcommand{\cU}{\mathcal{U}}

\newcommand{\cZ}{\mathcal{Z}}  

\newcommand{\mvD}{\boldsymbol{D}}

\newcommand{\mvU}{\boldsymbol{U}}
\newcommand{\mvX}{\boldsymbol{X}}
\newcommand{\mvZ}{\boldsymbol{Z}}\newcommand{\mva}{\boldsymbol{a}}

\newcommand{\mve}{\boldsymbol{e}}

\newcommand{\mvu}{\boldsymbol{u}}\newcommand{\mvv}{\boldsymbol{v}}
\newcommand{\mvw}{\boldsymbol{w}}\newcommand{\mvx}{\boldsymbol{x}}
\newcommand{\mvz}{\boldsymbol{z}}

\newcommand{\mvgamma}{\boldsymbol{\gamma}}
\newcommand{\mvDelta}{\boldsymbol{\Delta}}
\newcommand{\mvtheta}{\boldsymbol{\theta}}\newcommand{\mvTheta}{\boldsymbol{\Theta}}

\newcommand{\mvmu}{\boldsymbol{\mu}}

\newcommand{\mvpsi}{\boldsymbol{\psi}}\newcommand{\mvPsi}{\boldsymbol{\Psi}}   

\newcommand{\bI}{\mathbb{I}}

\newcommand{\bN}{\mathbb{N}}
\newcommand{\bP}{\mathbb{P}}\newcommand{\bR}{\mathbb{R}}
\newcommand{\bU}{\mathbb{U}}
\newcommand{\bV}{\mathbb{V}}

\usepackage{enumitem}

\theoremstyle{plain}

\newtheorem{theorem}{Theorem}[section]
\newtheorem{lemma}[theorem]{Lemma}
\newtheorem{corollary}[theorem]{Corollary}

\theoremstyle{remark}

\newtheorem{remark}{Remark}

\endlocaldefs

\begin{document}
\begin{frontmatter}
\title{Truncated LinUCB for Stochastic Linear Bandits}
\runtitle{Truncated LinUCB  for Stochastic Linear Bandits}

\begin{aug}
\author[A]{\fnms{Yanglei}~\snm{Song}\ead[label=e1]{yanglei.song@queensu.ca}}
\and
\author[B]{\fnms{Meng}~\snm{Zhou}\ead[label=e3]{simon.zhou@queensu.ca}}
\address[A]{Department of Mathematics and Statistics, Queen's University \printead[presep={\ }]{e1}}

\address[B]{School of Computing and Department of Mathematics and Statistics, Queen's University\printead[presep={,\ }]{e3}}
\runauthor{Y.~Song and M.~Zhou}
\end{aug}

\begin{abstract}
This paper considers contextual bandits with a finite number of arms, where the contexts are independent and identically distributed $d$-dimensional random vectors, and the expected rewards are linear in both the arm parameters and contexts. The LinUCB algorithm, which is near minimax optimal for related linear bandits, is shown to have a cumulative regret that is suboptimal in both the dimension $d$ and time horizon $T$, due to its over-exploration. A truncated version of LinUCB is proposed and termed ``Tr-LinUCB", which follows LinUCB up to a truncation time $S$ and performs pure exploitation afterwards. The Tr-LinUCB algorithm is shown to achieve $O(d\log(T))$ regret if $S = Cd\log(T)$ for a sufficiently large constant $C$, and a matching lower bound is established, which shows the rate optimality of Tr-LinUCB in both $d$ and $T$ under a low dimensional regime. Further, if $S = d\log^{\kappa}(T)$ for $\kappa>1$, the loss compared to the optimal is an extra $\log\log(T)$ factor, which does not depend on $d$. This insensitivity to overshooting in choosing the truncation time of Tr-LinUCB is of practical importance.
\end{abstract}

\begin{keyword}[class=MSC]
\kwd[Primary ]{62L10}
\kwd[; secondary ]{62L12}
\end{keyword}

\begin{keyword}
\kwd{Stochastic linear bandits}
\kwd{Upper confidence bounds}
\kwd{Minimax optimality}
\end{keyword}

\end{frontmatter}

\section{Introduction}

Multi-armed bandit problems is a fundamental example of sequential decision making,
that has wide applications, such as personalized medicine \citep{tewari2017ads,shen2020learning}, advertisement placement \citep{li2010exploitation, chapelle2011empirical}, recommendation systems \citep{li2010contextual, xu2020contextual}. 
In its classical formulation, introduced by \citet{thompson1933likelihood} and popularized by  \citet{robbins1952some}, there are a finite number of arms, each associated with a mean reward, and one chooses arms sequentially with the goal to minimize the cumulative regret, relative to the maximum reward, over some time horizon. Many algorithms that are based on different principles, 
including upper confidence bound (UCB) \citep{lai1985asymptotically, auer2002finite,cappe2013kullback}, Thompson sampling \citep{gupta2011thompson, agrawal2012analysis, russo2014learning_thompson}, information-directed sampling \citep{russo2018learning, kirschner2018information}, and $\epsilon$-greedy \citep{sutton2018reinforcement, cesa1998finite}, have been proposed, that attain either the instance-dependent lower bound \citep{lai1985asymptotically, cappe2013kullback} or minimax lower bound \citep{audibert2009minimax, bubeck2012regret} or both \citep{menard2017minimax}. 


In applications mentioned above, however, there is usually context information (i.e., covariates) that can assist decision making, and each arm may be optimal for some contexts. For example, in clinical trials for testing a new treatment, whether it is more effective may depend on the genetic or demographic information of patients \citep{tewari2017ads}. The availability of contexts introduces a range of possibilities in terms of modeling: parametric \citep{filippi2010parametric, li2010contextual} versus non-parametric \citep{perchet2013multi, guan2018nonparametric}, linear \citep{auer2002using,  dani2008stochastic, rusmevichientong2010linearly, abbasi2011improved, goldenshluger2013linear, bastani2020online} versus non-linear \citep{jun2017scalable,kveton2020randomized, ding2021efficient}, finite \citep{goldenshluger2013linear, bastani2020online} versus infinite \citep{dani2008stochastic,abbasi2011improved,rusmevichientong2010linearly,kirschner2018information} number of arms, stochastic \citep{goldenshluger2013linear,bastani2020online,bastani2021mostly} versus adversarial \citep{auer2002nonstochastic,kakade2008efficient} contexts, etc; see the textbook \citep{lattimore2020bandit} for a comprehensive survey. Due to the vast literature and inconsistent terminology across research communities, we first state the framework in the current paper, and focus on the most relevant works. 

Specifically, we consider \textit{stochastic linear bandits} with $2 \leq K < \infty$ arms, where the sequence of contexts $\{\mvX_t: t \geq 1\}$ are independent and identically distributed (i.i.d.) $\bR^d$-random vectors.  At each time $t \geq 1$, one observes the context $\mvX_t$, and there is a potential reward $Y_t^{(k)}$ for each arm $k \in [K] := \{1,\ldots,K\}$, where
\begin{align}\label{def:linear_response}
    Y_t^{(k)} = \mvtheta_k'\mvX_t + \epsilon_t^{(k)}, \quad \text{ with }\;\; \Exp[  \epsilon_t^{(k)} \, \vert\, \mvX_{t}] = 0.
\end{align}
That is, each arm $k \in [K]$ is associated with a $d$-dimensional unknown parameter vector $\mvtheta_k$, and its expected reward given $\mvX_t$ is $\mvtheta_k'\mvX_t$. Denote by $A_t \in [K]$ the selected arm at time $t$, and if $A_t = k$, i.e., $k$-th arm is selected, then a reward $Y_t=Y_t^{(k)}$ is realized. In choosing  which arm to pull at time $t$ (i.e., $A_t$), one may only use the previous observations $(\mvX_s,Y_s), s < t$ and the current context $\mvX_t$. We evaluate the performance of an admissible rule by its cumulative regret up to a known time horizon $T$, denoted by $R_T$, which is relative to an oracle with the knowledge of arm parameters $\{\mvtheta_k: k \in [K]\}$.

Under this framework, \citet{goldenshluger2013linear} proposes a ``forced sampling" strategy for the two-arm case (i.e., $K=2$), referred as the ``OLS" algorithm,  and establishes a $O(d^3\log(T))$\footnote{Note that in the upper (resp.~lower) bound notation $O(\cdot)$ (resp. $\Omega(\cdot)$), the hidden multiplicative constant does not depend on the variables inside the parentheses, but may on other quantities, which are understood to be \textit{fixed}; for example, for $O(\log(T))$, the hidden constant may depend on $d,K$, but for $O(d^3\log(T))$, it does not depend on $d$.} upper bounded on $R_T$ under a ``margin" condition, which  requires that the probability of a context vector falling within $\tau$ distance to the boundary $\{\mvx \in \bR^d: \mvtheta_1' \mvx = \mvtheta_2'\mvx\}$ is $O(\tau)$, for small $\tau >0$;    the upper bound is improved to $O(d^2\log^{3/2}(d)\log(T))$ in \citet{bastani2020online}. Further, for any admissible procedure, \citet{goldenshluger2013linear} establishes a $\Omega(\log(T))$ lower bound on the worst-case regret over  a family of problem instances, and  conclude that the OLS algorithm achieves the optimal logarithmic dependence on $T$. In practice, however, it is sensitive to its tuning parameters, including the rate of exploration $q$. Specifically,  the OLS algorithm is scheduled to choose arm $1$ (resp. $2$) at time $\tau_n := \lfloor \exp(qn)\rfloor$ (resp. $\tau_n+1$) for $n \geq 1$, where $\lfloor\cdot\rfloor$ is the floor function, leading to about $2q^{-1}\log(T)$ forced sampling. Both undershoot and overshoot in selecting $q$ entail large cost: on one hand, $q$ is required to be small enough to ensure sufficient exploration \citep[see][Theorem 1]{goldenshluger2013linear}; on the other hand, if $q$ vanishes as $T$ increases, resulting in, say,  $\Omega(\log^{\kappa}(T))$ forced action for some $\kappa > 1$, then the regret would be $\Omega(\log^{\kappa}(T))$.

For  more general linear bandits (see Subsection  \ref{subsec:review_linear_bandits}),  ``optimism in the face of uncertainty" is a popular design principle, which, for each $t \geq 1$, chooses an arm $A_t \in [K]$ that maximizes an upper bound $\text{UCB}_t(k)$ on the potential reward $\mvtheta_k'\mvX_t$  \citep{auer2002using,dani2008stochastic,rusmevichientong2010linearly,li2010contextual,abbasi2011improved,hamidi2021better,wu2020stochastic}. Among this family, the LinUCB algorithm in \citet{abbasi2011improved} is perhaps the best known, and is near minimax optimal \citep[Chapter 24]{lattimore2020bandit}. In \citet{hamidi2021better}, in the framework  under consideration, the LinUCB algorithm is shown to have a $O(\log^2(T))$ regret, and it was not clear whether the $\log(T)$ gap between this upper bound and the optimal rate, achieved by the OLS algorithm, does exist or is an artifact of the proof techniques therein.

It is commonly perceived that the exploration–exploitation trade-off is at the heart of multi-armed bandit problems. However, \citet{bastani2021mostly} shows that a  greedy, pure-exploitation  algorithm is rate optimal, i.e., achieving a $O(\log(T))$ regret, under a ``covariate adaptive" condition, which however does not hold if there exist discrete components in the context, e.g., an intercept. In the absence of this condition, \citet{bastani2021mostly} proposes a ``Greedy-First" algorithm, that starts initially with the greedy algorithm, and switch to another algorithm, such as OLS or LinUCB, if it detects that the greedy algorithm fails. In addition to deciding when to switch, the Greedy-First algorithm has the same issue as the algorithm that it may transit into, e.g., the sensitivity to  parameters of OLS, and the potential sub-optimality of LinUCB.

\subsection{Our contributions}
First, we construct explicit problem instances, for which the cumulative regret of the LinUCB algorithm is both $\Omega(d^2\log^2(T))$ and $O(d^2\log^2(T))$, and thus prove that LinUCB is suboptimal for stochastic linear bandits in both the dimension $d$ and the horizon $T$. 
The suboptimality of LinUCB is because the path-wise upper confidence bounds in LinUCB, based on the self-normalization principle \citep{pena2008self}, is wider than the actual order  of statistical error in estimating the arm parameters; see subsection \ref{subsec:subopt_linUCB}.


Second, in view of its over-exploration, we propose to truncate the duration of the LinUCB algorithm, and call the proposed algorithm  ``Tr-LinUCB". Specifically, we run LinUCB up to a truncation time $S$, and then perform pure exploitation afterwards. For Tr-LinUCB, if the truncation time $S = Cd\log(T)$ for a large enough $C$, its cumulative regret is  $O(d\log(T))$; more importantly, in practice, if we choose $S = d\log^{\kappa}(T)$ for some $\kappa > 1$, the regret is  $O(d\log(T)\log\log(T))$. Thus unlike OLS, whose regret would be linear in the number of forced sampling, the cost of overshooting for Tr-LinUCB, i.e., $S$ being a larger order than the optimal, is a multiplicative  $\log\log(T)$ factor, that does not depend on $d$.  The practical implication is that Tr-LinUCB is insensitive to the selection of the truncation time $S$ if we err on the side of overshooting. Extensive experiments, including on several real-world datasets, corroborate our theory.

Third, we establish a matching $\Omega(d\log(T))$ lower bound on the worst-case regret over concrete families of  problem instances,
and thus show the rate optimality of Tr-LinUCB, with a proper  truncation time, in both the dimension $d$ and horizon $T$, for such families. 
More specifically, the characterization of the optimal dependence on $d$, in both the upper and lower bounds, appears novel,  holds under the low dimensional regime $d = O(\log(T)/\log\log(T))$, and relies on an assumption on  contexts that relates the expected instant regret to the second moment of  the arm parameters estimation error; see condition \ref{cond:unit_sphere}. Under this assumption, by similar arguments, it can be shown that the OLS algorithm proposed by \cite{goldenshluger2013linear} also achieves $O(d\log(T))$ regret. Thus  our contribution in this regard should be understood as proposing and working with such a condition, 
and verifying it for concrete problem instances, e.g., when contexts have a log-concave Lebesgue density.  Without this condition, we establish   $O(d^2\log(2d) \log(T))$ upper bound for Tr-LinUCB,  similar to that for OLS  \citep{bastani2020online}, which, however, may not be tight (in $d$) for any family of problem instances. As discussed above, the main practical advantage of Tr-LinUCB is its insensitivity to tuning parameters.

Finally, we note that the elliptical potential lemma \citep[Lemma 19.4]{lattimore2020bandit}, which is the main tool for the analysis of LinUCB \citep{abbasi2011improved,li2019nearly,wu2020stochastic,hamidi2021better,li2019nearly},  does not lead to the $O(\log(T))$ upper bound for Tr-LinUCB, and a tailored analysis is required to handle the dependence among observations induced by sequential decision making, and to show that information accumulates at a linear rate in time for each arm.

\subsection{More on stochastic linear bandits} \label{subsec:review_linear_bandits}
In the formulation \eqref{def:linear_response}, under the  ``large margin" condition (for $K=2$) that
$\bP( |(\mvtheta_1-\mvtheta_2)' \mvX_1| \leq \tau) = O(\tau^{\alpha})$ with $\alpha > 1$, the optimal regret is $O(1)$, achieved by the Greedy algorithm  \citep[Corollary 1]{bastani2021mostly}  and the LinUCB algorithm \citep[Remark 8.4]{wu2020stochastic,hamidi2021better}. We note that if $d$ is fixed, and $\mvX_1$ has a continuous component with a bounded  density, then the margin condition (i.e.,~$\alpha = 1$) holds, and thus it has a wider applicability.
Under the high dimensional regime, 
\citet{bastani2020online} extends the OLS algorithm by replacing the least squares estimator by Lasso, which achieves a $O(s_0^2\log^2(T))$ regret if $\log(d) = O(\log(T))$, where $s_0$ is the number of non-zero elements in $\mvtheta_1,  \mvtheta_2$. In addition, \citet{bastani2020online} conjectures $\Omega(d\log(T))$ lower bound in the low dimensional regime (see Section 3.3 therein), which we prove in the current work. Note that the Tr-LinUCB algorithm uses the ridge regression as the estimation method, and thus the targeted regime is low dimensional.

Next, we discuss a more general version of stochastic linear bandits. Specifically, at each time $t \geq 1$, based on previous observations, a decision maker chooses an action $A_t$ from a possibly infinite action set $\cA_t \subset \bR^{p}$, and receives a reward  $Y_t = \mvtheta_{*}' A_t + \epsilon_t$, with the goal of maximizing the cumulative reward, where $\mvtheta_{*} \in \bR^{p}$ is an unknown vector, and $\epsilon_t$ is a zero mean observation noise. To see how  the formulation in \eqref{def:linear_response} fits into this general framework,  when $K=2$, we let $\mvtheta_{*} = (\mvtheta_1',\mvtheta_2')'$ and $\cA_{t} = \{ (\mvX_t',\boldsymbol{0}_p')',(\boldsymbol{0}_p',\mvX_t')'\}$. Then $A_t = 1$ (resp. $2$) is identified with the first (resp. second) vector in $\cA_t$, and $\epsilon_t = \sum_{k=1}^{2} \epsilon_{t}^{(k)} I(A_t = k)$. Thus the formulation in \eqref{def:linear_response} may be viewed as a special case  where the action sets $\{\cA_t, t \in [T]\} \subset \bR^{p}$ are i.i.d.~with $p=dK$, and each $\cA_t$ has $K$ actions that are constructed from the context vector $\mvX_t \in \bR^d$.

When the size of action set $\cA_t \subset \bR^{d}$ is infinite (resp.~bounded by $K < \infty$), without further assumptions, the optimal worst-case regret has a  $\Omega(d\sqrt{T})$ (resp.~$\Omega(\sqrt{dT})$) lower bound, and is achieved,  up to a logarithmic factor in $T$ (resp. $T$ and $K$), by, e.g.,    \citet{dani2008stochastic,abbasi2011improved,rusmevichientong2010linearly,kirschner2018information} (resp.~by \citet{auer2002using,chu2011contextual,li2019nearly,russo2018learning}). When the action set is {fixed} and finite, i.e., $\cA_t = \cA$ for $t \geq 1$ with $|\cA| < \infty$, and  
there is a positive gap between the reward for the best and the second best action in $\cA$, the algorithms in \citet{lattimore2017end,combes2017minimal,pmlr-v108-hao20b,kirschner2021asymptotically} achieve the asymptotically optimal regret $C_*\log(T)$ as $T \to \infty$, where $C_{*}$ is a problem dependent quantity. In contrast, for the formulation in \eqref{def:linear_response}, under the margin condition, the dominant part of the cumulative regret is incurred when contexts appear (arbitrarily) close to the boundary.

\subsection{Outline and notations}
In Section \ref{sec:prob_formulation}, we formulate the stochastic linear bandit problem, and propose the Tr-LinUCB algorithm. In Section \ref{sec:opt_Tr-LinUCB}, we establish upper bounds on the cumulative regret of Tr-LinUCB, and  matching lower bounds on the worst-case regret over families of problem instances. Further, we show that  LinUCB is suboptimal in both $d$ and $T$. In Section \ref{sec:simulation}, we present experiments on both synthetic and real-world data. We present the upper and lower bound analysis in Section \ref{proof:regret_main_Td} and \ref{sec:proof_lower_bound}  respectively and conclude in Section \ref{sec:conclusion}. The remaining proofs are provided in the appendix. 

\smallskip

\noindent \textbf{Notations.} For a positive integer $n$, define $[n] := \{1,\ldots,n\}$, and denote by $\bN$ (resp. $\bN_{+}$) the set of all non-negative (resp. positive) integers. 
For $\tau \geq 0$, let $\lfloor\tau\rfloor:=\sup\{n \in \bN:n \leq \tau\}$ and $\lceil \tau\rceil := \inf\{n \in \bN: n \geq \tau\}$ be the floor and ceiling of $\tau$.
All vectors are column vectors. For $d \in \bN_{+}$, denote by $\bR^d$ the $d$-dimensional Euclidean space and by $\cS^{d-1} := \{\mvx \in \bR^d: \|\mvx\| = 1\}$ the unit sphere in $\bR^d$, where $\|\mvx\|$ denotes the Euclidean norm of $\mvx$.  For  a $d$-by-$d$ matrix $\bV$ and a vector $\mvx$ of length $d$, define $\|\mvx\|_{\bV} = \sqrt{\mvx'\bV \mvx}$, where $\mvx'$ denotes the transpose of $\mvx$, and denote by $\lambda_{\min}(\bV)$ and $\lambda_{\max}(\bV)$ the smallest and largest (real) eigenvalue of $\bV$. Denote by $\boldsymbol{0}_{d}$ (resp.~$\boldsymbol{1}_{d}$) the $d$-dimensional all-zero (resp.~one) vector, and by $\bI_d$ the $d$-by-$d$ identity matrix.

Denote by $\sigma(Z_1,\ldots,Z_t)$ the sigma-algebra generated by   random variables $Z_1,\ldots,Z_{t}$, and 
 by  $I(A)$ the indicator function of an event $A$.
Denote by $\textup{Unif}(0,1)$ and  $\textup{Unif}(\sqrt{d}\cS^{d-1})$ the uniform distribution on the interval $(0,1)$ and on the sphere with radius $\sqrt{d}$ in $\bR^d$, respectively. 
Denote by $N_d(\mvmu, \bV)$ the $d$-dimensional normal distribution with the mean vector $\mvmu$ and the covariance matrix $\bV$; the subscript $d$ is omitted if $d=1$. For a random vector $\mvZ$, denote by $\Cov(\mvZ)$ its covariance matrix.

\section{Problem Formulation and Tr-LinUCB Algorithm} \label{sec:prob_formulation}

As discussed in the introduction, we consider $2 \leq K < \infty$ arms, and assume that the sequence of contexts $\{\mvX_t:t \geq 1\}$ are i.i.d.~$\bR^d$-random vectors, which may or may not contain an intercept. Recall that at each time $t \in \bN_{+}$, one observes the context $\mvX_t$, and the potential outcome, $Y_t^{(k)}$, for arm $k \in [K]$ is given by equation \eqref{def:linear_response}. If arm $k$ is selected at time $t$,  the realized reward $Y_t$ is $Y_t^{(k)}$. For simplicity, we assume that $(\mvX_t; \epsilon_t^{(k)}, k \in [K])$ for $ t \in  \bN_{+}$  are independent and identically distributed as a generic random vector  $(\mvX; \epsilon^{(k)}, k \in [K])$. Thus a  problem instance is determined by arm parameters $\{\mvtheta_{k}: k \in [K]\}$, and the  distribution of this generic random vector.


We assume that the time horizon $T \geq \max\{d,16\}$ is known, and then  an admissible rule is described by a sequence of measurable functions $\pi_t: (\bR^{d}*[K]*\bR)^{t-1}*\bR^{d}*\bR \to [K]$ for $t \in [T]$, where $\pi_t$ selects an arm based on the observations up to time $t-1$ and the current context $\mvX_t$, maybe  randomly with the help of a $\textup{Unif}(0,1)$ random variables $\xi_t$, that is,
\begin{align}\label{eq:admissible_rule} 
    A_t = \pi_t(\{\mvX_s,A_s,Y_s: s < t\},\; \mvX_t,\; \xi_t), \qquad
     Y_t =  Y_t^{(A_t)},\quad \text{ for } t \in [T],
\end{align}
where $\{\xi_t: t \in \bN_{+}\}$ are  i.i.d., independent from all  potential observations $\{\mvX_t, Y_t^{(k)}: k \in [K], t \in \bN_{+}\}$. 
Let $\cF_0 = \sigma(0)$; for each $t \in [T]$,  denote by $\cF_t :=  \sigma(\mvX_s,A_s,Y_s: s \in [t])$ the available information up to time $t$, and by $\cF_{t+} := \sigma(\cF_{t}, \mvX_{t+1}, \xi_{t+1})$ the information set during the decision making at time $t+1$. Then $A_t \in \cF_{(t-1)+}$ for each $t \in [T]$.



We evaluate the performance of an admissible rule in \eqref{eq:admissible_rule} in terms of its cumulative regret $R_T$, i.e.,
\begin{align} \label{def:regret}
    R_T(\{\pi_t:t \in [T]\}) := \sum_{t \in [T]} \Exp[\hat{r}_t], \quad \text{ where }\;\; \hat{r}_t := \max_{k \in [K]} (\mvtheta_k'\mvX_t) - \mvtheta_{A_t}'\mvX_t.
\end{align}
In particular, $\hat{r}_t$ may be viewed as the regret, averaged over the observation noises, at time $t$ given the context $\mvX_{t}$ and action $A_t$. If the rule $\{\pi_t\}$ is understood from its context, we omit the argument and simply write $R_T$.

Throughout the paper, we  assume that the arm parameters are bounded in length,  that the observation noises are subgaussian, and that the length of contexts are almost surely bounded, where the upper bound is allowed to increase with the dimension $d$.  Specifically,\\

\noindent \mylabel{assumption_parameter_noise}{(C.I)}  There exist absolute positive constants  $m_{\theta}$,  $m_{R}$, $m_{X}$, $\sigma^2$ such that for each  $k \in [K]$,    $\|\mvtheta_k\| \leq m_{\theta}$, $\Exp[|\mvtheta_k' X|] \leq m_{R}$, $\|\mvX\| \leq \sqrt{d} m_{X}$,  
$\Exp[e^{\tau \epsilon^{(k)}} \vert \mvX] \leq e^{\tau^2 \sigma^2/2}$ for  $\tau \in \bR$,  almost surely.\\

It is common in the literature to assume
that $\lambda_{\min}(\Exp[\mvX\mvX']) = \Omega(1)$; see \ref{cond:posdef} ahead. Since $\lambda_{\min}(\Exp[\mvX\mvX']) \leq d^{-1} \Exp[\|\mvX\|^2]$, it implies that the upper bound on $\|\mvX\|$ must be $\Omega(\sqrt{d})$. Note that \citet[Assumption 1]{bastani2020online} assumes  the $\ell_1$ norm of $\mvtheta_k$ and $\ell_{\infty}$ norm of $\mvX$ bounded by $m_{\theta}$ and $m_{X}$,  respectively, which are \textit{stronger} than the first three conditions in 
\ref{assumption_parameter_noise}.


\subsection{The proposed Tr-LinUCB Algorithm}

As discussed in the introduction, the exploration of the popular LinUCB algorithm \citep{li2010contextual,abbasi2011improved,lattimore2020bandit,hamidi2021better,wu2020stochastic} is excessive, which leads to its suboptimal performance. We propose to stop the LinUCB algorithm early, and perform pure exploitation afterwards; we call the proposed algorithm ``Tr-LinUCB", where ``Tr" is short for ``Truncated".

The Tr-LinUCB algorithm assumes that the constants $m_{\theta}$ and $\sigma^2$ in \ref{assumption_parameter_noise} known, and requires user provided parameters $\lambda > 0$ and  $S \leq T$, where $\lambda$ is used in estimating the arm parameters by  ridge regression, and $S$ denotes the truncation time of LinUCB. Specifically, let $\bV_0^{(k)} = \lambda \bI_{d}$ and $\mvU_0^{(k)} = \boldsymbol{0}_{d}$ for $k \in [K]$. At each time $t \in [T]$, it involves two steps.

\begin{enumerate}[wide]  
    \item (Arm selection) If $t \leq S$,  we follow the  LinUCB algorithm, by selecting the arm that maximizes  upper confidence bounds for potential rewards; otherwise, we select an arm greedily. Specifically,  $A_t =\arg\max_{k \in [K]}  \text{UCB}_t(k) I(t \leq S) + ((\hat \mvtheta_{t-1}^{(k)})' \mvX_t) I(t > S)$, where 
  \begin{align}\label{eq:arm_selection}
    \begin{split}
&\text{UCB}_t(k):= (\hat \mvtheta_{t-1}^{(k)})' \mvX_t + \sqrt{\beta_{t-1}^{(k)}} \|\mvX_t\|_{(\bV_{t-1}^{(k)})^{-1}},  \qquad 
\hat \mvtheta_{t-1}^{(k)} = (\bV_{t-1}^{(k)})^{-1} \mvU_{t-1}^{(k)},\quad \text{and}\\
& \sqrt{\beta_{t-1}^{(k)}} = m_{\theta} \sqrt{\lambda} + \sigma \sqrt{2\log(T) + \log\left({\text{det}(\bV_{t-1}^{(k)})}/{\lambda^d} \right)}.
    \end{split}
    \end{align}
The ties in the ``argmax" are broken either according to a fixed rule or at random.
    
    \item (Update estimates) We update the associated quantities using the current context and reward for the selected arm: let $(\bV_{t}^{(k)}, \mvU_{t}^{(k)} ) = (\bV_{t-1}^{(k)}, \mvU_{t-1}^{(k)} )$ for each $k \neq A_t$, and
    \begin{align}\label{eq:update_para}
        \bV_{t}^{(A_t)} = \bV_{t-1}^{(A_t)} + \mvX_t\mvX_t',\qquad
        \mvU_{t}^{(A_t)} = \mvU_{t-1}^{(A_t)} + \mvX_t Y_t.
    \end{align}
\end{enumerate}

If we set $S = T$, the Tr-LinUCB algorithm reduces to LinUCB. Here, $\hat \mvtheta_{t-1}^{(k)}$ is the ridge regression estimator for $\mvtheta_{k}$ based on data in those rounds, up to time $t-1$, for which the $k$-th arm is selected, i.e., $\{(\mvX_s, Y_s): 1 \leq s < t  \text{ and } A_s = k\}$. The next lemma explains the choice of $\{\beta_t^{(k)}\}$, leading to upper confidence bounds for the potential rewards, whose proof is essentially due to \citet{abbasi2011improved} and can be found in Appendix \ref{proof:lemma:LinUCB}. We note that it holds for all $t \in [T]$, \textit{beyond the time of truncation}, $S$.

\begin{lemma}\label{lemma:LinUCB}
Assume the condition \ref{assumption_parameter_noise} holds. With probability at least $1-K/T$, $
    \|\hat \mvtheta_{t}^{(k)} - \mvtheta_{k}\|_{\bV_{t}^{(k)}} \leq \sqrt{\beta_t^{(k)}}$     for all $t \in [T]$ and $k \in [K]$.
\end{lemma}

By the Cauchy–Schwarz inequality, with probability at least $1-K/T$, $\text{UCB}_t(k)$ is an upper bound for $\mvtheta_{k}'\mvX_t$ for all $t \in [T]$ and $k \in [K]$. 
Due to \ref{assumption_parameter_noise} and \citep[Section 20.2]{lattimore2020bandit}, we may use an upper bound $\tilde{\beta_t}$ in place of $\beta_t^{(k)}$, where 
\begin{align}
       \label{eq:uuper_beta_t}
       \sqrt{\tilde{\beta}_t} = \sqrt{\lambda} m_{\theta} + \sigma \sqrt{2\log(T) + d \log\left(1+{t m_X^2}/{\lambda}\right)},\quad \text{ for } t \in[T].
\end{align}
Using  $\beta_t^{(k)}, k \in [K]$  has the advantage of not requiring the knowledge of $m_X$ in practice, while $\tilde{\beta_t}$ is deterministic and does not depend on $k \in [K]$, and will be used in our analysis.


\section{Regret analysis for Tr-LinUCB}\label{sec:opt_Tr-LinUCB}
For regret analysis, we focus on the  $K = 2$  case for simplicity. The upper bound part extends to the $K >2$ case in a straightforward way, while the optimal dependence on $K$ requires new ideas and further investigation.


\subsection{Assumptions}\label{sect:assumptions}
In this subsection, we collect assumptions and their discussions. For each $k \in [2]$ and $h \geq 0$,  define $\cU_{h}^{(k)} := \{x \in \bR^d: \mvtheta_k'x > \max_{j \neq k} \mvtheta_j'x + h\}$ to be the set of context vectors for which the potential reward for the $k$-th arm is better than for the other arm by at least $h$.   Let  $\text{sgn}(\tau) = I(\tau >0) - I(\tau<0)$ for $\tau \in \bR$ be  the sign function. 

Assume that for some absolute positive constants $L_0, L_1> 1$ and $\ell_0, \ell_1 < 1$,\\

\noindent \mylabel{cond:margin}{(C.II)}  $\Pro\left( |(\mvtheta_1 - \mvtheta_2)' \mvX| \leq \tau\right) \leq L_0 \tau$ for all $\tau > 0$.\\

\noindent \mylabel{cond:posdef}{(C.III)} $\lambda_{\min}\left( \Exp\left[ \mvX \mvX'I\left(\mvX \in \cU_{{\ell_0}}^{(k)} \right)\right]\right) \geq \ell_0^2$ for $k = 1,2$.\\

\noindent \mylabel{cond:continuity}{(C.IV)}   $\Pro(|\mvu'\mvX| \leq \ell_1) \leq 1/4$ for all $\mvu \in \cS^{d-1}$.\\

\noindent \mylabel{cond:unit_sphere}{(C.V)}  
$\|\mvtheta_1- \mvtheta_2\| \geq L_1^{-1}$ and 
for any $\mvv \in \cS^{d-1}$, $\Exp[|\mvu_{*}'\mvX| I(\text{sgn}(\mvu_{*}'\mvX) \neq \text{sgn}(\mvv'\mvX))] \leq L_1 \|\mvu_{*}-\mvv\|^2$, where $\mvu_{*} = (\mvtheta_1 - \mvtheta_2)/\|\mvtheta_1 - \mvtheta_2\|$.\\



The first two conditions are standard in the literature; see, e.g., \citet{goldenshluger2013linear,bastani2020online, bastani2021mostly}, and the discussions therein. In particular, the condition \ref{cond:margin} is known as the ``margin condition", requiring that the probability of $\mvX$ falling within $\tau$ distance to the boundary $\{x \in \bR^d: \mvtheta_1'x = \mvtheta_2'x\}$ is upper bounded by $L_0 \tau$; note that since we can always increase $L_0$,   \ref{cond:margin} is in force only for small $\tau > 0$. The condition \ref{cond:posdef} is known as the ``positive-definiteness condition", which requires roughly that each arm is optimal by at least $\ell_0$ with a positive probability, and that conditional on this event, the context $\mvX$ spans $\bR^d$; note that we use $\ell_0$ on both sides of \ref{cond:posdef}, which is without loss of generality, since the left (resp.~right) hand side increases (resp. decreases) as $\ell_0$ becomes smaller.

The condition \ref{cond:continuity} requires that the projection of $\mvX$ onto any direction is not concentrated about zero. We refer to it as ``absolute continuity condition", as justified by the following lemma. Specifically, it holds with some constant $\ell_1$, depending on $d$, if the context vector $\mvX$ is absolutely continuous with respect to the Lebesgue measure, after maybe removing the intercept. If  its Lebesgue density is log-concave, then $\ell_1$ is dimension free, i.e., independent of $d$. Note that a density $p$ is log-concave, if $\log(p)$ is a concave function. 
In the following lemma, if the context $\mvX = (1,({\mvX}^{(-1)})')'$ has an intercept, let $\tilde{X} = {\mvX}^{(-1)}$ and $\tilde{d} = d-1$; otherwise, 
let  $\tilde{X} = {\mvX}$ and $\tilde{d} = d$. 


\begin{lemma}\label{lemma:continuity}
Let $C > 0$ be some constant, and assume that $\tilde{d} \geq 1$ and that $\tilde{X}$ has a density $p_{\tilde{X}}$ with respect to the $\tilde{d}$-dimensional Lebesgue measure.
\begin{enumerate}[label=(\roman*)]
    \item Assume that the condition \ref{assumption_parameter_noise} holds and that $d$ is fixed. If  $p_{\tilde{\mvX}}$ is upper bounded by $C$, then \ref{cond:continuity} holds for some constant $\ell_1$ that depends only on $C, m_X, d$.
    \item If $p_{\tilde{\mvX}}$ is log-concave,  $\|\Exp[\tilde{\mvX}]\| \leq C$, and the eigenvalues of $\Cov(\tilde{\mvX})$ are between $[C^{-1},C]$, then \ref{cond:continuity} holds for some constant $\ell_1$ that depends only on $C$.
\end{enumerate}
\end{lemma}
\begin{proof}
See Appendix \ref{proof:lemma:continuity}.
\end{proof}

When the context $\mvX$ has more  discrete components than the intercept, we require a generalization of the condition \ref{cond:continuity}; see 
 Subsection \ref{subsec:discrete}.
We choose to first focus on \ref{cond:continuity} in order to streamline our proofs. We refer readers to \citet[Chapter 10]{artstein2015asymptotic} for more information about log-concave densities. For a random vector $\mvZ$ with a log-concave density, $\mvZ$ is said to be \textit{isotropic} if $\Exp[\mvZ] = \boldsymbol{0}_{d}$ and $ \Cov(\mvZ) = \bI_d$; clearly, if $\tilde{X}$ has an isotropic log-concave density, then the part (ii) applies. More concrete examples are when  
components of $\tilde{X}$ are independent, and each has a log-concave density with mean $0$ and variance between $[C^{-1},C]$ (e.g., the uniform distribution on $[-1,1]$), or the uniform distribution on the Euclidean ball $\{\mvx \in \bR^{d}: \|\mvx\| \leq \sqrt{d}\}$.

The condition \ref{cond:unit_sphere}, together with its lower bound version, is the key to characterize the dependence of the optimal regret on the dimension $d$  for families of problem instances. We discuss its role in detail in Subsection \ref{subsec:optimal_d}, and here provide examples for which it holds.

\begin{lemma}\label{lemma:logconcave_example}
Let $C > 0$ be some constant. Assume $\mvX$ has a log-concave density  on $\bR^d$ with $\Exp[\mvX] = \boldsymbol{0}_{d}$ and the eigenvalues of $\Cov(\mvX)$  between $[C^{-1},C]$. Then there exists a constant $L > 0$, that depends only on $C$, such that  for any $\mvu, \mvv \in \cS^{d-1}$,
$$L^{-1} \|\mvu-\mvv\|^2 \;\leq\; \Exp[|\mvu'\mvX| I(\textup{sgn}(\mvu'\mvX) \neq \textup{sgn}(\mvv'\mvX))] \;\leq\; L \|\mvu-\mvv\|^2.
$$
The upper bound part continues to hold if $\|\Exp[\mvX]\| \leq C$, without requiring $\mvX$ centered.
\end{lemma}

\begin{proof}
See Appendix \ref{proof:lemma:logconcave_example}.
\end{proof}

\begin{remark}
Relevant properties regarding log-concave densities  are in Appendix \ref{app:logconcave}. In short, if $\mvX$ has an isotropic log-concave density on $\bR^d$, so does $(\mvu'\mvX, \mvw'\mvX)$ on $\bR^2$, for any $\mvu,\mvw \in \cS^{d-1}$ with $\mvu'\mvw = 0$. Further, isotropic log-concave densities  in low dimensions   are  uniformly upper bounded, bounded away from zero near the origin, and decay exponentially fast away from  the origin, which lead to the dimension-free results in Lemma \ref{lemma:continuity} and \ref{lemma:logconcave_example}.
\end{remark}

\begin{remark}
In addition to log-concave densities, conditions \ref{cond:continuity} and \ref{cond:unit_sphere} hold with absolute constants for any $d \geq 3$ if $\mvX$ has the uniform distribution on the sphere in $\bR^{d}$ with center $\boldsymbol{0}_{d}$ and radius $\sqrt{d}$, i.e., $\textup{Unif}(\sqrt{d} \cS^{d-1})$, which is verified in the proof of Theorem \ref{theorem:lower_bound}. Further, if a distribution $F$ for the context $\mvX$ verifies  conditions \ref{assumption_parameter_noise}-\ref{cond:unit_sphere}, then so does any equivalent distribution $G$, such that the Radon–Nikodym derivative $dG/dF$ takes value in $[C^{-1},C]$, for some absolute constant $C > 0$.
\end{remark}
\subsection{Regret analysis without the condition \ref{cond:unit_sphere}}
We denote by $\mvTheta_0 := (m_{\theta}, m_{R}, m_{X}, \sigma^2,\ell_0,\ell_1,L_0)$ the collection of parameters appearing in  conditions \ref{assumption_parameter_noise}-\ref{cond:continuity}, and  define
\begin{align} \label{chi_order}
    \Upsilon_{d,T} = d\log(T) + d^2 \log(d\log(T)).
\end{align}

\begin{theorem}\label{theorem:regret_main_T}
Consider problem instances that satisfy  conditions \ref{assumption_parameter_noise}-\ref{cond:continuity}, and the Tr-LinUCB algorithm with a fixed $\lambda > 0$. There exist positive constants $C_0$ and $C_1$, depending only on $\mvTheta_0, \lambda$, such that 
if the truncation time $S \geq S_0$ with $S_0 = \lceil C_0 \Upsilon_{d,T} \rceil$, then
\begin{align*}
   R_T\;\; \leq \;\; C_1  S_0  + C_1(d \log(T) + d^{2} \log(S))\log(S/S_0) + C_1 d^{2}\log(2d)\log(T/S).
\end{align*}
\end{theorem}
\begin{proof}
See Section \ref{proof:regret_main_Td}, where we also discuss the proof strategy. 
\end{proof}

In the following immediate corollary, we establish upper bounds on the cumulative regret corresponding to different choices of the truncation time $S$. 


\begin{corollary}\label{corollary:regret_main_T}
Consider the setup in Theorem \ref{theorem:regret_main_T}.
\begin{enumerate}
    \item[(i)] There exists  a positive constant $C_0$, depending only on $\mvTheta_0,   \lambda$, such that if $S = C \Upsilon_{d,T}$ for some  $C \geq C_0$, then $R_T\leq C_1 d^2\log(2d)\log(T)$, where the constant $C_1$ depends only on $\mvTheta_0,   \lambda$, and $C$.
    \item[(ii)] If $S = \Upsilon_{d,T} \log^{\kappa}(T)$ for  $\kappa > 0$, then $R_T\leq C_1 \tilde{\kappa}(d^2 \log(2d) \log(T) + d\log(T)\log\log(T))$, where $\tilde{\kappa} = \max\{\kappa,1\}$, and the constant $C_1$ depends only on $\mvTheta_0,   \lambda$. 
    \item[(iii)] If $S = T$, then $R_T \leq C_1d^2 \log^2(T)$,  where the constant $C_1$ depends only on $\mvTheta_0,   \lambda$.
\end{enumerate}
\end{corollary}

As we shall see, the dependence on $d$ in the above corollary is not optimal, so we assume $d$ fixed for now, and in particular $\Upsilon_{d,T} = O(\log(T))$.  If we select $S = C\log(T)$ for a large enough constant $C$, the regret is of order $\log(T)$, which matches the optimal dependence on $T$; see   \citet[Theorem 2]{goldenshluger2013linear} and also Theorem \ref{theorem:lower_bound} ahead. In practice, the constant $C_0$ in part (i) above is unknown. However, part (ii) shows that the cost is only a  $\log(\log(T))$ multiplicative factor, if we choose $S$ to be of order $\log^{\kappa}(T)$ with $\kappa > 1$, larger than the optimal $\log(T)$ order. This suggests that we prefer ``overshooting" than ``undershooting" in deciding the truncation time $S$ in practice. 

Further, since the proposed Tr-LinUCB algorithm with $S=T$ reduces to  LinUCB, part (iii) establishes a $O(d^2 \log^2(T))$ upper bound for LinUCB, which generalizes \citet[Corollary 8.1]{hamidi2021better} in making the dependence on $d$ explicit. More importantly, we establish a matching lower bound for LinUCB in Section \ref{subsec:subopt_linUCB}, and thus explicitly show that LinUCB is \textit{sub-optimal in both $d$ and $T$}, and that the truncation is necessary.

Finally, we note that \citet{bastani2020online} establishes an $O(d^2\log^{3/2}(d)\log(T))$ upper bound for the OLS algorithm proposed by \citet{goldenshluger2013linear} under conditions 
\ref{cond:margin}, \ref{cond:posdef}, and a slightly stronger version of \ref{assumption_parameter_noise}.
In part (i) above, we establish a similar result for Tr-LinUCB, under the additional assumption \ref{cond:continuity}, which does not allow discrete components other than the intercept; we relax this condition in Subsection \ref{subsec:discrete}. As mentioned in the introduction, the main practical advantage of Tr-LinUCB over OLS is  its insensitivity to tuning parameters. 

\subsection{Optimal dependence on the dimension $d$}\label{subsec:optimal_d}

\begin{figure}[!tbp]
    \centering
    \includegraphics[width=\textwidth]{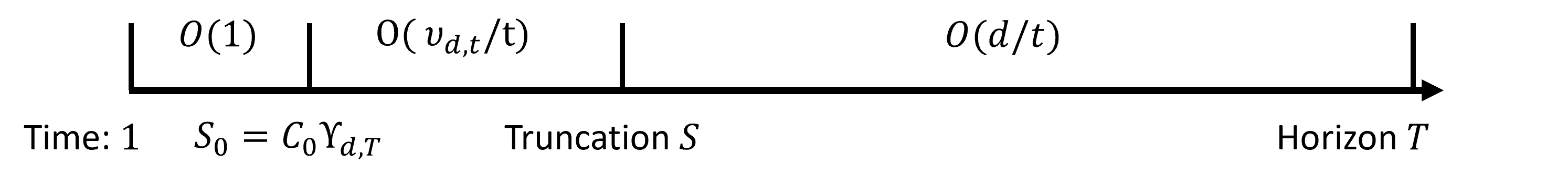}
    \caption{Above the time axis are orders of the expected regret at time $t$ within each stage, where $\upsilon_{d,t} = d\log(T) + d^2 \log(t)$ and $\Upsilon_{d,T}$ in \eqref{chi_order}, and below are important moments for the proposed Tr-LinUCB algorithm.}
    \label{fig:regret_analysis}
\end{figure}

Next, we show that under the additional condition \ref{cond:unit_sphere}, the Tr-LinUCB algorithm achieves the optimal dependence in both the dimension $d$ and horizon $T$.  We start with a discussion on the strategy for the regret analysis, and emphasize the role of  \ref{cond:unit_sphere}.


One of the key steps is to show that with a high probability, $\lambda_{\min}(\bV_t^{(k)})$ is $\Omega(t)$  for  each $k \in [2]$ and $t \geq S_0 := C_0\Upsilon_{d,T}$, where $C_0$ is an appropriate constant, which implies that
$$
|\text{UCB}_t(k) -  (\hat \mvtheta_{t-1}^{(k)})' \mvX_{t}| =O_{P}((\upsilon_{d,t}/t)^{1/2}), \qquad
\|\hat \mvtheta_t^{(k)} -\mvtheta_k\| = O_{P}((d/t)^{1/2}),
$$
where the former is  the bonus part in the upper confidence bound with  $\upsilon_{d,t} := d\log(T) + d^2 \log(t)$ (see \eqref{eq:arm_selection}), and
the latter  the estimation error. As depicted in Figure \ref{fig:regret_analysis}, the analysis  involves three periods. In the first stage, up to time $S_0$, due to the bonus part, the behavior of Tr-LinUCB is close to random guess.
In the second stage, i.e., from $S_0$ to  the truncation time $S$,  Tr-LinUCB chooses an action $A_t$ by maximizing $\text{UCB}_t(k)$ over $k \in [K]$. Since the bonus dominates the estimation error,  Tr-LinUCB suffers a $O((\upsilon_{d,t}/t)^{1/2})$ regret when $\mvX_{t}$ falls within $O((\upsilon_{d,t}/t)^{1/2})$ distance to the  boundary, which leads to an expected $O(\upsilon_{d,t}/t)$ regret at time $t$ under the ``margin" condition \ref{cond:margin}. 



The condition \ref{cond:unit_sphere} is used in the analysis for the third stage, i.e., after the truncation time $S$, and is the key to remove a $d\log(d)$-factor in the cumulative regret bound in Theorem \ref{theorem:regret_main_T}. Specifically, for some $t > S$, denote by $\hat{\mvDelta}_{t-1}:=\hat \mvtheta_{t-1}^{(1)}-\hat \mvtheta_{t-1}^{(2)}$ an estimator for $\mvDelta = \mvtheta_1 - \mvtheta_2$, and note that $\mvX_{t}$ is independent from $\cF_{t-1}$, and $\hat{\mvDelta}_{t-1} \in \cF_{t-1}$. In the proof of Theorem \ref{theorem:regret_main_T},
we establish an \textit{exponential} bound on the tail probability of $(\hat{\mvDelta}_{t-1}-\mvDelta)'\mvX_t$, \textit{conditional on} $\mvX_t$, using the pessimistic  $O(\sqrt{d})$ bound  in \ref{assumption_parameter_noise} for $\|\mvX_t\|$, i.e.,
$|(\hat{\mvDelta}_{t-1}-\mvDelta)'\mvX_t| \leq \sqrt{d} m_{X} \|\hat{\mvDelta}_{t-1}-\mvDelta\|$. 
Now, assume the condition \ref{cond:unit_sphere}  holds.
When the sign of $\hat{\mvDelta}_{t-1}'\mvX_{t}$ differs from that of $\mvDelta'\mvX_{t}$, an instant regret $|\mvDelta'\mvX_{t}|$ is incurred;  by conditioning on $\hat{\mvDelta}_{t-1}$, \ref{cond:unit_sphere} upper bounds the expected regret at time $t$  by, up to a multiplicative constant, the \textit{second moment} of  the estimation error $\|\hat{\mvDelta}_{t-1} - \mvDelta\|$.
Thus, exchanging the order of conditioning, i.e., from $\mvX_t$ to $\hat{\mvDelta}_{t-1}$, leads to the removal of a $d$-factor. The additional $\log(d)$-factor is due to the difference between the exponential and polynomial moment bounds.

Denote by $\mvTheta_1: = \mvTheta_0 \cup \{L_1\}$ the parameters appearing in conditions  \ref{assumption_parameter_noise}-\ref{cond:unit_sphere}.


\begin{theorem}\label{theorem:regret_main_d}
Consider problem instances for which conditions  \ref{assumption_parameter_noise}-\ref{cond:unit_sphere} hold, 
and the Tr-LinUCB algorithm with a fixed $\lambda > 0$.  There exist positive constants $C_0$ and $C_1$, depending only on $\mvTheta_1, \lambda$, such that 
if the truncation time $S \geq S_0$ with $S_0 = \lceil C_0 \Upsilon_{d,T} \rceil$, then 
$R_T \leq  C_1 d \log(T)\log(2S/S_0) +  C_1 d^{2}\log(S)\log(2S/S_0)$.
\end{theorem}
\begin{proof}
See Section \ref{proof:regret_main_Td}. 
\end{proof}

As an immediately corollary, we improve the dependence on $d$ over Theorem \ref{theorem:regret_main_T}. For simplicity, we focus on the following low dimensional regime:
\begin{align}\label{def:regime}
    d \leq \log(T)/(\log\log(T)),
\end{align}
under which we are able to characterize the optimal regret.

\begin{corollary}\label{cor:regret_d}
Consider the setup in Theorem \ref{theorem:regret_main_d}, and assume \eqref{def:regime} holds. 
\begin{enumerate}
    \item[(i).] There exists a positive  constant $C_0$, depending only on $\mvTheta_1,   \lambda$, such that 
if $S = C d\log(T)$ for some  $C \geq C_0$, then $R_T\leq C_1 d\log(T)$, where the constant $C_1$ depends only on $\mvTheta_1,   \lambda$ and $C$.
    \item[(ii).] If $S = d \log^{\kappa}(T)$ for some $\kappa > 1$, then $R_T\leq C_1 \kappa^2  d\log(T)\log\log(T)$, where the constant $C_1$ depends only on $\mvTheta_1,   \lambda$.
\end{enumerate}
\end{corollary}

Next we establish a lower bound that matches the order in the part (i) of  Corollary \ref{cor:regret_d}. For $0\leq r_1\leq r_2$, denote by $\cB_d(r_1,r_2) = \{\mvx \in \bR^d: \|\mvx\| \in [r_1,r_2]\}$ the region between two spheres with radius $r_1$ and $r_2$. Consider the following problem instances.

\begin{enumerate}[label=\textsc{(P.I)}, wide]  
    \item $K = 2$ and $d \geq 3$. $\mvtheta_1 = \boldsymbol{0}_{d}$, and $\mvtheta_2 \in \cB_d(1/2,1)$;  the context $\mvX$ has a distribution $F$, independent from $\epsilon^{(1)},\epsilon^{(2)}$, which are i.i.d.~$N(0,1)$ random variables.
    \label{prob_instances_all_lower}
\end{enumerate}

Given $T$, $d$, $\mvtheta_2$ and $F$, a problem instance in \ref{prob_instances_all_lower} is completely specified, and to emphasize the dependence, we write $R_T(\{\pi_t, t \in[T]\};\ d,\mvtheta_2,F)$ for the cumulative regret $R_T$ of an admissible rule $\{\pi_t:t\in [T]\}$. 

\begin{theorem}\label{theorem:lower_bound}
Consider problem instances in \ref{prob_instances_all_lower} under the assumption  \eqref{def:regime}.  
Assume either the distribution $F$ is $\textup{Unif}(\sqrt{d}\cS^{d-1})$ or  $F$ has an isotropic log-concave density and $\|\mvX\| \leq \sqrt{d} m_{X}$ almost surely. Then 
there exist an absolute constant $c > 0$ and a constant $C > 0$, that only depends on $m_X$, such that
\begin{align*}
c d \log(T) \;\;\leq \;\;
    \inf_{\{\pi_t, t \in[T]\}} \ \sup_{\mvtheta_2 \in \cB_d(1/2,1)} R_T(\{\pi_t, t \in[T]\};\ d,\mvtheta_2, F) 
    \;\;\leq \;\; C d\log(T),
\end{align*}
where the infimum is taken over  all admissible algorithms.
\end{theorem}
\begin{proof}
See Section \ref{sec:proof_lower_bound} for the lower bound proof, and Appendix \ref{proof:theorem:lower_bound_upper} for the upper bound. 
\end{proof}

First,  the proof for the lower bound is in the same spirit as that for \citet[Theorem 2]{goldenshluger2013linear}. 
The novel steps include 
establishing a lower bound version of the condition \ref{cond:unit_sphere} (i.e., Lemma \ref{lemma:logconcave_example}), and  
an application of van Tree's inequality to make the dependence on $d$ explicit (Appendix \ref{app:van_tree}). Note that the lower bound does not require the condition  $\|\mvX\| \leq \sqrt{d} m_{X}$, and holds beyond the low-dimensional regime.

Second, for the upper bound part, we verify the conditions \ref{assumption_parameter_noise}-\ref{cond:unit_sphere}, and apply part (i) of Corollary \ref{cor:regret_d} for the proposed Tr-LinUCB algorithm. In particular, we conclude that Tr-LinUCB achieves the optimal dependence in both $d$ and $T$, if we choose the truncation time $S = Cd \log(T)$ for some sufficiently large $C$, for the problem instances in \ref{prob_instances_all_lower}, under the low dimensional regime \eqref{def:regime}. We also note that if $S = d\log^{\kappa}(T)$ for some $\kappa > 1$, the cost is a multiplicative $\log\log(T)$-factor, that \textit{does not} depend on $d$.

\subsection{Sub-optimality of LinUCB}\label{subsec:subopt_linUCB}
In Corollary  \ref{corollary:regret_main_T}, we establish a $O(d^2\log^2(T))$ upper bound for LinUCB, i.e., Tr-LinUCB with $S=T$. 
Below, we construct concrete problem instances, for which the cumulative regret of LinUCB is $\Omega(d^2\log^2(T))$. This indicates that the $O(d^2\log^2(T))$ upper bound is in fact tight, and demonstrates that  LinUCB is sub-optimal, suffering a $d\log(T)$-factor compared to its appropriately truncated version; see discussions below.


\begin{enumerate}[label=\textsc{(P.II)}, wide]  
    \item $K = 2$, $d \geq 3$, $\mvtheta_1 = (1, \boldsymbol{0}_{d-1}')'$, $\mvtheta_2 = (-1, \boldsymbol{0}_{d-1}')'$.  The context vector $\mvX$ is distributed as $(\iota |\mvPsi_1|,\mvPsi_2,\ldots, \mvPsi_{d})$, where $\mvPsi = (\mvPsi_1,\mvPsi_2,\ldots, \mvPsi_{d})$ has the $\textup{Unif}(\sqrt{d}\cS^{d-1})$ distribution, and $\iota$ takes value $+1$ and $-1$ with probability $p$ and $1-p$ respectively, independent from $\mvPsi$. Further,  $\epsilon^{(1)}, \epsilon^{(2)}$ are i.i.d.~$N(0, \sigma^2)$ random variables, independent from $\mvX$.
    \label{prob_instances_LinUCB}
\end{enumerate}

That is, for each context, with probability $p$ and $1-p$, respectively, it is uniformly distributed over the ``northern" and ``southern" hemisphere with radius $\sqrt{d}$ in $\bR^d$.

\begin{theorem}\label{them:lb_linucb}
Consider problem instances in \ref{prob_instances_LinUCB} with $p = 0.6, \sigma^2=1$, and the cumulative regret $R_T$ for the LinUCB algorithm, i.e., Tr-LinUCB with $S = T$, with $\lambda= m_{\theta} =  1$. Assume \eqref{def:regime} holds. There exists an absolute  positive constant $C$ such that $ R_T/(d^2\log^2(T)) \in [C^{-1},C]$.
\end{theorem}
\begin{proof}
See Appendix
 \ref{app:proof_lower_LinUCB}.
\end{proof}

We note that the problem instances in \ref{prob_instances_LinUCB} verify  conditions  \ref{assumption_parameter_noise}-\ref{cond:unit_sphere}, due to 
Theorem \ref{theorem:lower_bound} and since the context $\mvX$ has a density, relative to $\textup{Unif}(\sqrt{d} \cS^{d-1})$, that takes value in $[2(1-p),2p]$ if $p > 0.5$.
Thus by Corollary \ref{cor:regret_d}, the regret for Tr-LinUCB is $O(d\log(T))$ if $S = C d\log(T)$ for some sufficiently large constant $C$.
 


Next, we provide intuition for the $\Omega(d^2\log^2(T))$ regret of LinUCB, which explains its ``excessive" exploration. 
Recall from the discussions in Subsection \ref{subsec:optimal_d} that for $t \geq C_0 \Upsilon_{d,T}$
since the bonus part, $|\text{UCB}_t(k)-(\hat \mvtheta_{t-1}^{(k)})' \mvX_t|$, in the upper confidence bound, dominates the estimation error, we have  $|\text{UCB}_t(k) - \mvtheta_k'X_{t}| \asymp c_k (\upsilon_{d,t}/t)^{1/2}$ for $k \in [2]$, where $\upsilon_{d,t} := d\log(T)+d^2\log(t)$. If $p > 0.5$, then the \textit{proportion} of times arm $1$ selected is larger than arm $2$, and thus $c_1 < c_2$. As a result, on the event $\{0 < \mvtheta_1'X_{t} - \mvtheta_2'X_{t} < (c_2-c_1) (\upsilon_{d,t}/t)^{1/2}\}$, which occurs with a probability $\Omega((\upsilon_{d,t}/t)^{1/2})$, we have 
$\text{UCB}_t(1) < \text{UCB}_t(2)$, and a $\Omega((\upsilon_{d,t}/t)^{1/2})$ regret is incurred, which implies $\Omega(\upsilon_{d,t}/t)$ expected regret at time $t$, and $\Omega(d^2\log^2(T))$ cumulative regret.



\subsection{Discrete components in contexts} \label{subsec:discrete}
In Lemma \ref{lemma:continuity}
we show that  the condition \ref{cond:continuity} holds if the context vector $\mvX$ has a bounded Lebesgue density, after maybe removing the intercept.  In this section, we allow $\mvX$ to have both discrete and continuous components. In order not to over-complicate the proof, we assume the dimension $d$ \textit{fixed} in this subsection, and note that by similar arguments as for Theorem \ref{theorem:regret_main_T}, we could make the dependence on $d$ explicit, e.g., $O(d^2\log(2d)\log(T))$ with a properly chosen truncation time.

Suppose the context vector $\mvX = ((\mvX^{(\text{d})})', (\mvX^{(\text{c})})')'$, where $\mvX^{(\text{d})} \in \bR^{d_1}$, $\mvX^{(\text{c})} \in \bR^{d_2}$, and $d = d_1+d_2$ with $d_1,d_2 \geq 1$. Here, $\mvX^{(\text{d})}$ is a discrete random vector, with support $\cZ = \{ \mvz_1,\ldots, \mvz_{L_2} \} \subset \bR^{d_1}$. Further, we denote by $\bar{\mvX}^{(\text{c})} = (1,(\mvX^{(\text{c})})')'$, and assume that for some absolute constant $\ell_2>0$, \\

\noindent \mylabel{cond:discrete}{(C.IV')} 
For each $j \in [L_2]$,  $k \in [2]$, and $\bar{\mvu} \in \cS^{d_2}$, $\bP(|\bar{\mvu}'\bar{\mvX}^{(\text{c})}| \leq \ell_2 \;\vert\; \mvX^{(\text{d})} = \mvz_j) \leq 1/4$, and $\lambda_{\min}\left( \Exp\left[ \bar{\mvX}^{(\text{c})} (\bar{\mvX}^{(\text{c})})'I\left(\mvX \in \cU_{{\ell_2}}^{(k)},  \mvX^{(\text{d})} = \mvz_j \right)\right]\right) \geq \ell_2^2$.\\

The two parts in the above condition may be viewed as the conditional version of conditions \ref{cond:continuity} and \ref{cond:posdef}, given the value of the first $d_1$ components. By Lemma \ref{lemma:continuity}, the first condition  holds if for each $j \in [L_2]$, given $\{\mvX^{(\text{d})} = \mvz_j\}$, $\mvX^{(\text{c})}$ has a Lebesgue density on $\bR^{d_2}$ that is upper bounded by some constant $C > 0$. 
The second condition   requires, for each $j \in [L_2]$, that $\mvX^{(\text{d})}$ assumes $\mvz_j$ with a positive probability, and conditional on $\{\mvX^{(\text{d})} = \mvz_j\}$, $\mvX$ is optimal for each arm $k \in [2]$ by at least $\ell_2 > 0$ with a positive probability, and that $\bar{\mvX}^{(\text{c})}$ expands $\bR^{d_2+1}$ on the event $\{\mvX \in \cU_{{\ell_2}}^{(k)},  \mvX^{(\text{d})} = \mvz_j \}$.  Denote by $\mvTheta_2 := (\mvTheta_0\setminus \{\ell_1\}) \cup \{\ell_2, L_2\}$ the collection of parameters in conditions \ref{assumption_parameter_noise}-\ref{cond:posdef},  \ref{cond:discrete}, and  the size of support for the discrete components $\mvX^{\text{(d)}}$.


\begin{theorem}\label{theorem:discrete_regret_main_T} 
Consider problem instances for which conditions \ref{assumption_parameter_noise}-\ref{cond:posdef} and \ref{cond:discrete} hold, and the Tr-LinUCB algorithm with a fixed $\lambda > 0$.  Assume $d$ is fixed. (i).  There exist  a constant $C_0 > 0$, depending only on $\mvTheta_2,  d, \lambda$, such that if $S = C \log(T)$ for some  $C \geq C_0$, then $R_T\leq C_1 \log(T)$, where the constant $C_1$ depends only on $\mvTheta_2,d,   \lambda$, and $C$. (ii). If $S = \log^{\kappa}(T)$ for some $\kappa > 1$, then $R_T\leq C_1 \log(T)\log\log(T)$, where the constant $C_1$ depends only on $\mvTheta_2, d,  \lambda$, and $\kappa$.
\end{theorem}
\begin{proof}
See Appendix \ref{app:proof_discrete}. 
\end{proof}

\begin{remark}
By similar but longer arguments, we may allow that for a subset  $\tilde{\cZ} \subset \cZ$, if $\mvX^{(\text{d})} = \mvz \in \tilde{\cZ}$, one arm has a better reward than the other, regardless the value of $\mvX^{(\text{c})}$.
\end{remark}

Next, we indicate the key step in the proof of above Theorem. Note that if $d_1 \geq 2$, then $\Exp[\mvX \mvX'I(\mvX^{(\text{d})} = \mvz_j)]$ is not invertible, which motivates us to replace the first $d_1$ coordinates by a constant $1$, resulting in $\bar{\mvX}^{(\text{c})}$. The next lemma shows that if we cluster contexts based on the value of their discrete components $\mvX^{\text{(d)}}$, then we can deal with $\mvX^{\text{(d)}}$ in the same way as an intercept.

\begin{lemma}\label{aux:replaced_by_1}
Fix $\lambda > 0$ and let $n,d_1, d_2 \geq 1$ be integers.
Let $\mva \in \bR^{d_1}$, and $\mvz_{1},\ldots, \mvz_{n}$ be $\bR^{d_2}$-vectors. Define $\tilde{\mvz}_i = [\mva',\mvz_{i}']'$ and $\bar{\mvz}_i = [1,\mvz_{i}']'$ for each $i \in [n]$. For any $\mvv \in \bR^{d_2}$,
$$
\tilde{\mvv}' (\lambda \bI_{d_1+d_2} + \sum_{i=1}^{n} \tilde{\mvz}_i \tilde{\mvz}_i')^{-1} \tilde{\mvv} \;\;\leq\;\; \max(1,\|a\|^2)\  
\bar{\mvv}' (\lambda \bI_{1+d_2} + \sum_{i=1}^{n} \bar{\mvz}_i \bar{\mvz}_i')^{-1} \bar{\mvv},
$$
where $\tilde{\mvv} = [\mva',\mvv']'$ and $\bar{\mvv} = [1,\mvv']'$.
\end{lemma}
\begin{proof}
See Appendix \ref{proof:aux:replaced_by_1}.
\end{proof}

\begin{remark}\label{rk:replace_a}
Let $\tilde{\bV}_n = \lambda \bI_{d_1+d_2} + \sum_{i=1}^{n} \tilde{\mvz}_i \tilde{\mvz}_i'$. If $d_1 \geq 2$, then the smallest eigenvalue of $\tilde{\bV}_n$ does not grow with $n$, since $\tilde{\mvu}'\tilde{\bV}_n \tilde{\mvu} = \lambda$ for any $n \geq 1$, where $\tilde{\mvu} = (\mvu',\boldsymbol{0}_{d_2}')'$ and $\mvu \in \cS^{d_1-1}$ is any vector such that $\mvu'\mva = 0$. Note that if $d_1 = 1$ and $\mva = 1$, such $\mvu$ does not exist.

The above lemma implies that $\tilde{\mvv}' \tilde{\bV}_n^{-1} \tilde{\mvv}$ decays as $n$ increases for those $\tilde{\mvv} \in \bR^{d_1+d_2}$ such that the first $d_1$ components is $\mva$, which may not hold for general $\tilde{\mvv}$.
\end{remark}

\section{Experiments} \label{sec:simulation}
In this section, we conduct two simulation studies to compare the empirical  performance of the following algorithms: (i). the proposed Tr-LinUCB algorithm in Section \ref{sec:prob_formulation};\footnote{The implementation can be found at \url{https://github.com/simonZhou86/Tr_LinUCB}. The LinUCB algorithm corresponds to Tr-LinUCB with the truncation time $S=T$.} (ii). the LinUCB algorithm \citep{abbasi2011improved}; (iii). the OLS algorithm \citep{goldenshluger2013linear}; (iv). the Greedy-First algorithm \citep{bastani2021mostly}.\footnote{The implementation for \citet{bastani2021mostly} can be found at \url{https://github.com/khashayarkhv/contextual-bandits}. We used their implementation for the OLS algorithm and the Greedy-First algorithm. The only modification we made is that 
in simulationsynth.m, we set the intercept-scale variable on line 78 to 1, and remove the $/2$ part on line $112$ and $114$.}\\

\subsection{Synthetic Data} \label{sim_syn}
\noindent \textbf{Problem instances.} 
Except for Figure \ref{fig:varyT}, we consider the following setup, that matches the implementation in \citet{bastani2021mostly}.
The arm parameters $\{\mvtheta_k: k \in [K]\}$ are a random sample from the mixture of two $d$-dimensional normal distributions with equal weight, $2^{-1}N_d(\boldsymbol{1}_{d}, \bI_d)+ 2^{-1}N_d(-\boldsymbol{1}_{d}, \bI_d)$, where the first (resp. second) component has the mean vector $\boldsymbol{1}_{d}$ (resp. $-\boldsymbol{1}_{d}$), and both covariance matrices are the identity matrix.
For the context vector $\mvX$, its first component $\mvX^{(1)}$ is set to be $1$ (i.e., intercept), and the remaining $d-1$ components have the same distribution as $h(\mvZ)$, where $\mvZ$ has the $N_{d-1}(\boldsymbol{1}_{d-1}, 0.5\bI_{d-1})$ distribution, $h(x) = \min(\max(x,-1),1)$ for $x \in \bR$, and $h(\mvZ)$ means applying $h$ to each component in $\mvZ$. The observation noises $\{\epsilon_t^{(k)}: t \in [T], k \in [K]\}$ are i.i.d.~$N(0,\sigma^2)$ random variables with $\sigma^2 = 0.25$. The arm parameters, contexts, and noises are all independent. Further, each reported data point below is averaged over $1000$  realizations, where the arm parameters $\{\mvtheta_k: k \in [K]\}$ are also independently generated for each realization.\\

\noindent \textbf{Parameters.} For Tr-LinUCB, we set $\lambda=0.1$, $m_\theta = 1$, $\sigma^2 = 0.25$, and $S = Kd\log^{\kappa}(T)$ with $\kappa=2$. 
For LinUCB, we set $\lambda=0.1$, $m_\theta = 1$ and $\sigma^2 = 0.25$. For OLS, it requires the specification of exploration rate $q$ and sub-optimality gap $h$, and we set $q=1$ and $h=5$ following the implementation for \citet{bastani2021mostly}. For Greedy-First, from some time $t_0$ onward, it starts checking  whether the greedy algorithm fails, and if so, it transits into OLS; following the implementation for \citet{bastani2021mostly}, we set $t_0 = c_0 Kd$ with $c_0 = 4$, and $q=1, h =5$ for the OLS algorithm. These parameters are used in all studies, except for sensitivity analysis for $\kappa$ in Tr-LinUCB, $q,h$ in OLS, and $c_0$ in Greedy-First.

\begin{figure}[!tb]
\centering
\begin{subfigure}{0.48\textwidth}
\centering
    \includegraphics[width=6.5cm]{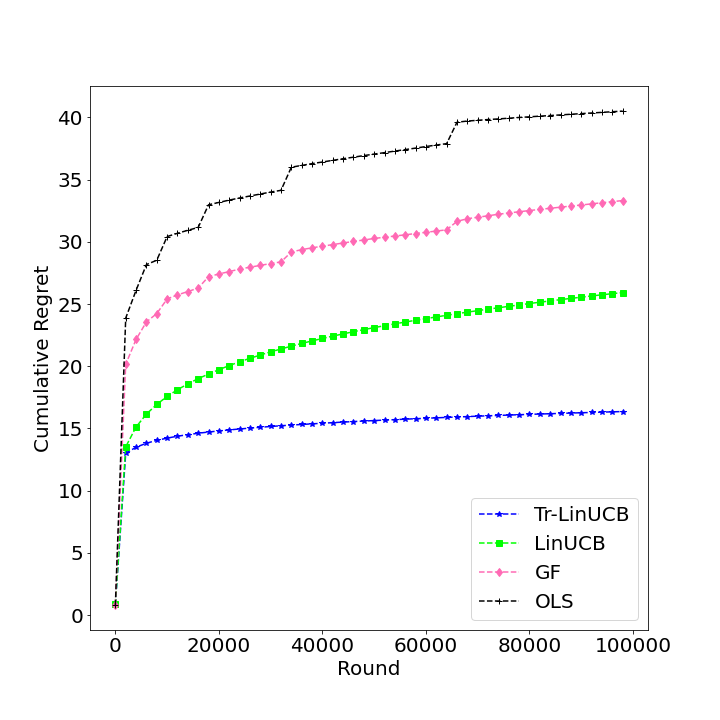}
    \caption{Cumulative regret from time $0$ to $T$ for $T=10^5$, $K=2$, $d=4$; ``GF" is  for Greedy-First.}
    \label{fig:cumu_reg}
\end{subfigure}
\hfill
\begin{subfigure}{0.48\textwidth}
\centering
    \includegraphics[width=6.5cm]{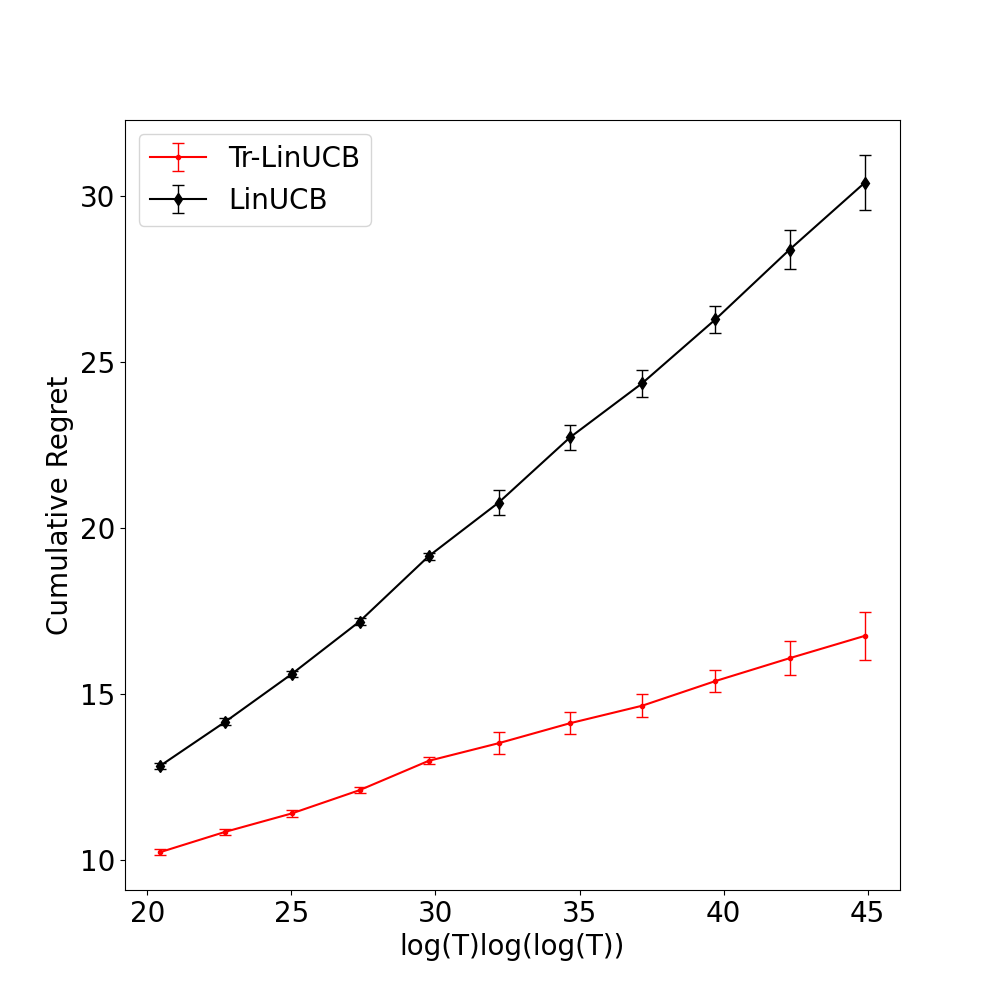}
    \caption{Cumulative regret $R_T$ with varying $T$ for problem instances in \ref{prob_instances_LinUCB}.}
    \label{fig:varyT}
\end{subfigure}
\caption{Cumulative regrets for different algorithms}
\label{mainRes}
\end{figure}



\begin{table}[!tb]
\centering
\begin{tabular}{ccccclcccc}
             & \multicolumn{4}{c}{$K=2$} &  & \multicolumn{4}{c}{$d=4$} \\ \cline{2-5} \cline{7-10} 
             & $d=4$        & $d=8$       & $d=15$       & $d=20$       &  & $K=5$        & $K=8$        & $K=10$       & $K=15$      \\ \cline{2-5} \cline{7-10} 
Tr-LinUCB    & 16.1       & 25.5      & 42.1       & 55.6       &  & 76.3       & 138.2      & 180.4      & 282.8     \\ \cline{1-5} \cline{7-10} 
LinUCB       & 25.7       & 31.7      & 46.3       & 57.9       &  & 113.6      & 198.0      & 250.8      & 366.9     \\ \cline{1-5} \cline{7-10} 
Greedy-First & 34.0       & 38.8      & 120.3      & 246.0      &  & 221.4      & 390.4      & 512.7      & 823.8     \\ \cline{1-5} \cline{7-10} 
OLS          & 43.0       & 62.6      & 111.9      & 219.0      &  & 197.8      & 351.0      & 481.4      & 749.1     \\ \hline
\end{tabular}
\smallskip
\caption{Cumulative regret $R_T$ for  algorithms with $T = 10^5$ and varying pairs of $K,d$.}
\label{tab:dk}
\end{table}

\noindent \textbf{Main Results.} In Table \ref{tab:dk}, we report the cumulative regret $R_T$ of the four algorithms with $T=10^5$ and varying pairs of $K$ and $d$. In Figure \ref{fig:cumu_reg}, we plot the cumulative regret over time (from $0$ to $T$)  of the four algorithms with $T=10^5$, $K=2$, and $d=4$. It is evident that, in terms of the cumulative regret, the proposed Tr-LinUCB algorithm performs favourably against others. Note that the gap between the performance of Tr-LinUCB and LinUCB gets smaller as the dimension $d$ increases. This does not contradict with our theoretical results, as we focus on the low dimensional regime, which requires $T$ to increase with $d$.

To compare the performance of Tr-LinUCB and LinUCB for large $T$, we consider problem instances in \ref{prob_instances_LinUCB} with $d=4$, $p=0.7$ and $\sigma^2 =0.25$. We plot the cumulative regret $R_T$ in Figure \ref{fig:varyT} for $T \in \{2^{i}\times10^4: i =0,\ldots,10\}$. Although we cannot conclude from the figure that the cumulative regret of LinUCB scales as $\log^2(T)$, the gap does become wider as $T$ increases.

\subsubsection{Sensitivity Analysis on Synthetic Data} \label{sen_ana_syn}
Next, we study the sensitivity of the algorithms to the tuning parameters, i.e., $\kappa$ in Tr-LinUCB, $q,h$ in OLS, and $c_0$ in Greedy-First, as discussed above. Note that we assume the noise variance $\sigma^2$ is known to Tr-LinUCB.

In Table \ref{tab:sen_all}, for $T = 10^5$, $K = 2$, and $d = 4$, we report the cumulative regret $R_T$ for the above three algorithms with different values of the tuning parameters. As expected, the proposed Tr-LinUCB algorithm is not too sensitive to overshooting, and in practice we recommend $S = Kd\log^2(T)$. On the other, the OLS algorithm is sensitive to both the exploration rate $q$ and sub-optimality gap $h$, and indeed $q=1$ and $h=5$ used in the above studies is a good configuration for OLS (for $T=10^5,K=2,d=4$). For the Greedy-First algorithm, it seems not too sensitive to the choice of $c_0$, but since it transits to OLS once it detects that the greedy algorithm fails, it inherits the same issue from OLS.

\begin{table}[!tb]
    \begin{subtable}[t]{1\textwidth}
        \centering
        \begin{tabular}{cccccccc}
         $\kappa=1.1$ & $\kappa=1.3$ & $\kappa=1.8$ & $\kappa=2.0$ & $\kappa=2.2$ & $\kappa=2.7$ & $\kappa=3.0$ & $\kappa=3.2$ \\ \hline
         16.9 & 15.9       & 16.0       & 16.5       & 16.8  & 18.4       & 19.6       & 20.9       \\         \hline
        \end{tabular}
       \caption{Tr-LinUCB with varying $\kappa$}
       \label{trlucb_sen}
    \end{subtable}
    
    \smallskip
    \begin{subtable}[t]{1\textwidth}
        \centering
        \begin{tabular}{cccccc}
        $c_0=0.5$ & $c_0=1.0$ & $c_0=5.0$ & $c_0=10.0$ & $c_0=20.0$ & $c_0=40.0$ \\ \hline
        42.3  & 39.2  & 30.5  & 31.6   & 35.6   & 37.7   \\ \hline
        \end{tabular}
        \caption{Greedy-First with varying $c_0$  ($q=1$ and $h=
        5$ for OLS)}
        \label{gf_sen}
     \end{subtable}

    \smallskip
    \begin{subtable}[t]{1\textwidth}
        \centering
        \begin{tabular}{ccccccccc}
        \multicolumn{4}{c}{$q = 1$}         &  & \multicolumn{4}{c}{$h = 5$}          \\ \cline{1-4} \cline{6-9} 
        $h = 1$     & $h = 3$    & $h = 5$    & $h = 9$    &  & $q = 2$    & $q = 3$    & $q = 5$     & $q = 9$     \\ \cline{1-4} \cline{6-9}
        239.5 & 44.9 & 39.2 & 38.4 &  & 32.3 & 78.8 & 117.3 & 191.7 \\ \hline
        \end{tabular}
        \caption{OLS with varying $q$ and $h$}
        \label{ols_sen}
     \end{subtable}
     \caption{Cumulative regret $R_T$ for different algorithms  with $K=2$, $d=4$, $T=10^5$.}
     \label{tab:sen_all}
\end{table}

\subsection{Real-World Data}

We now compare the performance of the proposed Tr-LinUCB algorithm with the other three competing algorithms on real-world datasets. As in \citet{bastani2021mostly}, we use the following healthcare-related datasets: (1)  Cardiotocography \footnote{\url{https://www.openml.org/search?type=data&sort=runs&id=1560&status=active}}, (2)  EEG \footnote{\url{https://archive.ics.uci.edu/ml/datasets/EEG+Eye+State}}, (3) EyeMovement \footnote{\url{https://www.openml.org/search?type=data&sort=runs&id=1044&status=active}}, and (4)  Warfarin dosing dataset \citep{international2009estimation, bastani2020online}.

\noindent \textbf{Problem Setup.} 
For the four datasets, we perform classification tasks using patient features, where the number of classes is treated as the number of arms $K$. For datasets (1)--(4), $K = 3, 2, 3, \text{ and } 3$, respectively. At each round $t \in [T]$, we observe a patient's features $X_t \in \bR^{d}$ and select an arm $A_t \in [K]$. We then receive a reward $Y_t \in \{0,1\}$, which equals 1 if $A_t$ matches the true label, and 0 otherwise.
The values of $(d,T)$ for datasets (1)--(4) are
$(35,2127), (14,14981), (27,10938), \text{ and } (93,5528)$. To ensure robustness, we conducted 100 trials with patients randomly permuted within each trial. We follow the same implementations and configurations of the Greedy-First and OLS algorithms as presented in \citet{bastani2021mostly}. Refer to our public codebase for details about the experiments for Tr-LinUCB and LinUCB presented in this section.


\noindent \textbf{Results.} 
We report the cumulative regret in Table~\ref{tab:Real} for four algorithms evaluated across four datasets. 
First, in both datasets (1) and (3), we observe that the proposed Tr-LinUCB algorithm outperforms the other methods by a substantial margin. 
For dataset (2), the Tr-LinUCB and Greedy-First algorithms exhibit similar performance. 
Compared to LinUCB, the cumulative regret is reduced by over 10\%.
Finally, for dataset (4), the OLS algorithm performs best, followed closely by Tr-LinUCB and Greedy-First. Notably, the class distribution is highly imbalanced, with 1835, 2992, and 701 patients in classes 0, 1, and 2, respectively. Due to limited data for class 2 during Tr-LinUCB's exploration phase, insufficient information may lead to higher regret during exploitation. 
Overall, the results of our experiments demonstrate the superiority of the Tr-LinUCB algorithm over existing methods in most cases and the crucial role of the truncation operation in mitigating the over-exploration problem. 

\begin{table}[!tb]
\centering
\begin{tabular}{c|c|c|c|c}
          & Cardiotocography (1)         & EEG (2)            & EyeMovement (3)     & Warfarin (4)        \\ \hline
Tr-LinUCB & \textbf{223.59} & \textbf{5398.16}         & \textbf{5715.49} & 2148.69          \\ \hline
Greedy-First        & 327.83          & 5412.10 & 6576.70           & 2143.40 \\ \hline
LinUCB      & 419.72                &  6056.62               &    6141.79              &   2190.20              \\ \hline
OLS       & 326.65          & 6012.60          & 6578.40           & \textbf{2122.1}          \\ \hline
\end{tabular}
\caption{Cumulative regret $R_T$ for different algorithms across four datasets, averaged over 100 trials.}
\label{tab:Real}
\end{table}

\subsubsection{Sensitivity Analysis on Real Data}
We now investigate the impact of the truncation time $S$, controlled by the tuning parameter $\kappa$, on the performance of Tr-LinUCB on real-world datasets. Cumulative regret is visualized as the fraction of misclassified samples at each time step $t \in [1, T]$. Figure~\ref{sen_realsim} provides a zoomed-in view over a shorter range for clarity.

As shown in Figure~\ref{sen_realsim}, the choice of $S$ has minimal effect on cumulative regret across all four datasets. This insensitivity to the tuning parameter is practically valuable and consistent with our theoretical findings.

\begin{figure}[!tb]
\centering
\begin{subfigure}{0.48\textwidth}
\centering
    \includegraphics[width=6.5cm]{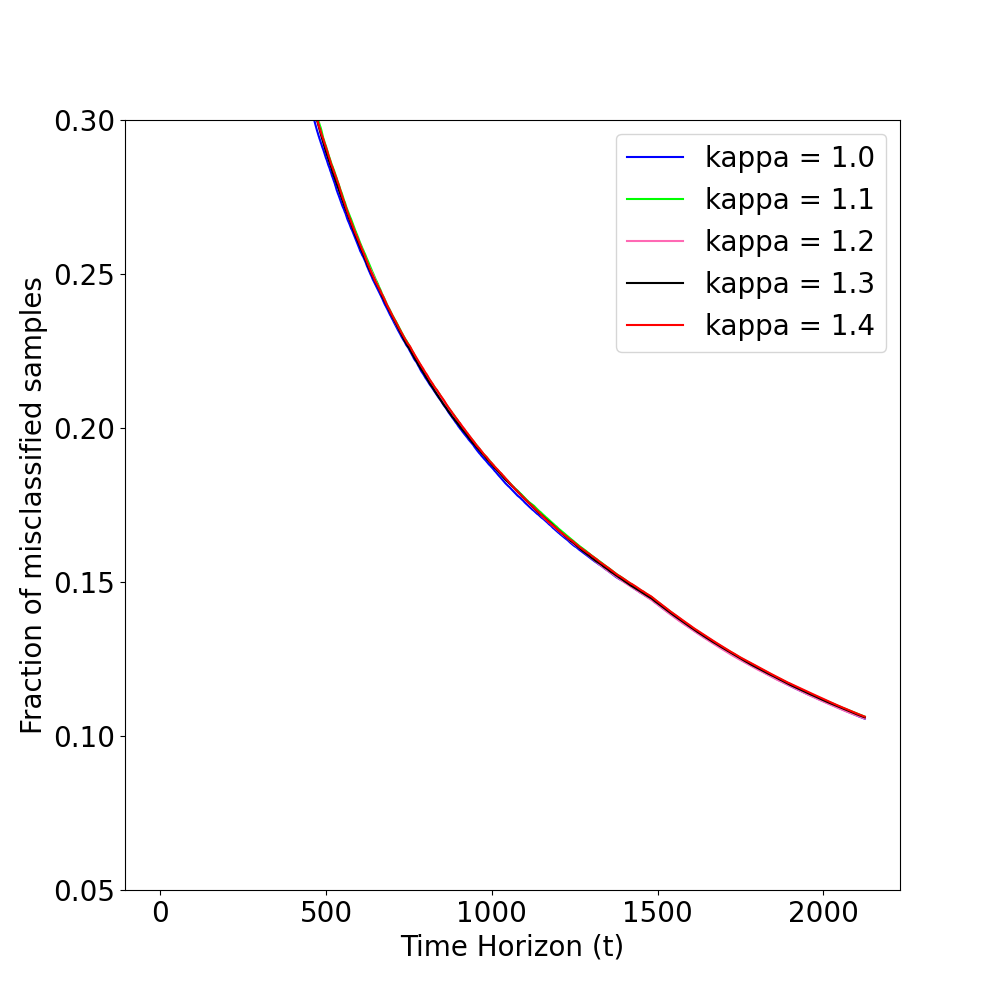}
    \caption{Dataset (1)}
    \label{fig:sen_ana_car}
\end{subfigure}
\hfill
\begin{subfigure}{0.48\textwidth}
\centering
    \includegraphics[width=6.5cm]{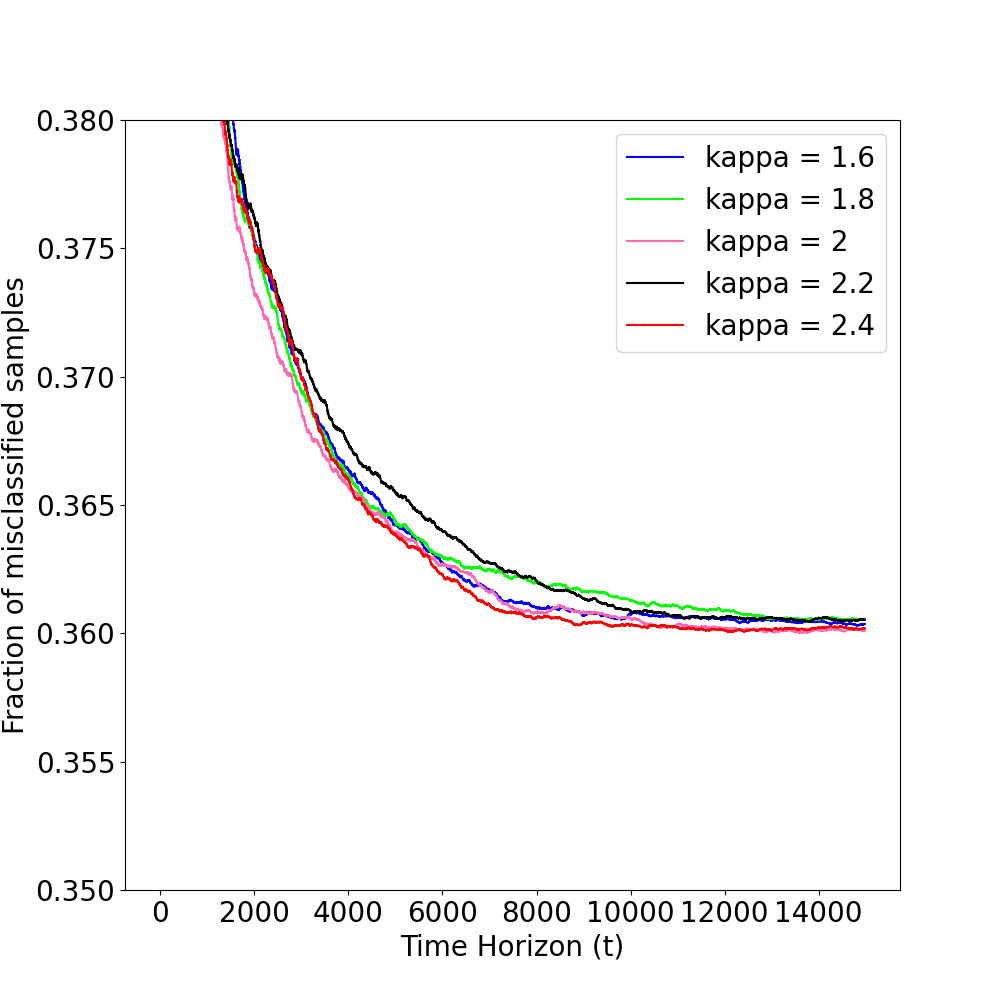}
    \caption{Dataset (2)}
    \label{fig:sen_ana_eeg}
\end{subfigure}
\begin{subfigure}{0.48\textwidth}
\centering
    \includegraphics[width=6.5cm]{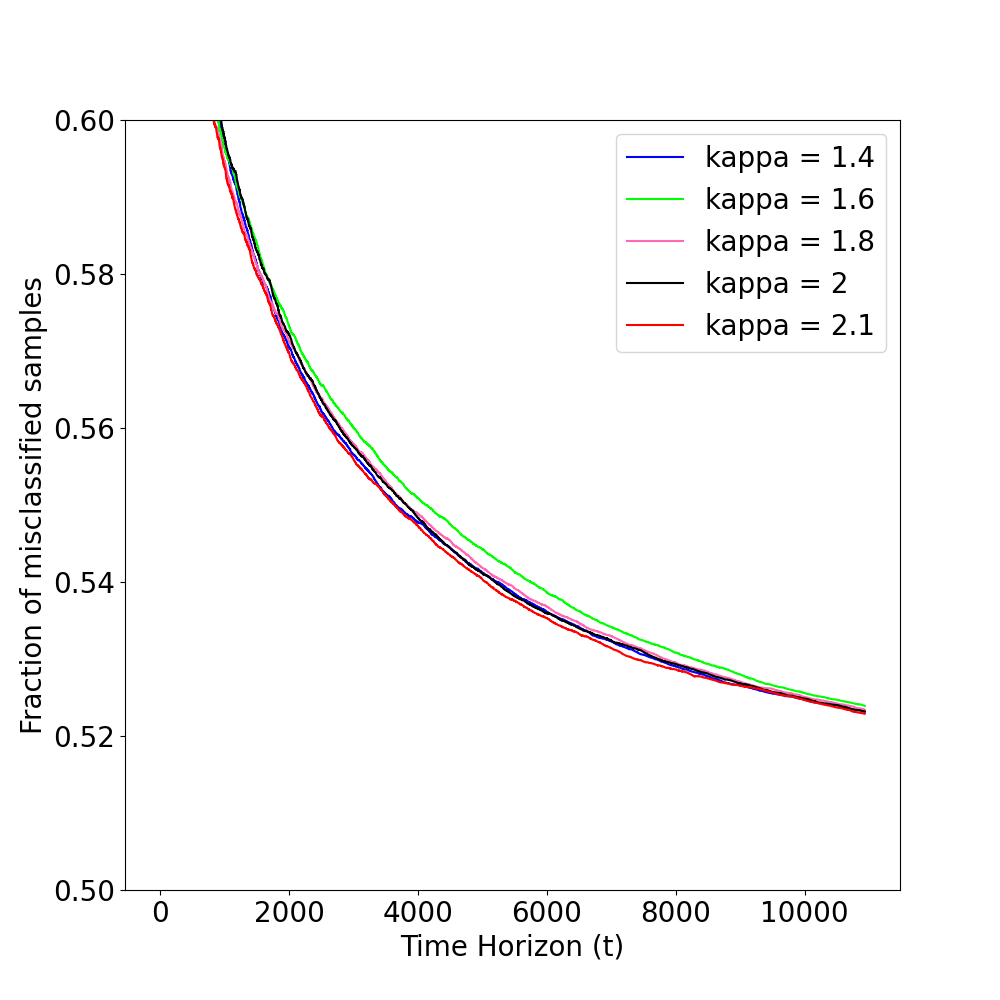}
    \caption{Dataset (3)}
    \label{fig:sen_ana_eye}
\end{subfigure}
\hfill
\begin{subfigure}{0.48\textwidth}
\centering
    \includegraphics[width=6.5cm]{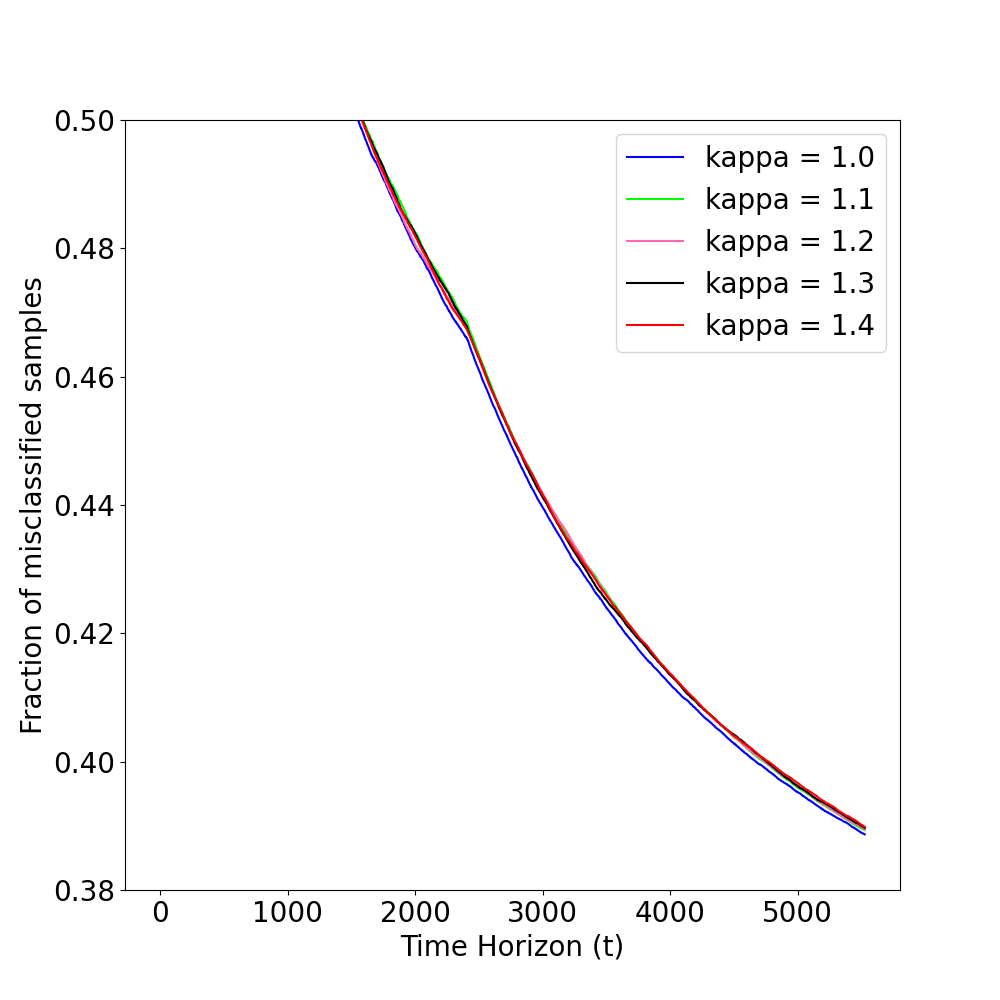}
    \caption{Dataset (4)}
    \label{fig:sen_ana_war}
\end{subfigure}
\caption{Sensitivity analysis of the tuning parameter $\kappa$ in Tr-LinUCB on  datasets (1)–(4).}
\label{sen_realsim}
\end{figure}

\section{Upper bound for Tr-LinUCB: proofs of  Theorem \ref{theorem:regret_main_T} and  \ref{theorem:regret_main_d}} \label{proof:regret_main_Td}
Recall  $\Upsilon_{d,T}$ in \eqref{chi_order}. First, we show that as long as the truncation time $S \geq C_0 \Upsilon_{d,T}$, for a large enough $C_0$, then with a high probability, at any time $t \geq C_0 \Upsilon_{d,T}$,  the smallest eigenvalues of the ``design" matrices are $\Omega(t)$. 
Thus, although the sequential decisions make the observations dependent across time, due to the i.i.d.~contexts, the Tr-LinUCB algorithm is able to accumulate enough information for each arm, that is of the same order as for independent observations.

Define, for each $t \in [T]$ and $k \in [2]$, the following events
\begin{align}\label{def:info_event}
    \cE_t^{(k)} = \{\lambda_{\min}\left( \bV_{t}^{(k)}\right) \geq 4^{-1}\ell_*^2 t\}, \;\;\text{ where }\;\; \ell_{*} := \min\{\ell_1,\ell_0\}/3.
\end{align}


\begin{lemma}\label{lemma:small_eignenvalues_anyt}
Assume that conditions \ref{assumption_parameter_noise}, \ref{cond:posdef} and \ref{cond:continuity} hold. There exists a constant $C_0 \geq 1$, depending only on $\mvTheta_0$, $\lambda$, such that if $S \geq C_0 \Upsilon_{d,T}$, then with probability at least $1-4d/T$, the event $\cap_{k=1}^{2}  \cE_t^{(k)}$ occurs for each  $t \geq C_0 \Upsilon_{d,T}$.
\end{lemma}
\begin{proof}
We present the proof, as well as discussions on the strategy, in  Section \ref{section:proof_key_lemma}.
\end{proof}

Second, we show that if the smallest eigenvalues of the ``design" matrices are large,  the estimation of arm parameters is accurate. 
In the following lemma, for each arm, the first result  establishes an \textit{exponential} bound on the tail probability of the estimation error, $\|\hat \mvtheta_t^{(k)} - \mvtheta_k\|$, while the second result provides an upper bound on its  \textit{second moment}.  


\begin{lemma}\label{lemma:beta_V}
Assume that the condition \ref{assumption_parameter_noise}  holds. Then there exists a constant $C_2 \geq 1$, depending only on $\mvTheta_0$, $\lambda$, such that  for any $t \in [T]$, $k \in [2]$,   $\tau \geq 0$, 
\begin{align*}    &\bP(\|\hat \mvtheta_t^{(k)} - \mvtheta_k \| \geq C_2 (d\log(2d)/t)^{1/2}\tau,\;\; \mathcal{E}_{t}^{(k)}) \leq 2\exp(-\tau^2),
\\    &\Exp\left[\|\hat \mvtheta_t^{(k)} - \mvtheta_k\|^2 I( \mathcal{E}_{t}^{(k)}) \right] \leq C_2 d/t.
\end{align*}
\end{lemma}

\begin{proof}
See Appendix \ref{proof:lemma:beta_V}.
\end{proof}

Next, we  prove Theorem \ref{theorem:regret_main_T}, by considering the three periods of the Tr-LinUCB algorithm.  Note that the peeling argument for the period  \textit{after the truncation time} $S$ is similar to that in \citet{bastani2021mostly}, but uses an improved exponential tail bound in Lemma \ref{lemma:beta_V}.

\begin{proof}[Proof of Theorem \ref{theorem:regret_main_T}]
In this proof, $C$ is a constant, depending only on $\mvTheta_0$ and $\lambda$, that may vary from line to line.  Let $C_0$ be the constant in 
 Lemma \ref{lemma:small_eignenvalues_anyt}, and
recall that $S_0 = \lceil C_0 \Upsilon_{d,T} \rceil$ with 
$\Upsilon_{d,T}$ defined in \eqref{chi_order}, and that the truncation time $S \geq S_0$. For each $t \in [T]$, define
\begin{align*}
    \tilde{\cE}_{t} = \cap_{k=1}^{2} \{\ \|\hat \mvtheta_{t}^{(k)} - \mvtheta_{k}\|_{\bV_{t}^{(k)}} \leq \sqrt{\beta_{t}^{(k)}},\;\;\; \cE_{t}^{(k)} \},
\end{align*}
where the event $\cE_{t}^{(k)}$ is defined in \eqref{def:info_event}. 
 By Lemma \ref{lemma:LinUCB} and \ref{lemma:small_eignenvalues_anyt},   with probability at least $1-(2+4d)/T$, the event $\tilde{\cE}_{t}$ occurs for each $t \geq S_0$. First, we consider the expected instant regret at time $t$, $\hat{r}_t$ in \eqref{def:regret}, for some fixed $t \in [T]$.\\

\noindent \underline{Case 1: $t \leq {S}_0$. } Due to the condition \ref{assumption_parameter_noise}, $\Exp[\hat{r}_t] \leq \Exp[|\mvtheta_1 - \mvtheta_2)'\mvX_t|] \leq 2  m_{R}$.\\

\noindent \underline{Case 2:  $S_0 < t \leq S$. }  For each $k \in [2]$, since $t \leq S$, i.e., prior to truncation,
\begin{align*}
\{A_t= k\} =  \{(\hat \mvtheta_{t-1}^{(k)})' \mvX_t + \sqrt{\beta_{t-1}^{(k)}} \|\mvX_t\|_{(\bV_{t-1}^{(k)})^{-1}} \geq  (\hat \mvtheta_{t-1}^{(\bar k)})' \mvX_t + \sqrt{\beta_{t-1}^{(\bar k)}} \|\mvX_t\|_{(\bV_{t-1}^{(\bar k)})^{-1}},
\}
\end{align*}
where  $\bar{k} = 3 -k$, i.e., $\bar{k} = 1$ (resp.~$2$) if $k=2$ (resp.~$1$). As a result, on the event $\{A_t = k\} \cap \tilde{\cE}_{t-1}$,  the ``potential regret" $\mvtheta_{\bar k}'\mvX_t - \mvtheta_{k}'\mvX_t$ can be upper bounded by
\begin{align*}
&(\mvtheta_{\bar{k}} - \hat \mvtheta_{t-1}^{(\bar{k})})' \mvX_t - (\mvtheta_{{k}}- \hat \mvtheta_{t-1}^{({k})})' \mvX_t + \sqrt{\beta_{t-1}^{(k)}} \|\mvX_t\|_{(\bV_{t-1}^{(k)})^{-1}} - \sqrt{\beta_{t-1}^{(\bar k)}} \|\mvX_t\|_{(\bV_{t-1}^{(\bar k)})^{-1}}\\
\leq & 2\sqrt{\tilde \beta_{t-1}} \left(\|\mvX_t\|_{(\bV_{t-1}^{(k)})^{-1}} +
\|\mvX_t\|_{(\bV_{t-1}^{(\bar k)})^{-1}}
\right) \leq C (\sqrt{\log(T) + d\log(t)} (d/t)^{1/2} := \tilde{\delta}_0,
\end{align*}
where recall the definition of $\tilde \beta_{t-1}$ in \eqref{eq:uuper_beta_t}, and the last inequality is because $\|\mvX_t\| \leq C \sqrt{d}$ by the condition   \ref{assumption_parameter_noise}, and $\lambda_{\min}\left( \bV_{t-1}^{(i)}\right) \geq C^{-1} t$ for $i \in [2]$ due to the definition of the event  $\tilde{\cE}_{t-1}$. As a result,  the regret, if incurred, is at most $\tilde{\delta}_0$, which implies that
\begin{align*}
   \Exp[\hat r_t] \;\;\leq \;\; & 2 \sqrt{d} m_{X} m_{\theta} \bP(\tilde{\cE}_{t-1}^c)  + \sum_{k=1}^{2} \tilde{\delta}_0 \bP(\{A_t = k\} \cap \{0 \leq \mvtheta_{\bar k}'\mvX_t - \mvtheta_{k}'\mvX_t \leq \tilde{\delta}_0\} \cap \tilde{\cE}_{t-1})\\
   \leq & \;\;Cd^{1.5}/T + \sum_{k=1}^{2} \tilde{\delta}_0 \bP(| \mvtheta_{\bar k}'\mvX_t - \mvtheta_{k}'\mvX_t| \leq \tilde{\delta}_0\}) \leq
   C d^{1.5}/T + C \tilde{\delta}_0^2,
\end{align*}
where the last inequality is due to the condition \ref{cond:margin}. Thus 
$$\Exp[\hat r_t] \leq  Cd\log(T)/t + Cd^{2}\log(t)/t.$$

\noindent \underline{Case 3:  $t > S$. } 
Let $C_2$ be the constant in Lemma \ref{lemma:beta_V} and  $\delta_0 := \sqrt{d}m_X C_2 (d\log(2d)/t)^{1/2}$, and
for $k \in [2]$ and $n \in \bN$, 
$
D_{n,k} :=\{2n\delta_0 < \mvtheta_{\bar{k}}' \mvX_t - \mvtheta_{{k}}' \mvX_t \leq 2(n+1)\delta_0\}
$, the event that arm $\bar{k}$ is better than the arm $k$ by an amount between $(2n\delta_0,2(n+1)\delta_0]$.

A regret is incurred if the arm $k$ is selected, but the arm $\bar{k}$ is in fact better. Thus we have the following:
$\hat{r}_t \leq 2 \sqrt{d} m_{X} m_{\theta}  I(\tilde{\cE}_{t-1}^c) + \sum_{k \in [2]} \sum_{n \in \bN} 2(n+1)\delta_0 I(A_{t} = k ,\, D_{n,k}, \, \tilde{\cE}_{t-1})$. Since $\|\mvX_t\| \leq \sqrt{d}m_X$
due to the condition \ref{assumption_parameter_noise} and by the definition of $\delta_0$, for each $k \in [2]$,
\begin{align*}
&\{A_t = k\}  \cap  D_{n,k} \;\;\subset\;\; \{(\hat \mvtheta_{t-1}^{(k)})' \mvX_t \geq (\hat \mvtheta_{t-1}^{(\bar{k})})' \mvX_t,\;  2n\delta_0 < \mvtheta_{\bar{k}}' \mvX_t - \mvtheta_{{k}}' \mvX_t\} \cap D_{n,k} \\ 
    \subset\;\; &\{ (\mvtheta_{\bar{k}} - \hat \mvtheta_{t-1}^{(\bar{k})})' \mvX_t - (\mvtheta_{{k}}- \hat \mvtheta_{t-1}^{({k})})' \mvX_t > 2n\delta_0\} \cap  D_{n,k} \\
 \subset\;\; &\left( \cup_{k \in [2]}\  \left\{\left\|\mvtheta_{k} - \hat \mvtheta_{t-1}^{(k)} \right\| \geq  C_2 (d\log(2d)/t)^{1/2} n \right\}\right) \; \cap \; D_{n,k}.
\end{align*}
Since $\mvX_t$, and thus $D_{n,k}$, is independent from $\cF_{t-1}$, and both $\hat \mvtheta_{t-1}^{(k)}$ and $\tilde{\cE}_{t-1}$ are $\cF_{t-1}$ measurable,  by  Lemma \ref{lemma:beta_V}, for each $n \in \bN$,
\begin{align*}
\bP(A_{t} = k ,\, D_{n,k}, \, \tilde{\cE}_{t-1}) \;\; \leq \;\;  4e^{-n^2}\bP(D_{n,k}) \;\; \leq \;\;4e^{-n^2} (L_0  2(n+1)\delta_0),
\end{align*}
where the last inequality is due to the condition \ref{cond:margin}.
Thus we have
\begin{align*}
    \Exp[\hat{r}_t] &\leq Cd^{1.5}/T + C   \delta_0^2 \sum_{n=0}^{\infty}  (n+1)^2 e^{-n^2} \leq C d^{2}\log(2d)/t.
\end{align*}

\noindent \underline{Sum over $t \in [T]$. } Now we combine the three cases. 
For integers $m > n \geq 3$, $\sum_{s=n+1}^{m} s^{-1} \leq \log(m/n)$, and  $\sum_{s=n+1}^{m} \log(s)/s \leq \log(m/n)\log(m)$. Thus
\begin{align*}
    R_T \leq & C  S_0 + C \sum_{t=S_0+1}^{S} (d \log(T)/t + d^{2}\log(t)/t) + C \sum_{t=S+1}^{T} d^{2}\log(2d)/t\\
    \leq & C  S_0 + C d \log(T)\log(S/S_0) + Cd^{2} \log(S)\log(S/S_0) + Cd^{2}\log(2d)\log(T/S),
\end{align*}
which completes the proof.
\end{proof}

Finally, we prove Theorem \ref{theorem:regret_main_d}, which relies on the condition \ref{cond:unit_sphere} and the second result in Lemma \ref{lemma:beta_V} for the period  \textit{after the truncation time} $S$. In Figure \ref{fig:regret_analysis}, we depict the order of expected  instant regret within each of the three periods.

\begin{proof}[Proof of Theorem \ref{theorem:regret_main_d}]
In this proof, $C$ is a constant, depending only on $\mvTheta_1$, $\lambda$, that may vary from line to line.  Let $C_0$ be the constant in 
 Lemma \ref{lemma:small_eignenvalues_anyt}, and
recall that $S_0 = \lceil C_0 \Upsilon_{d,T} \rceil$ with 
$\Upsilon_{d,T}$  in \eqref{chi_order}, and that the truncation time $S \geq S_0$. Recall the definition of ${\cE}_t^{(k)}$ in \eqref{def:info_event}, and by  Lemma \ref{lemma:small_eignenvalues_anyt}, with probability at least $1-4d/T$, the event ${\cE}_{t}^{(1)} \cap {\cE}_{t}^{(2)}$ occurs for each $t \geq S_0$. As in the proof of Theorem \ref{theorem:regret_main_T}, first, we consider the expected instant regret at time $t$, $\hat{r}_t$ in \eqref{def:regret}, for some fixed $t \in [T]$.

If $t \leq S_0$, by the condition \ref{assumption_parameter_noise}, $\Exp[\hat{r}_t] \leq C$. For  $S_0 < t \leq S$, in the proof of Theorem \ref{theorem:regret_main_T} above, we have shown that $\Exp[\hat{r}_t] \leq C d\log(T)/t + Cd^{2}\log(t)/t$.\\

Now we focus on $t > S$. Let
$\mvDelta = \mvtheta_1 - \mvtheta_2$ and $\hat \mvDelta_{t-1} = \hat \mvtheta_{t-1}^{(1)}-\hat \mvtheta_{t-1}^{(2)}$. 
By the condition \ref{cond:unit_sphere}, $\|\mvDelta\|\geq L_1^{-1}$. Note that $\mvX_t$ is independent from $\cF_{t-1}$, and that $\hat \mvDelta_{t-1}$ and $\cE_{t-1}^{(k)}, k \in [2]$ are both $\cF_{t-1}$-measurable. Then, due to \ref{assumption_parameter_noise}, for each $k \in [2]$,
\begin{align*}
    \Exp[\hat{r}_t \ I((\cE_{t-1}^{(k)})^c)] \; \leq \; \Exp[|(\mvtheta_1-\mvtheta_2)'\mvX_t|] \ \bP((\cE_{t-1}^{(k)})^c) \;\leq \; C d/T.
\end{align*}
Further, on the event  $\{\hat \mvDelta_{t-1} \neq \boldsymbol{0}_d\}$, 
by the condition \ref{cond:unit_sphere} with   $\mvv =\hat \mvDelta_{t-1}/\|\hat \mvDelta_{t-1}\|$,
\begin{align*}
&\Exp[\hat{r}_t \vert \cF_{t-1}]   = \|\mvDelta\|\times \Exp[|\mvu_{*}'\mvX_{t}|I(\text{sgn}(\mvu_{*}' \mvX_t) \neq \text{sgn}(\mvv' \mvX_t))\, \vert \, \cF_{t-1}] \\
\leq &  C    \|\mvDelta\| \left\| \frac{\mvDelta}{\|\mvDelta\|} - \frac{\hat \mvDelta_{t-1}}{\|\hat \mvDelta_{t-1}\|} \right\|^2
\leq C  \|\mvDelta\|^{-1} \|\mvDelta-\hat \mvDelta_{t-1}\|^2,
\end{align*}
where the last inequality is due to Lemma \ref{lemma:triangle} in Appendix \ref{app:elementary_lemmas}. 
On the event  $\{\hat \mvDelta_{t-1} = \boldsymbol{0}_d\}$,  $\Exp[\hat{r}_t\vert \cF_{t-1}] \leq C  \|\mvDelta-\hat \mvDelta_{t-1}\|^2$ due to conditions \ref{assumption_parameter_noise} and \ref{cond:unit_sphere} (i.e.,~$\|\mvDelta\|\geq L_1^{-1}$).  
Thus, 
\begin{align*}
    \Exp[\hat r_t] \leq& Cd/T + C   \Exp[ \|\mvDelta-\hat \mvDelta_{t-1}\|^2 I({\cE}_{t-1}^{(1)} \cap {\cE}_{t-1}^{(2)})] \\
    \leq& Cd/T + C  \sum_{k=1}^{2}\Exp[\|\hat \mvtheta_{t-1}^{(k)}-\mvtheta_k\|^2  I(\cE_{t-1}^{(k)} )]\leq  C d/t.
\end{align*}
Combining three cases, by a similar calculation as before, we have
\begin{align*}
      R_T \leq & C  S_0 + C \sum_{t=S_0+1}^{S} (d\log(T)/t + d^{2}\log(t)/t) + C \sum_{t=S+1}^{T} d /t\\
    \leq & C d \log(T)\log(2S/S_0) +  C d^2 \log(S)\log(2S/S_0),
\end{align*}
where the last line is due to the definition of $\Upsilon_{d,T}$ in \eqref{chi_order}.  Then the proof is complete.
\end{proof}

\subsection{Proof of Lemma \ref{lemma:small_eignenvalues_anyt}}\label{section:proof_key_lemma}

We preface the proof with a discussion on the strategy. First, we show that for a large enough $C$, at time  $T_0 = \left\lceil C \Upsilon_{d,T} \right\rceil$, \textit{at least one arm} has accumulate enough information, in the sense that the smallest eigenvalue of its ``design matrix" is $\Omega(T_0)$. This fact is due to condition  \ref{cond:continuity}, and stated formally in Lemma \ref{lemma:cst_adaptive}.

Second, if the truncation time $S \geq 2T_0$, we show that at time $2T_0$, both arms have accumulated enough information, that is, the smallest eigenvalues of both ``design matrices" are $\Omega(T_0)$. To gain intuition,  assume that at time $T_0$, it is the first arm that can be accurately estimated, i.e., $\lambda_{\min}\left( \bV_{T_0}^{(1)}\right)  \geq cT_0$ for some $c>0$. Then for $t \in (T_0,2T_0]$, the upper confidence bound $\text{UCB}_t(1)$ is closed to $\mvtheta_1'\mvX_t$. Due to Lemma \ref{lemma:LinUCB}, $\text{UCB}_t(2) \geq \mvtheta_2'\mvX_t$ with a large probability, and thus if $\mvX_t \in \cU_{c'}^{(2)}$ for some small $c' > 0$, which happens with a positive probability for each $t$ due to the condition \ref{cond:posdef}, then the second arm would be chosen by definition.

Finally, we use induction to show that at any time $t \geq 2T_0$, the smallest eigenvalues of the ``design matrices" are at least  $\Omega(t)$, ``bootstrapping" the result at time $2T_0$, which would conclude the proof.

Recall $\tilde{\beta}_t$ in \eqref{eq:uuper_beta_t},  $\Upsilon_{d,T}$ in \eqref{chi_order}, and  $\ell_{*} = \min\{\ell_1,\ell_0\}/3$.

\begin{proof}[Proof of Lemma \ref{lemma:small_eignenvalues_anyt}]
\underline{Step 1.} By Lemma 
\ref{lemma:LinUCB},  and Lemma \ref{lemma:cst_adaptive}, \ref{lemma:posdef} and \ref{lemma:beta_t} (ahead), there exists a constant $C$, depending only on $\mvTheta_0$, $\lambda$, such that  
the event $\cA := \cA_1 \cap \cA_2 \cap \cA_3 \cap \cA_4$ happens with probability at least $1 - 4d/T$, where
\begin{align}\label{events:def}
\begin{split}
    \cA_1 &= \{
    \|\hat \mvtheta_{t}^{(k)} - \mvtheta_{k}\|_{\bV_{t}^{(k)}} \leq \sqrt{\beta_t^{(k)}}: \text{ for all } t \in [T], k \in [K]\},\\
    \cA_2 &=  \{ \max_{k=1,2} \lambda_{\min}\left( \bV_t^{(k)}\right)  \geq 6^{-1}\ell_1^2 t, \text{ for all } t \in [C(d +\log(T)),T] \}, \\
    \cA_3 &= \{ \lambda_{\min}\left( \sum_{s=t_1+1}^{t_2}  \mvX_s \mvX_s' I(\mvX_s \in \cU_{\ell_0}^{(k)})\right) \geq {\ell_0^2 (t_2-t_1)}/{2}, \\
    &\quad \qquad \text{ for any } t_1,t_2 \in [T], \text{ with } t_2 - t_1 \geq C d \log(T), \text{ and  } k = 1,2\},\\
    \cA_4 &= \{\sqrt{\tilde{\beta}_{t}} \left( \ell_{*}^2t/2  \right)^{-1/2} (\sqrt{d} m_{X}) \leq  {\ell_0}/8, \;\; \text{ for all } \;t \geq C \Upsilon_{d,T}\}.
\end{split}
\end{align}
We recall $\tilde{\beta}_{t} \geq \beta_t^{(k)}$ for each $k \in [K]$ and $t \in [T]$, and note that $\cA_4$ in fact involves no randomness.  Define $T_0 = \left\lceil C \Upsilon_{d,T} \right\rceil$. We show below that if the truncation time $S \geq 2T_0$, on the event $\cA$, $\min_{k = 1,2} \lambda_{\min}\left( \bV_{t}^{(k)}\right) \geq 4^{-1}\ell_*^2 t/d$ for each $t \in [T]$ and $t  \geq 2T_0$; that is, the Lemma \ref{lemma:small_eignenvalues_anyt} holds with $C_0 = 2C+1$. Thus, assume $S \geq 2T_0$ and focus on the event $\cA$. \\

\noindent \underline{Step 2.} show that on the event $\cA$, $\min_{k = 1,2} \lambda_{\min}\left( \bV_{2T_0}^{(k)}\right) \geq \ell_*^2T_0$.

On the event $\cA_2$,  one of the following   holds: (I) $\lambda_{\min}\left( \bV_{T_0}^{(1)}\right)  \geq \ell_{*}^2 {T_0}$ or (II) $\lambda_{\min}\left( \bV_{T_0}^{(2)}\right)  \geq \ell_{*}^2 T_0$.  We first consider case (I), and in particular the conclusion holds for arm 1. For each $t \in [T_0 +1 , 2T_0]$, since $\cA_1$,  $\cA_2$, and $\cA_4$ (using  $t = 2T_0$) occur, we have  $\mvtheta_2'\mvX_t \leq  \text{UCB}_t(2)$ and
\begin{align*}
 \text{UCB}_t(1) = &(\hat \mvtheta_{t-1}^{(1)})' \mvX_t + \sqrt{\beta_{t-1}^{(1)}} \|\mvX_t\|_{(\bV_{t-1}^{(1)})^{-1}} \leq  \mvtheta_1'\mvX_t + 2\sqrt{\beta_{t-1}^{(1)}} \|\mvX_t\|_{(\bV_{t-1}^{(1)})^{-1}} \\
 \leq &\mvtheta_1'\mvX_t + 2 \sqrt{\tilde{\beta}_{2T_0}} \left(\lambda_{\min}\left( \bV_{T_0}^{(1)}\right)\right)^{-1/2} (\sqrt{d} m_X) 
 \leq \mvtheta_1'\mvX_t + \ell_0/4.
\end{align*}
Since the truncation time $S \geq 2T_0$, if $\mvX_t \in \cU^{(2)}_{\ell_0}$, i.e., $\mvtheta_1'\mvX_t + \ell_0 < \mvtheta_2'\mvX_t$, then we must have $A_t = 2$, since arm 2 has a larger upper confidence bound than arm 1. Further, since $\cA_3$ occurs, we have
 \begin{align*}
     \lambda_{\min}\left( \bV_{2T_0}^{(2)}\right)  \geq 
      \lambda_{\min}\left(\sum_{t=T_0+1}^{2T_0} \mvX_t \mvX_t'  I(\mvX_t \in \cU_{\ell_0}^{(2)})\right)  \geq \ell_*^2 T_0.
 \end{align*}
 The same argument applies to the case (II), and the proof for Step 2 is complete.\\

\noindent \underline{Step 3.} show that on the event $\cA$, for each  $t \geq 2T_0$, $\min_{k = 1,2} \lambda_{\min}\left( \bV_{t}^{(k)}\right) \geq 4^{-1}\ell_*^2 t$.

 It suffices to show that 
 \begin{align}\label{aux:enough_info_anyt}
     \min_{k = 1,2} \lambda_{\min}\left( \bV_{n(2T_0)}^{(k)}\right) \geq \ell_*^2 nT_0, \;\text{ for all } n \in \bN_{+} \;\text{ and } \;2n T_0 \leq T,
 \end{align}
 as it would imply that if $t \in [2nT_0, 2(n+1)T_0)$ for some $n \in \bN_{+}$, since $n/(2(n+1)) \geq 4^{-1}$, we would have  
 $\min_{k = 1,2} \lambda_{\min}\left( \bV_{t}^{(k)}\right)
 \geq  \ell_*^2 nT_0 \geq 4^{-1} \ell_*^2 t$. Next we use induction to prove \eqref{aux:enough_info_anyt}, and note that the case $n = 1$ is shown in Step 2. Thus assume  \eqref{aux:enough_info_anyt} holds for some $n\in \bN_{+}$.
 
 
 Let $\cI_t(k) = (\hat \mvtheta_{t-1}^{(k)})' \mvX_t + \sqrt{\beta_{t-1}^{(k)}} \|\mvX_t\|_{(\bV_{t-1}^{(k)})^{-1}} I(t \leq S)$ be the index for arm $k$ at time $t$, which is equal to $\text{UCB}_t(k)$ if $t \leq S$, and $(\hat \mvtheta_{t-1}^{(k)})' \mvX_t$ otherwise.  On the event $\cA_1$ and $\cA_2$, and by induction in \eqref{aux:enough_info_anyt}, for each $t \in (2nT_0, 2(n+1)T_0]$ and $k  = 1,2$,
 \begin{align*}
      |\cI_t(k)  - \mvtheta_k'\mvX_t| \leq 2\sqrt{\tilde{\beta}_{t-1}} \|\mvX_t\|_{(\bV_{t-1}^{(k)})^{-1}} \leq   2\sqrt{\tilde \beta_{2(n+1)T_0}}  \left(\ell_*^2 n T_0\right)^{-1/2} (\sqrt{d} m_X).
 \end{align*}
Due to 
the event $\cA_4$ with $t = 2(n+1)T_0$, and since $\sqrt{(n+1)/n} \leq \sqrt{2}$,  we have
$|\cI_t(k) - \mvtheta_k'\mvX_t|\leq\ell_0/(2\sqrt{2})$ for each $t \in (2nT_0, 2(n+1)T_0]$ and $k  \in [2]$.
 Since $A_t = \arg\max_{k \in [2]} \cI_t(k)$, for each $t \in (2nT_0, 2(n+1)T_0]$ and $k  = 1,2$, if  $\mvX_t \in \cU^{(k)}_{\ell_0}$,  then we must have $A_t = k$, which implies
 \begin{align*}
     \lambda_{\min}\left( \bV_{2(n+1)T_0}^{(k)}\right)  \geq  \lambda_{\min}\left( \bV_{2nT_0}^{(k)}\right)+
      \lambda_{\min}\left(\sum_{t=2nT_0+1}^{2(n+1)T_0} \mvX_t \mvX_t'  I(\mvX_t \in \cU_{\ell_0}^{(k)})\right).  
 \end{align*}
Then the induction is complete due to the event $\cA_3$. The proof is complete.
\end{proof}

Next we show that the events $\cA_2, \cA_3, \cA_4$ in \eqref{events:def} happens with a high probability.

\begin{lemma}
\label{lemma:cst_adaptive}
Assume the condition \ref{cond:continuity} holds. There exists an absolute constant $C > 0$ such that the event $\cA_2 = \left\{ \max_{k=1,2} \lambda_{\min}\left( \bV_t^{(k)}\right)  \geq 6^{-1}\ell_1^2 t, \text{ for all } t \in [C(d +\log(T)),T] \right\}$ happens with probability at least $1-1/T$.
\end{lemma}

\begin{proof}
In this proof we denote by $C,\tilde{C}$ \textit{absolute} constants that may differ from line to line.  Observe that by definition, for each $t \in [T]$,
\begin{align*}
    &\sum_{k=1}^{2} \lambda_{\min}\left( \bV_t^{(k)}\right) = \inf_{\mvu,\mvv \in \cS^{d-1}} \left(\mvu' \bV_t^{(1)} \mvu +
    \mvv'\bV_t^{(2)}\mvv \right) \\
    \geq &\inf_{\mvu,\mvv \in \cS^{d-1}} \sum_{s=1}^{t}\left(
    u' X_s X_s' I(A_s = 1) u +  v' X_s X_s' I(A_s = 2) v
    \right)\\
    \geq &\inf_{\mvu,\mvv \in \cS^{d-1}} \sum_{s=1}^{t} {\ell}_1^2 I\left(|\mvu'X_s| \geq {\ell}_1,\;
    |\mvv'X_s| \geq {\ell}_1 \right).
\end{align*}
For  $\mvu,\mvv \in \cS^{d-1}$, define $\phi_{\mvu, \mvv}(\mvx) = I\left(|\mvu' \mvx| \geq {\ell}_1,\;
    |\mvv'\mvx| \geq {\ell}_1 \right)$,
 $N_t(\mvu, \mvv)  = \sum_{s=1}^{t} \phi_{\mvu, \mvv}(\mvX_s)$, 
 and $\Delta_t = \sup_{\mvu,\mvv \in \cS^{d-1}}|N_t(\mvu, \mvv) - \Exp[N_t(\mvu, \mvv)]|$.  Then 
 \[
2 \max_{k=1,2} \lambda_{\min}\left( \bV_t^{(k)}\right) \;\geq \; {\ell}_1^2\inf_{\mvu,\mvv \in \cS^{d-1}} N_t(\mvu, \mvv) \; \geq \; {\ell}_1^2( \inf_{\mvu,\mvv \in \cS^{d-1}} \Exp[N_t(\mvu, \mvv)] - \Delta_t).
 \]
Due to \ref{cond:continuity}, for each $\mvu,\mvv \in \cS^{d-1}$,
$$
\Exp[\phi_{\mvu, \mvv}(\mvX)] \geq  1 - \bP(|\mvu'\mvX| \leq {\ell}_1) - 
\bP(|\mvv'\mvX| \leq {\ell}_1) \geq 1/2,
$$
which implies that $\inf_{\mvu,\mvv \in \cS^{d-1}} \Exp[N_t(\mvu, \mvv)] \geq t/2$ for each $t \in [T]$. Further, by Lemma \ref{lemma:Talagrand} with $\tau = 2\log(T)$, and the union bound,  with probability at least $1-1/T$, for all $t \in [T]$, 
$\Delta_t \leq \tilde{C}(\sqrt{d t} + \sqrt{t\log(T)} + \log(T))$. 
Note that there exists an absolute constant ${C}$ such that if $t \geq {C}(d+\log(T))$, then $6^{-1} t \geq \tilde{C}(\sqrt{dt} + \sqrt{t\log(T)} + \log(T))$. As a result, with probability at least $1-1/T$,   $\inf_{\mvu,\mvv \in \cS^{d-1}} N_t(\mvu, \mvv) \geq 3^{-1} t$ for any $t \in  [{C}(d +\log(T)),T]$, which completes the proof.
\end{proof}

\begin{lemma}\label{lemma:posdef}
Assume the conditions \ref{assumption_parameter_noise} and \ref{cond:posdef} hold. There exists a positive constant $C$, depending only on $m_{X},\ell_0$, such that
with probability at least $1-d/T$, the following event $\cA_3$ happens: 
$\lambda_{\min}\left( \sum_{s=t_1+1}^{t_2}  \mvX_s \mvX_s' I(\mvX_s \in \cU_{\ell_0}^{(k)})\right) \geq {\ell_0^2 (t_2-t_1)}/2$, for any $t_1,t_2 \in [T]$ with $t_2 - t_1 \geq C d \log(T)$, and $k = 1,2$.
\end{lemma}

\begin{proof}
Denote $\Delta = t_2 - t_1$. By \citep[Theorem 1.1]{tropp2012user}, with $R = d m_X^2$, $\mu_{\min} = \Delta \ell_0^2$, and $\delta = 1/2$ therein, 
\begin{align*}
    \bP\left(\lambda_{\min}\left( \sum_{s=t_1+1}^{t_2}  \mvX_s \mvX_s' I(\mvX_s \in \cU_{\ell_0}^{(k)})\right) \leq \frac{\ell_0^2 \Delta}{2} \right) \leq d \exp\left( - \frac{\Delta \ell_0^2 \log(\sqrt{e/2})}{d m_X^2}\right).
\end{align*}
Thus if $\Delta \geq C d\log(T)$, with $C = 3m_X^2/(\ell_0^2\log(\sqrt{e/2}))$, the above probability is upper bounded by  $d/T^3$, which completes the proof by the union bound over $t_1,t_2 \in [T]$ and $k = 1,2$.
\end{proof}

Recall $\tilde{\beta}_t$ in \eqref{eq:uuper_beta_t},  $\Upsilon_{d,T}$ in \eqref{chi_order}, and  $\ell_{*} = \min\{\ell_1,\ell_0\}/3$.

\begin{lemma}
\label{lemma:beta_t}
Assume the condition \ref{assumption_parameter_noise} hold.  Then
there exists a constant $C$, depending only on $\mvTheta_0$ and $\lambda$, such that
$\sqrt{\tilde{\beta}_{t}} \left( \ell_{*}^2t/2 \right)^{-1/2} (\sqrt{d}m_{X}) \leq  {\ell_0}/{8}$ for  all $t \geq C \Upsilon_{d,T}$.
\end{lemma}

\begin{proof}
By definition, there exists $\tilde{C}$, depending only on $\mvTheta_0$ and $\lambda$,  such that
\begin{align*}
    \tilde \beta_{t} \leq \tilde{C}(\log(T) + d \log(t))\, \text{ for } t \geq 2,\qquad
\tilde{C}_1 := {128\tilde{C}m_X^2}/(\ell_{0}^2\ell_{*}^2) \geq 9.
\end{align*}
Let $a := \tilde{C}_1 d\log(T),\ b := \tilde{C}_1 d^2$. By Lemma \ref{lemma:aux_exp_inequal}, if $t \geq a + 2b\log(a+b)$, then 
\begin{align*}
a + b \log(t) \leq t  \;\;\Longleftrightarrow\;\; \tilde{C}(\log(T) + d \log(t))\left( \ell_{*}^2t/2 \right)^{-1} d m_{X}^2 \leq \left( {\ell_0}/{8}\right)^2,
\end{align*}
which completes the proof.
\end{proof}

\subsection{Proof of Lemma  \ref{lemma:beta_V}} \label{proof:lemma:beta_V}
Next, we prove Lemma \ref{lemma:beta_V}, which shows that if the smallest eigenvalues of the ``design" matrices are large,  the estimation of arm parameters is accurate.

\begin{proof}
Fix $t \in [T]$, $k \in [2]$. In this proof, $C$ is a constant, depending only on $\mvTheta_0$, $\lambda$, that may vary from line to line.
By definition,  $\hat \mvtheta_t^{(k)} - \mvtheta_k =
(\bV_{t}^{(k)})^{-1}(\sum_{s=1}^{t} \mvX_s I(A_s = k) \epsilon_s - \lambda \mvtheta_k)$. Thus due to the condition \ref{assumption_parameter_noise}, and on the event $\mathcal{E}_{t}^{(k)}$ (defined in \eqref{def:info_event}), we have $\|\hat \mvtheta_t^{(k)} - \mvtheta_k \| \leq C t^{-1} (\|\sum_{s=1}^{t} \mvDelta_s\|+ 1)$, 
where $\mvDelta_s = \mvX_s I(A_s = k) \epsilon_s$.
Note that $\{\mvDelta_s: s\in [t]\}$ is a sequence of vector martingale differences with respect to $\{\cF_{s}: s \in \{0\} \cup [t]\}$. 

Due to the condition \ref{assumption_parameter_noise}, for any $\tau \geq 0$, almost surely,
$\bP(\|\mvDelta_s\| \geq \tau \ \vert \cF_{s-1}) \leq 
\bP(|\epsilon_s^{(k)}| \geq \tau/(\sqrt{d}m_X) \ \vert \cF_{s-1}) \leq 2\exp(-\tau^2/(2dm_X^2 \sigma^2))$.
Then by \citep[Corollary 7]{jin2019short},
for any $\tau > 0$, 
$$
\bP(\|\sum_{s=1}^{t} \mvDelta_s\| \leq C (\sqrt{dt\log(d)} + \tau \sqrt{dt}) \geq 1 - 2e^{-\tau^2},
$$
which completes the proof of the first claim, by considering $\tau \leq \sqrt{\log(2)}$ and $\tau > \sqrt{\log(2)}$.

Further, for $1 \leq s_1 < s_2 \leq t$, $\Exp[\mvDelta_{s_1}' \mvDelta_{s_2}]
= \Exp[\mvDelta_{s_1}' \Exp[\mvDelta_{s_2} \vert \cF_{s_2-1}]] = 0$. Thus due to \ref{assumption_parameter_noise},
\begin{align*}
    \Exp\left[\|\sum_{s=1}^{t} \mvDelta_s\|^2\right] = 
    \Exp\left[\sum_{s=1}^{t} \mvX_{s}'\mvX_{s} I(A_t = s) (\epsilon_s)^2  \right] \leq \sigma^2 \Exp[\sum_{s=1}^{t} \mvX_{s}'\mvX_{s} ] \leq C d t,
\end{align*}
which completes the proof for the second claim.
\end{proof}

\section{Lower bound for all admissible rules: proof of Theorem \ref{theorem:lower_bound}}\label{sec:proof_lower_bound}

Here, we provide the proof for the lower bound part in Theorem \ref{theorem:lower_bound}, and 
the upper bound proof is in Appendix \ref{proof:theorem:lower_bound_upper}.

\begin{proof}[Proof for the lower bound part of Theorem \ref{theorem:lower_bound}]
In this proof, $C$ is an absolute, positive constant, that may vary from line to line. First, we consider the case that the context vector $\mvX$ has the $\textup{Unif}(\sqrt{d}\cS^{d-1})$ distribution.

For a given $d \geq 3$, a problem  instance  in \ref{prob_instances_all_lower} is identified with $\mvtheta_2 \in \bR^{d}$. Let $\mvTheta_2$ be a random vector  with a  Lebesgue density  ${\rho}_d(\cdot)$ on $\bR^{d}$, supported on $ \cB_d(1/2,1) = \{\mvx \in \bR^{d}: 
2^{-1} \leq \|\mvx\| \leq 1\}$:
\begin{align}
\label{LB:theta_density}
{\rho}_d(\mvtheta) = \frac{\tilde \rho(\|\mvtheta\|)}{A_d\|\mvtheta\|^{d-1}} \text{ for } \mvtheta \in \bR^{d},\;\; \text{ with } \; \tilde{\rho}(\tau) = 4\sin^2(2\pi\tau)I(2^{-1}\leq \tau\leq 1),
\end{align}
where $A_d$ is the Lebesgue area of $\cS^{d-1}$. Then for any admissible rule $\{\pi_t, t \in[T]\}$,
\begin{align*}
\sup_{\mvtheta_2 \in \cB_d(1/2,1)}R_T(\{\pi_t, t \in[T]\};\ d,\mvtheta_2) 
     \geq \Exp[R_T(\{\pi_t, t \in[T]\};\ d,\mvTheta_2)].
\end{align*}
Below, we fix some admissible rule $\{\pi_t, t \in[T]\}$, and study its ``Bayes" risk $\Exp[R_T(d,\mvTheta_2)]$, where the randomness comes from $\mvTheta_2$, in addition to the contexts $\{\mvX_{t}: t\in [T]\}$, observation noises $\{\epsilon_{t}^{(k)} : t\in [T], k \in [2]\}$, and possible random mechanism enabled by i.i.d.~$\textup{Unif}(0,1)$ random variables $\{\xi_{t}: t \in [T]\}$. Recall that $\mvtheta_1 = \boldsymbol{0}_{d}$ is deterministic, and let $\mvTheta_1= \boldsymbol{0}_{d}$. Recall from Section \ref{sec:prob_formulation} that  $\cF_0 = \sigma(0)$, and for each $t \in [T]$,  $\cF_t =  \sigma(\mvX_s,A_s,Y_s: s \in [t])$ denotes the available information up to time $t$, and  $\cF_{t+} := \sigma(\cF_{t}, \mvX_{t+1}, \xi_{t+1})$  the information set during the decision making at time $t+1$; in particular, $A_t \in \cF_{(t-1)+}$ for each $t \in [T]$.\\


By definition, $\Exp[R_T(d,\mvTheta_2)] = \sum_{t=1}^{T} \Exp[\hat{r}_t]$,
where $\hat{r}_t := \max_{k \in [K]} (\mvTheta_k'\mvX_t) - \mvTheta_{A_t}'\mvX_t$.
Since $\mvTheta_1 = \boldsymbol{0}_{d}$, $\hat{r}_t  = (\mvTheta_2'\mvX_t) I(\mvTheta_2'\mvX_t \geq 0) I(A_t = 1) - (\mvTheta_2'\mvX_t) I(\mvTheta_2'\mvX_t < 0)I(A_t = 2)$. Then the Bayes rule is: $\hat{A}_t = 1$ if and only if the conditional cost, given $\cF_{(t-1)+}$, for arm 1 is no larger than  for arm 2, i.e.,  
\begin{align*}
     & \Exp[(\mvTheta_2'\mvX_t) I(\mvTheta_2'\mvX_t \geq 0) \vert \cF_{(t-1)+}] \leq 
\Exp[ -(\mvTheta_2'\mvX_t) I(\mvTheta_2'\mvX_t < 0)\vert \cF_{(t-1)+}], \\
 \;\; &\Longleftrightarrow\;\; \Exp[ \mvTheta_2'\mvX_t \vert \cF_{(t-1)+}] \leq 0
\;\;\Longleftrightarrow\;\; (\hat{\mvTheta}_{t-1}^{(2)})' \mvX_t \leq 0,
\end{align*}
where $\hat{\mvTheta}_{t-1}^{(2)} := \Exp[ \mvTheta_2 \vert \cF_{(t-1)+}]$ and the last equivalence is because $\mvX_{t} \in \cF_{(t-1)+}$.
Thus, for $t \in [T]$,
\begin{align*}
    \Exp[\hat{r}_t] \geq \Exp\left[ |\mvTheta_2'\mvX_t| I(\text{sgn}(\mvTheta_2'\mvX_t) \neq \text{sgn}((\hat{\mvTheta}_{t-1}^{(2)})'\mvX_t))\right].  
\end{align*}
Since $\mvTheta_2$ are independent from $\mvX_{t}$ and $\xi_{t}$, $\hat{\mvTheta}_{t-1}^{(2)} = \Exp[ \mvTheta_2 \vert \cF_{t-1}]$ almost surely.  Note that $\mvTheta_2 \in \cB_d(1/2,1)$ and so is $\hat{\mvTheta}_{t-1}^{(2)}$.
Since $\mvX_{t}$ is independent from $\cF_{t-1}$ and $\mvTheta_2$, due to Lemma \ref{app:unit_sphere_lu} with $u=\mvTheta_2/\|\mvTheta_2\|$ and $v = \hat{\mvTheta}_{t-1}^{(2)}/\|\hat{\mvTheta}_{t-1}^{(2)}\|$, 
\begin{align*}
    \Exp[\hat{r}_t\ \vert \cF_{t-1}] \geq C^{-1}   \|\mvTheta_2\| \left\|\frac{\mvTheta_2}{\|\mvTheta_2\|} - \frac{\hat{\mvTheta}_{t-1}^{(2)}}{\|\hat{\mvTheta}_{t-1}^{(2)}\|} \right\|^2.
\end{align*}
For $t \geq 0$, denote by ${\cH}_t = \sigma(\mvX_s, Y_s^{(1)}, Y_{s}^{(2)}, \xi_s: s \in [t])$ all potential random observations up to time $t$, and by definition, $\cF_{t} \subset \cH_t$. Thus for $t \in [T]$, since $\mvTheta_2 \in \cB_d(1/2,1)$,
\begin{align*}
    \Exp[\hat{r}_t] \geq C^{-1}  \inf\{\Exp\left[ \|\hat \mvpsi_{t-1} - {\mvTheta_2}/{\|\mvTheta_2\|}\|^2 \right]: \hat \psi_{t-1} \in \cH_{t-1} \text{ is an } \bR^d \text{ random vector}\}.
\end{align*}
Since $Y_s^{(1)} = \epsilon_s^{(1)}$ for $s \in [t]$, 
$\{\mvX_s,Y_{s}^{(2)}: s \in [t]\}$ are independent from 
$\{Y_s^{(1)}, \xi_s: s \in [t]\}$. Since $\Exp[\|\mvX_1\|^2] = d$, by Lemma \ref{app:lower_bound_van_Tree},  $\Exp[\hat{r}_t] \geq C^{-1}   (d-1)^2/((t-1)d + Cd^2)$ for some $C > 0$, and thus 
\begin{align*}
\Exp[R_T(d,\mvTheta_2)] \geq C^{-1}     \sum_{t \in [T]} (d-1)^2/((t-1)d + Cd^2)  \geq C^{-1} d \log(T/d),
\end{align*}
which completes the proof for the case that $\mvX$   has the $\textup{Unif}(\sqrt{d}\cS^{d-1})$ distribution.

Finally, note that in the above arguments, the distributional properties we require for the context $\mvX$ are Lemma \ref{app:unit_sphere_lu} and $\Exp[\|\mvX_1\|^2] \leq d$, which continue to hold if $\mvX$ has an isotropic log-concave density, in view of Lemma \ref{lemma:logconcave_example} and since $\Exp[\|\mvX_1\|^2]= \text{trace}(\Cov(\mvX_1)) = d$. The proof for the lower bound is complete.
\end{proof}

\section{Conclusion}\label{sec:conclusion}
In this work, we consider the stochastic linear bandit problem in a low-dimensional regime, where the covariate dimension $d$ is much smaller than the time horizon $T$. We show that the LinUCB algorithm is suboptimal in this setting due to over-exploration, and propose a truncated variant, Tr-LinUCB, which switches to pure exploitation after a specified time $S$. Through theoretical analysis and simulations, we demonstrate that Tr-LinUCB is robust to the choice of $S$. Furthermore, we characterize the minimax rate for concrete families of problem instances and show that Tr-LinUCB achieves minimax optimality. Although the setup is classical, the optimal dependence on $d$ established here is, to our knowledge, novel.

As for future directions, it is of interest to consider the stochastic high-dimensional sparse linear bandit problem, where the minimax rate remains unknown. In addition, we plan to extend the framework to generalized linear models and to settings with unknown observation noise.





\appendix

\section{Lower bound for LinUCB - proof of Theorem \ref{them:lb_linucb}}\label{app:proof_lower_LinUCB}
In this subsection, we consider problem instances in \ref{prob_instances_LinUCB}. We preface the proof with a few lemmas.

\begin{lemma}\label{app:lunucb_info}
Consider problem instances in \ref{prob_instances_LinUCB} with $p=0.6,\sigma^2=1$ and the LinUCB algorithm, i.e., the truncation time $S = T$, with $\lambda=   m_{\theta} = 1$. There exists an absolute  positive constant $C$ such that with probability at least $1-Cd/T$,   $\tilde{\Gamma}_t$ occurs for all $t \geq Cd\log(T)$, where $\tilde{\Gamma}_t$ denotes the event that $0.35 t \leq \lambda_{\min}( \bV_{t}^{(2)} ) \leq \lambda_{\max}( \bV_{t}^{(2)} )\leq  0.45 t
 \leq 0.55 t\leq \lambda_{\min}( \bV_{t}^{(1)} ) \leq \lambda_{\max}( \bV_{t}^{(1)} )\leq  0.65 t$.
\end{lemma}

\begin{proof}
In this proof, $C$ is an absolute, positive constant, that may vary from line to line.  Recall that  $\mvX$ is distributed as $(\iota |\mvPsi_1|,\mvPsi_2,\ldots, \mvPsi_{d})$, where $\mvPsi = (\mvPsi_1,\mvPsi_2,\ldots, \mvPsi_{d})$ has the uniform distribution on the  sphere with radius $\sqrt{d}$, i.e., $\textup{Unif}(\sqrt{d}\cS^{d-1})$, $\iota$ takes value $+1$ and $-1$ with probability $p=0.6$ and $1-p$ respectively, and $\mvPsi$ and $\iota$ are independent.

By definition, $\cU_{h}^{(1)} = \{\mvx \in \bR^{d}: \mvx_1 > h/2\}$, and $\cU_{h}^{(2)} = \{\mvx \in \bR^{d}: \mvx_1 < -h/2\}$. The condition \ref{assumption_parameter_noise} clearly holds with $m_X = m_{\theta} = \sigma^2=1$. By Lemma \ref{aux:unit_sphere}, for any $\mvu \in \cS^{d-1}$ and $\tau > 0$,
\begin{align*}
    \bP(|\mvu'\mvX| \leq \tau) \leq 2\sup_{\mvv \in \cS^{d-1}}\bP(|\mvv'\mvPsi| \leq \tau) \leq 4\tau.
\end{align*}
Thus the condition \ref{cond:margin} holds  with $L_0 = 4$ and the condition \ref{cond:continuity} holds with $\ell_1 = 1/16$.
Further, for any $\mvu \in \cS^{d-1}$ and $\ell_0>0$, $\Exp[(\mvu'\mvX)^2 I(\mvX \in \cU_{\ell_0}^{(1)})]) = 0.6(1-\Exp[(\mvu'\mvPsi)^2  I(|\mvPsi_1| \leq \ell_0/2)])$ and 
$\Exp[(\mvu'\mvX)^2 I(\mvX \in \cU_{\ell_0}^{(2)})]) = 0.4(1-\Exp[(\mvu'\mvPsi)^2  I(|\mvPsi_1| \leq \ell_0/2)])$. Thus by Lemma \ref{app:sphere_smallest_eigen}, 
\begin{align*}
   \lambda_{\min}(\Exp[\mvX\mvX' I(\mvX \in \cU_{0.01}^{(1)})]) \geq 0.58, \quad
      \lambda_{\min}(\Exp[\mvX\mvX' I(\mvX \in \cU_{0.01}^{(2)})]) \geq 0.38.
\end{align*}
In particular, the condition \ref{cond:posdef} holds with $\ell_0 = 0.01$.  Recall $\Upsilon_{d,T}$ in \eqref{chi_order}, and due to \eqref{def:regime}, $\Upsilon_{d,T} \leq  3d\log(T)$. Thus by  Lemma \ref{lemma:small_eignenvalues_anyt},  with probability at least $1-4d/T$, for each $t \geq Cd\log(T)$,
$\min_{k = 1,2} \lambda_{\min}\left( \bV_{t}^{(k)}\right) \geq C^{-1} t$. In view of \eqref{eq:uuper_beta_t} and by Lemma \ref{lemma:LinUCB},   with probability at least $1-Cd/T$, for each $t \geq Cd\log(T)$,
\begin{align*}
    |\text{UCB}_t(k) - \mvtheta_k'\mvX_t| \leq 2\sqrt{\beta_{t-1}^{(k)}}\|\mvX_t\|_{(\bV^{(k)}_{t-1})^{-1}} \leq \ell_0/2,
\end{align*}
which implies that for each $k \in [2]$, if $\mvX_{t} \in \cU_{\ell_0}^{(k)}$, the optimal arm would be selected, i.e., $A_t = k$. By \citep[Theorem 1.1]{tropp2012user} (see Lemma \ref{lemma:posdef}), 
with probability at least $1-Cd/T$, the following occurs: for any $t_1, t_2 \in [T]$, if $t_2 - t_1 \geq Cd\log(T)$, 
$\lambda_{\min}(\sum_{s=t_1+1}^{t_2}  \mvX_s \mvX_s' I(\mvX_s \in \cU_{\ell_0}^{(k)})) \geq p_k (t_2-t_1)$ with $p_1 = 0.57$ and $p_2 = 0.37$, which concludes the proof of the part regarding $\lambda_{\min}$.

Finally, again by \citep[Theorem 1.1]{tropp2012user}, since $\lambda_{\max}(\Exp[\mvX\mvX']) = 1$, with probability at least $1-Cd/T$, for each $t \geq Cd\log(T)$,
$\lambda_{\max}(\sum_{s=1}^{t}  \mvX_s \mvX_s' ) \leq 1.01 t$.
Note that
$$
\lambda_{\max}(2 \bI_d + \sum_{s=1}^{t}  \mvX_s \mvX_s' ) \;\geq\; 
\max\{ \lambda_{\max}(\bV_{t}^{(2)}) 
+ \lambda_{\min}(\bV_{t}^{(1)} ),\;\;
\lambda_{\max}(\bV_{t}^{(1)}) 
+ \lambda_{\min}(\bV_{t}^{(2)} )
\},
$$
which leads to the part regarding $\lambda_{\max}$, and completes the proof.
\end{proof}

\begin{lemma}\label{aux:lb_linucb_cal}
Consider problem instances in \ref{prob_instances_LinUCB} with $p = 0.6, \sigma^2=1$ and the LinUCB algorithm, i.e., the truncation time $S = T$, with $\lambda=  m_{\theta}  = 1$. Assume \eqref{def:regime} holds.  There exists an absolute  constant $\tilde{C} \geq 1$  such that if $T \geq \tilde{C}$,  for each $1 \leq t < T$, on the event $\tilde{\Gamma}_{t}$ (defined in Lemma \ref{app:lunucb_info}), the following occurs:
\begin{align*}
    \tilde{\Delta}_t := \sqrt{\beta_{t}^{(2)}}\|\mvX_{t+1}\|_{(\bV_{t}^{(2)})^{-1}} -  \sqrt{\beta_{t}^{(1)}}\|\mvX_{t+1}\|_{(\bV_{t}^{(1)})^{-1}} \geq \tilde{C}^{-1} \sqrt{(d\log(T)+d^2\log(t))/t}.
\end{align*}
\end{lemma}
\begin{proof}
By definition, on the event $\tilde{\Gamma}_{t}$,   we have
\begin{align*}
\tilde{\Delta}_t
  \;\;\geq\;\; &(1+\sqrt{2\log(T) + d\log(0.35t)} ) \ (0.45t)^{-1/2} \sqrt{d} \\
  - & (1+\sqrt{2\log(T) + d\log(0.65t)})\  (0.55t)^{-1/2}\sqrt{d}.
\end{align*}
Due to \eqref{def:regime}, if $T \geq \tilde{C}$, $\log(T) \geq 10d\log(1/0.35)$, and as a result
\begin{align*}
       \tilde{\Delta}_t \geq &d^{1/2} t^{-1/2} \left( \sqrt{ (1.9\log(T)+d\log(t))/0.45} - \sqrt{(2\log(T)+d\log(t))/0.55} \right) \\
       \geq &\tilde{C}^{-1} \sqrt{(d\log(T)+d^2\log(t))/t},
\end{align*}
for some absolute constant $\tilde{C} > 0$, which completes the proof.
\end{proof}

\begin{proof}[Proof of Theorem \ref{them:lb_linucb}]
In this proof, $C,C'$ are absolute positive constants, that may vary from line to line. Recall the constant $\tilde{C} \geq 1$ in Lemma \ref{aux:lb_linucb_cal}, and define 
$$
D_{t+1} = \{8^{-1}\tilde{C}^{-1} (\upsilon_{d,t}/t)^{1/2} \leq  \mvX_{t+1,1} \leq  4^{-1}\tilde{C}^{-1} (\upsilon_{d,t}/t)^{1/2}\},  \text{ with }
\upsilon_{d,t} = d\log(T) + d^2 \log(t),
$$
where $\mvX_{t+1,1}$ is the first component of $\mvX_{t+1}$. If $t \geq Cd\log(T)$, due to equation \eqref{def:regime},  
$$4^{-1}\tilde{C}^{-1} (\upsilon_{d,t}/t)^{1/2} \leq 1,$$ and thus by Lemma \ref{aux:unit_sphere}, $\bP(D_{t+1})\geq C^{-1} (\upsilon_{d,t}/t)^{1/2}$.

Further, due to \eqref{def:regime} and  by Lemma \ref{app:lunucb_info}, if $T \geq C$, for each $t \geq Cd\log(T)$, $\bP(\tilde{\Gamma}_t) \geq 0.9$, where $\tilde{\Gamma}_{t}$ is defined in Lemma \ref{app:lunucb_info}. By Lemma \ref{lemma:beta_V} and Markov inequality, for each $t \in [T]$ and $k \in [2]$,
$\bP(\|\hat \mvtheta_t^{(k)} - \mvtheta_k\| \geq C' (d/t)^{1/2},  \tilde{\Gamma}_t) \leq 0.1$. Thus for each $t \geq Cd \log(T)$, since $\mvX_{t+1}$ are independent from $\cF_{t}$,
\begin{align*}
     &\bP(D_{t+1},\ \|\hat \mvtheta_t^{(1)} - \mvtheta_1\| \leq C'(d/t)^{1/2},\ \|\hat \mvtheta_t^{(2)} - \mvtheta_2\| \leq C'(d/t)^{1/2},\ \tilde{\Gamma}_t) \\ 
    \geq \;\;& \bP(D_{t+1})\;(\bP(\tilde{\Gamma}_t) - \sum_{k=1}^{2} \bP(|\hat \mvtheta_t^{(k)} - \mvtheta_k\| \geq C' (d/t)^{1/2},  \tilde{\Gamma}_t)) \geq C^{-1} (\upsilon_{d,t}/t)^{1/2}.
\end{align*}

On the event $D_{t+1}$, $\mvtheta_2'\mvX_{t+1} - \mvtheta_1'\mvX_{t+1} \geq -2^{-1}\tilde{C}^{-1} (\upsilon_{d,t}/t)^{1/2}$. On the event $\cap_{k=1}^{2} \{|\hat \mvtheta_t^{(k)} - \mvtheta_k\| \leq C'(d/t)^{1/2}\}$, since $\|\mvX_{t+1}\| = \sqrt{d}$, 
we have $|(\hat \mvtheta_t^{(k)} - \mvtheta_k)'\mvX_{t+1}| \leq C'd t^{-1/2}$ for $k \in [2]$. Finally, due to Lemma \ref{aux:lb_linucb_cal}, on the event $\tilde{\Gamma}_t$, $\sqrt{\beta_{t}^{(2)}}\|\mvX_{t+1}\|_{(\bV_{t}^{(2)})^{-1}} -  \sqrt{\beta_{t}^{(1)}}\|\mvX_{t+1}\|_{(\bV_{t}^{(1)})^{-1}} \geq \tilde{C}^{-1}(\upsilon_{d,t}/t)^{1/2}$. Combining them,  on the intersection of these events, we have
\begin{align*}
    &\textup{UCB}_{t+1}(2) - \textup{UCB}_{t+1}(1) \;\;=\;\; \mvtheta_2'\mvX_{t+1} - \mvtheta_1'\mvX_{t+1} + 
    (\hat \mvtheta_t^{(2)}-\mvtheta_2)'\mvX_{t+1} - (\hat \mvtheta_t^{(1)}-\mvtheta_1)'\mvX_{t+1}
    \\
     &\qquad \qquad\qquad\qquad\qquad\qquad
+     \sqrt{\beta_{t}^{(2)}}\|\mvX_{t+1}\|_{(\bV_{t}^{(2)})^{-1}} -  \sqrt{\beta_{t}^{(1)}}\|\mvX_{t+1}\|_{(\bV_{t}^{(1)})^{-1}} \\
 &\geq - 2^{-1}\tilde{C}^{-1} (\upsilon_{d,t}/t)^{1/2} - 2C'd t^{-1/2} + \tilde{C}^{-1} (\upsilon_{d,t}/t)^{1/2}
\geq 2^{-1}\tilde{C}^{-1} (\upsilon_{d,t}/t)^{1/2} - 2C'd t^{-1/2}.
\end{align*}
In particular, if $\log(t) > 4 (\tilde{C} C')^2$, then $\textup{UCB}_{t+1}(2) - \textup{UCB}_{t+1}(1) > 0$, and the second arm would be selected, i.e., $A_t = 2$, incurring a regret that is at least $4^{-1}\tilde{C} (\upsilon_{d,t}/t)^{1/2}$.

To sum up, if $T \geq C$ and $t \geq Cd\log(T)$,
$ \Exp[\hat{r}_{t+1}] \geq C^{-1}(\upsilon_{d,t}/t)^{1/2}\; \bP(D_{t+1},\ \cap_{k=1}^{2}\|\hat \mvtheta_t^{(k)} - \mvtheta_k\|  \leq C't^{-1/2},\ \tilde{\Gamma}_t) \geq C^{-1}\upsilon_{d,t}/t$.Thus
$R_T \geq \sum_{t=Cd\log(T)}^{T} d^2 \log(t)/t \geq C^{-1}d^2\log^2(T)$, which completes the proof, since the upper bound follows from Corollary \ref{corollary:regret_main_T}.
\end{proof}

\section{Discrete components - proofs} \label{app:proof_discrete}

\begin{proof}[Proof of Theorem \ref{theorem:discrete_regret_main_T}]
Note that the difference between Theorem \ref{theorem:discrete_regret_main_T} and Theorem \ref{theorem:regret_main_T} is that the condition \ref{cond:continuity} is replaced by \ref{cond:discrete}.
Recall that in the proof for Theorem \ref{theorem:regret_main_T}, the arguments rely on Lemma \ref{lemma:small_eignenvalues_anyt} and 
\ref{lemma:beta_V}, but not on the condition \ref{cond:continuity}. Further,  Lemma \ref{lemma:beta_V} does not use the condition \ref{cond:continuity}. Thus the same arguments for Theorem \ref{theorem:regret_main_T} apply here as long as we replace Lemma \ref{lemma:small_eignenvalues_anyt} by  Lemma \ref{lemma:discrete_anyt} below.
\end{proof}

We introduce a few notations. For $j \in [L_2]$, $k \in [2]$, $t \in [T]$,  we cluster contexts based on the value of the first $d_1$ coordinates, and define
\begin{align*}
&  \tilde{\bV}_t^{(k)}(\mvz_j) = \lambda \bI_{d} + \sum_{s=1}^{t}  {\mvX}_s {\mvX}_s' I(\mvX_s^{(\text{d})} = \mvz_j,\ A_s = k),\\
&  \bar{\bV}_t^{(k)}(\mvz_j) = \lambda \bI_{1+d_2} + \sum_{s=1}^{t}  \bar{\mvX}_s^{(\text{c})} (\bar{\mvX}_s^{(\text{c})})'I(\mvX_s^{(\text{d})} = \mvz_j,\ A_s = k).
\end{align*}
Note that  for some constant $C_3 > 0$, depending only on $\lambda,m_{\theta},m_{X}, \sigma^2, d$,  
\begin{align}\label{beta_t_Tbound}
\beta_t^{(k)} \;\;\leq \;\; C_3 \log(T), \text{ for any } k \in [2],\ t \in [T].
\end{align}
Recall that before we use $\tilde{\beta}_t$ in \eqref{eq:uuper_beta_t} to bound $\sup_{k \in [2]} \beta_t^{(k)}$ for each $t \in [T]$, in order to make the dependence on $d$ explicit. In this section, however, we assume $d$ fixed, and thus can use the above simpler bound.


\begin{lemma}\label{lemma:discrete_anyt}
Assume that  conditions \ref{assumption_parameter_noise}, \ref{cond:posdef}, and \ref{cond:discrete} hold. There exists a constant $C >0$, depending only on $\mvTheta_2,d,\lambda$, such that if $S \geq C\log(T)$, then with probability at least $1-C/T$, $\min_{k = 1,2} \lambda_{\min}\left( \bV_{t}^{(k)}\right) \geq C^{-1} t$ for each  $t  \geq C\log(T)$, 
\end{lemma}
\begin{proof}
By Lemma \ref{lemma:LinUCB}, \ref{lemma:discrete_diverse} and \ref{lemma:posdef}, there exists a constant $C >0$, depending only on $\mvTheta_2,d,\lambda$, such that if $S \geq C\log(T)$, then with probability at least $1-C/T$, the following event $\bar{\cE}$ occur: for all $k \in [2], j \in [L_2], t \geq T_1$, and $t_1,t_2 \in [T]$ with $t_2 -t_2 \geq T_1$,
\begin{align*}
&\|\hat \mvtheta_{t-1}^{(k)} - \mvtheta_{k}\|_{\bV_{t-1}^{(k)}} \leq \sqrt{\beta_t^{(k)}}, \qquad
 \lambda_{\min}\left( \bar{\bV}_{T_1}^{(k)}(\mvz_j)\right)  \geq \tilde{C}_3  \log(T),    \\
 &\lambda_{\min}\left( \sum_{s=t_1+1}^{t_2}  {\mvX}_s \mvX_s' I(\mvX \in \cU_{\ell_0}^{(k)})  \right) \geq {\ell_0^2 (t_2-t_1)}/2, 
\end{align*}
where $T_1 = \lceil C \log(T) \rceil$,  $\tilde{C}_3 = 16 C_3 (1+ \sqrt{d} m_X)^4 \tilde{\ell}_*^{-2}$, $\tilde{\ell}_* = \min\{\ell_2,\ell_0\}$, and $C_3$, $\ell_0$, $\ell_2$ appear in \eqref{beta_t_Tbound}, \ref{cond:posdef}, and \ref{cond:discrete}  respectively.

First, for $t > T_1$ and $k \in [2]$, due to \eqref{beta_t_Tbound}, on the event $\bar{\cE}$, for \textit{both} $t \leq S$ and $t > S$,
\begin{align*}
    |\text{UCB}_{t}(k) - \mvtheta_{k}'\mvX_{t}| \leq 2 \sqrt{\beta_{t-1}^{(k)}}  \|\mvX_t\|_{(\bV_{t-1}^{(k)})^{-1}} \leq 2 (C_3\log(T))^{1/2} \|\mvX_t\|_{(\bV_{t-1}^{(k)})^{-1}},
\end{align*}
and by Lemma \ref{aux:replaced_by_1}, for each $j \in [L_2]$, if $\mvX_t^{(\text{d})} = \mvz_j$, then
\begin{align}\label{aux:from_tilde_to_bar}
\begin{split}
  \|\mvX_t\|_{(\bV_{t-1}^{(k)})^{-1}} \leq
    &\|\mvX_t\|_{(\tilde \bV_{t-1}^{(k)}(\mvz_j))^{-1}} \\
    \leq & (1+\sqrt{d}m_X) \|\bar{\mvX}^{(\text{c})}_t\|_{(\bar{\bV}_{t-1}^{(k)}(\mvz_j))^{-1}} \leq (1+\sqrt{d}m_X)^2 (\tilde{C}_3\log(T))^{-1/2},
\end{split}
\end{align}
which, by the definition of $\tilde{C}_3$, implies that for $t > T_1$ and $k \in [2]$, on the event $\bar{\cE}$, $|\text{UCB}_{t}(k) - \mvtheta_{k}'\mvX_{t}| \leq \ell_0/2$, and thus  if $\mvX_{t} \in \cU_{\ell_0}^{(k)}$, regardless of the value of $\mvX_t^{(\text{d})}$,  $k$-th arm would be selected, and then
\begin{align*}
 \lambda_{\min}\left( {\bV}_{t}^{(k)}\right)  \geq 
  \lambda_{\min}\left( \sum_{s=T_1+1}^{t}  {\mvX}_s \mvX_s' I(\mvX \in \cU_{\ell_0}^{(k)})   \right).
\end{align*}
Thus, for $n \geq 2$, if $n T_1 \leq t < (n+1)T_1$, on the event $\bar{\cE}$, for each $k \in [2]$
\begin{align*}
   \lambda_{\min}\left( {\bV}_{t}^{(k)}\right)   \geq C^{-1} \left\lfloor\frac{t-T_1}{T_1} \right\rfloor T_1 \geq C^{-1} \frac{n-1}{n+1} t \geq (3C)^{-1} t,
\end{align*}
which completes the proof.
\end{proof}

\begin{lemma}\label{lemma:discrete_diverse}
Assume  conditions \ref{assumption_parameter_noise} and \ref{cond:discrete} hold. There exists a positive constant $C$, depending only on $\mvTheta_2,d,\lambda$, such that if the truncation time $S \geq T_1$ with $T_1 = \lceil C \log(T) \rceil$, then with probability at least $1-C/T$, $\lambda_{\min}\left( \bar{\bV}_{T_1}^{(k)}(\mvz_j)\right)  \geq \tilde{C}_3  \log(T)$ for each $k \in [2]$, $j \in [L_2]$, where $\tilde{C}_3 = 16 C_3 (1+ \sqrt{d} m_X)^4 \tilde{\ell}_*^{-2}$, $\tilde{\ell}_* = \min\{\ell_2,\ell_0\}$, and  $C_3$, $\ell_0$, $\ell_2$ appear in \eqref{beta_t_Tbound}, \ref{cond:posdef}, and \ref{cond:discrete}  respectively.
\end{lemma}

\begin{proof} In this proof, $C$ is a positive constant, depending only on $\mvTheta_2,d,\lambda$,  that may vary from line to line. By the union bound, it suffices to consider a fixed $j \in [L_2]$. 
By Lemma \ref{lemma:discrete_N_V}, the event $\Gamma_1$ occurs
 with probability at least $1-C/T$,  where $\Gamma_1 =\{\sum_{s=1}^{t} I(\mvX_s^{(\text{d})} = \mvz_j) \geq C^{-1}t$ for all $t \geq C\log(T)\}$. Then due to the condition \ref{cond:discrete}, and by applying Lemma \ref{lemma:cst_adaptive}  \textit{conditional} on the event $\Gamma_1$, for some constant
$C_4 > 0$ depending only on $\mvTheta_2,\lambda, d$, the event $\Gamma_2$ occurs  with probability at least $1-C_4/T$, where $\Gamma_2$ denotes the event that $ \max_{k \in [2]} \lambda_{\min}\left( \bar{\bV}_t^{(k)}(\mvz_j)\right)  \geq \tilde{C}_3 \log(T)$, for all $t \geq C_4\log(T)$.

Further, by Lemma \ref{lemma:discrete_N_V}, there exists some constant
$C_5 > 0$ depending only on $\mvTheta_2, \lambda,  d$, such that  the event $\Gamma_3$ occurs with probability at least $1-C_5/T$,  where $\Gamma_3$ denotes the event that for all $k \in [2]$ and $t_1,t_2 \in [T]$ with $t_2 - t_1 \geq C_5 \log(T)$,
$\lambda_{\min}( \sum_{s=t_1+1}^{t_2}  \bar{\mvX}_s^{(\text{c})} (\bar{\mvX}_s^{(\text{c})})' I(\mvX \in \cU_{\ell_2}^{(k)}, \; \mvX_s^{(\text{d})} = \mvz_j)  ) \geq \tilde{C}_3 \log(T)$.

We focus on the event 
$$
\Gamma := \{\|\hat \mvtheta_{t-1}^{(k)} - \mvtheta_{k}\|_{\bV_{t-1}^{(k)}} \leq \sqrt{\beta_t^{(k)}}, \text{ for } t \in [T], k \in [2]\} \;\cap \; \Gamma_2 \;\cap\; \Gamma_3,
$$
which occurs with probability at least $1-C/T$, due to Lemma \ref{lemma:LinUCB} and above discussions. Let $T_0 = \lceil C_4  \log(T)\rceil$ and $T_1 = T_0 + \lceil C_5 \log(T) \rceil$. Assume that the truncation time $S \geq T_1$. \\

On the event $\Gamma$, at least one of the following cases occurs: (I). $\lambda_{\min}\left( \bar{\bV}_{T_0}^{(1)}(\mvz_j)\right)  \geq \tilde{C}_3 \log(T)$; or (II). $\lambda_{\min}\left( \bar{\bV}_{T_0}^{(2)}(\mvz_j)\right)  \geq \tilde{C}_3 \log(T)$. We first consider case (I). On the event $\Gamma$,
due to \eqref{beta_t_Tbound},  for $t \in (T_0, T_1]$,  
\begin{align*}
    &\text{UCB}_{t}(2)  \geq \mvtheta_{2}'\mvX_t, \qquad
    \text{UCB}_{t}(1) \leq \mvtheta_{1}'\mvX_t + 2 ({C}_3 \log(T))^{1/2} \|\mvX_t\|_{(\bV_{t-1}^{(1)})^{-1}},
\end{align*}
which, due to \eqref{aux:from_tilde_to_bar} and by the definition of $\tilde{C}_3$, implies that if $\mvX_t \in \cU_{\ell_2}^{(2)}$ and $\mvX_t^{(\text{d})} = \mvz_j$, then 
$\text{UCB}_{t}(1)\leq \mvtheta_{1}'\mvX_t  + \ell_2/2 < \text{UCB}_t(2)$, and
thus the second arm would be selected. 
As a result, on the event  $\Gamma$, under the case (I),
\begin{align*}
    \lambda_{\min}\left( \bar{\bV}_{T_1}^{(2)}(\mvz_j)\right) \geq 
    \lambda_{\min}\left(\sum_{s=T_0+1}^{T_1}  \bar{\mvX}_s^{(\text{c})} (\bar{\mvX}_s^{(\text{c})})' I(\mvX \in \cU_{\ell_2}^{(2)}, \; \mvX_s^{(\text{d})} = \mvz_j)   \right)  \geq \tilde{C}_3 \log(T).
\end{align*}
The same argument applies to case (II), and the proof is complete.
\end{proof}

\begin{lemma}
\label{lemma:discrete_N_V}
Assume conditions \ref{assumption_parameter_noise} and \ref{cond:discrete} hold. There exists a positive constant $C$, depending only on $\mvTheta_2, d$, such that 
with probability at least $1-C/T$, 
\begin{align*}
   & \sum_{s=t_1+1}^{t_2} I(\mvX_s^{(\text{d})} = \mvz_j)  \geq \ell_2^2 (t_2-t_1)/(2dm_X^2),\\
 &\lambda_{\min}\left( \sum_{s=t_1+1}^{t_2}  \bar{\mvX}_s^{(\text{c})} (\bar{\mvX}_s^{(\text{c})})' I(\mvX \in \cU_{\ell_2}^{(k)}, \; \mvX_s^{(\text{d})} = \mvz_j)  \right) \geq \ell_2^2 (t_2-t_1)/2,
\end{align*}
for any $t_1,t_2 \in [T]$ with $t_2 - t_1 \geq C \log(T)$, $k = 1,2$, and $j \in [L_2]$.
\end{lemma}
\begin{proof}
The condition \ref{cond:discrete} implies that $\bP(\mvX^{(\text{d})} = \mvz_j) \geq \ell_2^2/(d m_X^2)$. Then the proof for the first claim is complete due to the Hoeffding bound \citep[Proposition 2.5]{wainwright2019high} and the union bound. The proof for the second claim is the same as for  Lemma \ref{lemma:posdef}.
\end{proof}

\subsection{Proof of Lemma \ref{aux:replaced_by_1}}\label{proof:aux:replaced_by_1}

Before proving Lemma \ref{aux:replaced_by_1}, we make the following observation.

\begin{lemma}\label{lemma:dual}
Let $n,d \geq 1$ be integers, and  $\mvz_{1},\ldots, \mvz_{n}$  $\bR^{d}$-vectors. Denote by $\bV = \sum_{i=1}^{n} \mvz_{i}\mvz_{i}'$, and assume that $\bV$ is invertible. Then for any $\mvz \in \bR^d$,
$$
\|\mvz\|_{\bV^{-1}}^2 = \inf\{\|\mvgamma\|^2: \mvgamma \in \bR^{n},\ \sum_{i=1}^{n} \mvgamma_i \mvz_{i} = \mvz\}.
$$
\end{lemma}
\begin{proof}
Solve the optimization problem using the elementary Lagrange multiplier method.
\end{proof}

\begin{proof}[Proof of Lemma \ref{aux:replaced_by_1}]
For any $d \geq 1$, denote by $\mve^{(d)}_i \in \bR^{d}$ the vector with the $i$-th coordinate being $1$, and all other coordinates being $0$. Then
\begin{align*}
&\lambda \bI_{d_1+d_2} + \sum_{i=1}^{n} \tilde{\mvz}_i \tilde{\mvz}_i'= \sum_{i=1}^{d_1+d_2} \sqrt{\lambda} \mve^{(d_1+d_2)}_i (\sqrt{\lambda} \mve^{(d_1+d_2)}_i)' + \sum_{i=1}^{n} \tilde{\mvz}_i \tilde{\mvz}_i',\\
&\lambda \bI_{1+d_2} + \sum_{i=1}^{n} \bar{\mvz}_i \bar{\mvz}_i' = 
\sum_{i=1}^{1+d_2} \sqrt{\lambda} \mve^{(1+d_2)}_i (\sqrt{\lambda} \mve^{(1+d_2)}_i)' + \sum_{i=1}^{n} \bar{\mvz}_i \bar{\mvz}_i'.
\end{align*}

For a vector $\mvgamma$, denote by $\mvgamma_{[i:j]}$ the sub-vector from the $i$-th coordinate to $j$-th. Define $\tilde{\cC}$ to be the collection of $\tilde{\mvgamma}\in \bR^{d_1+d_2+n}$ such that $\sqrt{\lambda} \tilde{\mvgamma}_{[1:d_1]} +\mva (\sum_{i=1}^{n} \tilde{\mvgamma}_{d_1+d_2+i}) = \mva$ and $ 
    \sqrt{\lambda} \sum_{j=1}^{d_2} \tilde{\mvgamma}_{d_1+j} \mve_{j}^{(d_2)}  + \sum_{i=1}^{n}\tilde{\mvgamma}_{d_1+d_2+i} \mvz_i = \mvv$. Further, define $\bar{\cC}$ to be the collection of
    $\bar{\mvgamma}\in \bR^{1+d_2+n}$ such that
    $
    \sqrt{\lambda} \bar{\mvgamma}_{1} +\sum_{i=1}^{n} \bar{\mvgamma}_{1+d_2+i} = 1$ and $
    \sqrt{\lambda} \sum_{j=1}^{d_2} \bar{\mvgamma}_{1+j} \mve_{j}^{(d_2)}  + \sum_{i=1}^{n}\bar{\mvgamma}_{1+d_2+i} \mvz_i = \mvv$. 
    Then by Lemma \ref{lemma:dual},
\begin{align*}
    \tilde{\mvv}' (\lambda \bI_{d_1+d_2} + \sum_{i=1}^{n} \tilde{\mvz}_i \tilde{\mvz}_i')^{-1} \tilde{\mvv}  =  
    \inf_{\tilde{\mvgamma}\in \tilde{\cC}} \|\tilde{\mvgamma}\|^2,\quad
    \bar{\mvv}' (\lambda \bI_{1+d_2} + \sum_{i=1}^{n} \bar{\mvz}_i \bar{\mvz}_i')^{-1} \bar{\mvv} = \inf_{\bar{\mvgamma}\in \bar{\cC}} \|\bar{\mvgamma}\|^2.
\end{align*}
Finally, note that in the constraint set $\tilde{\cC}$, $\tilde{\gamma}_{[1:d_1]}$ must be proportional to $\mva$, and thus
\begin{align*}
    \inf_{\tilde{\mvgamma}\in \tilde{\cC}} \|\tilde{\mvgamma}\|^2
 =\inf_{\bar{\mvgamma}\in \bar{\cC}}\{\ 
 \|\mva\|^2 \bar{\mvgamma}_1^2 +  \|\bar{\mvgamma}_{2:(1+d_2+n)}\|^2 \ \} \leq \max(1, \|\mva\|^2) \inf_{\bar{\mvgamma}\in \bar{\cC}} \|\bar{\mvgamma}\|^2,
\end{align*}
which completes the proof.
\end{proof}


\section{Proofs for some lemmas in the main text}\label{app:lemma_more_proofs} 
In this section, we present the proofs for Lemma  \ref{lemma:LinUCB}, \ref{lemma:continuity}, and \ref{lemma:logconcave_example}.

\subsection{Proof of Lemma \ref{lemma:LinUCB}} \label{proof:lemma:LinUCB}
\begin{proof} 
Fix some $k \in [K]$. Define for each $t \in [T]$,
\begin{align*}
    \tilde{\mvX}_t = \mvX_t I(A_t = k), \quad
    \tilde{\epsilon}_t = \epsilon_t^{(k)} I(A_t = k), \quad
    \tilde{Y}_t = \mvtheta_k'  \tilde{\mvX}_t +  \tilde{\epsilon}_t.
\end{align*}
By definition, $\bV_t^{(k)} = \lambda \bI_d + \sum_{s \in [t]} \tilde{\mvX}_s\tilde{\mvX}_s'$, $\mvU^{(k)}_t = \sum_{s \in [t]} \tilde{\mvX}_s\tilde{Y}_s$.
Define the filtration $\{\cH_t: t \geq 0\}$, where $\cH_t = \sigma(\mvX_s,Y_s, \xi_s, s \leq t; \mvX_{t+1}, \xi_{t+1})$, and recall that $\xi_t$ is the random mechanism at time $t$, e.g., to break ties. Then $\{\tilde{\mvX}_t,\tilde{Y}_t: t \in [T]\}$ are adapted $\{\cH_t: t \geq 0\}$, and $\tilde{\mvX}_{t} \in \cH_{t-1}$ for $t \geq 1$. Due to the condition \ref{assumption_parameter_noise},
 $\Exp[e^{\tau \tilde{\epsilon}_t}\vert \cH_{t-1}] \leq e^{\tau^2 \sigma^2/2}$ for any $\tau \in \bR$ almost surely for $t \geq 1$. Then the proof is complete due to  \citep[Theorem 2]{abbasi2011improved},  and the union bound.
\end{proof}

\subsection{Proof of Lemma \ref{lemma:continuity}}\label{proof:lemma:continuity}
We start with the part (i).  For any $\mvu \in \cS^{d-1}$, denote by $\mvu^{(-1)} \in \bR^{d-1}$ the vector after removing the first coordinate of $\mvu$, and by $\mvu_1$ the first coordinate of $\mvu$.   

\begin{proof}[Proof of Lemma \ref{lemma:continuity}(i)]
Consider the first case that $\mvX$ has a Lebesgue density on $\bR^d$ that is bounded by $C$. By Lemma \ref{lemma:density}, there exists some constant $\tilde{C} > 0$, depending only on $d, C, m_{X}$, such that the density of $\mvu'\mvX$ is bounded by $\tilde{C}$ for any $\mvu \in \cS^{d-1}$. Then \ref{cond:continuity} holds with $\ell_0 = 1/(8\tilde{C})$.

Now consider the second case that $d \geq 2$, $\mvX = (1;{\mvX}^{(-1)})$, and ${\mvX}^{(-1)} = \tilde{\mvX}$ has a Lebesgue density on $\bR^{d-1}$ that is bounded above $C$. 

For $\mvu \in \cS^{d-1}$, if $\|\mvu^{(-1)}\| \geq 1/(2\sqrt{d} m_X+1)$, then by Lemma \ref{lemma:density}, there exists some constant $\tilde{C} > 0$, depending only on $d, m_{X}, C$, such that the density of 
$(\mvu^{(-1)}/\|\mvu^{(-1)}\|)'\mvX^{(-1)}$ is bounded by $\tilde{C}$.
Thus for any $\tau > 0$,  
\begin{align*}
    \bP(|\mvu'\mvX| \leq \tau) &= \bP\left( \frac{-\tau -\mvu_1}{\|\mvu^{(-1)}\|}
     \leq 
    \frac{(\mvu^{(-1)})'\mvX^{(-1)}}{\|\mvu^{(-1)}\|} \leq \frac{\tau -\mvu_1}{\|\mvu^{(-1)}\|}\right) 
    \leq  2\tilde{C} (2\sqrt{d} m_X+1) \tau.
\end{align*}
Then \ref{cond:continuity} holds with $\ell_1 \leq 1/(8\tilde{C}(2\sqrt{d} m_X+1))$.

For $\mvu \in \cS^{d-1}$, if $\|\mvu^{(-1)}\| < 1/(2\sqrt{d} m_X+1)$, which, by the triangle inequality, implies $|\mvu_1| > 2\sqrt{d} m_X/(2\sqrt{d} m_X+1)$, then
\begin{align*}
    |\mvu_1 + (\mvu^{(-1)})'\mvX^{(-1)}| > 2\sqrt{d} m_X/(2\sqrt{d} m_X+1) -  \sqrt{d} m_X/(2\sqrt{d} m_X+1).
\end{align*}
Thus if we let $\ell_1 \leq  \sqrt{d} m_X/(2\sqrt{d} m_X + 1)$, then $\Pro(|\mvu'X| \leq \ell_1) =0$. Combining two cases for $\mvu \in \cS^{d-1}$ completes the proof.
\end{proof}


For part (ii), recall that $p_{\tilde{\mvX}}$ is log-concave,  $\|\Exp[\tilde{\mvX}]\| \leq C$, and the eigenvalues of $\Cov(\tilde{\mvX})$ are between $[C^{-1},C]$. 

\begin{proof}[Proof of Lemma \ref{lemma:continuity}(ii)]
Consider the first case that $\mvX$ has no intercept, i.e., $\tilde{\mvX} = \mvX$. By Lemma \ref{app:logconcave_proj},  there exists some constant $\tilde{C} > 0$, depending only on $C$, such that for any $\mvu \in \cS^{d-1}$,  the density of $\mvu'\mvX$ is bounded by $\tilde{C}$, which implies that  \ref{cond:continuity} holds with $\ell_0 = 1/(8\tilde{C})$.

Now consider the second case that $d \geq 2$, $\mvX = (1;{\mvX}^{(-1)})$, and ${\mvX}^{(-1)} = \tilde{\mvX}$.  Let $\mvu \in \cS^{d-1}$. 
If $|\mvu_1| = 1$ and $\ell_1 \in (0,1)$, then $ \bP(|\mvu'\mvX| \leq \ell_1) = 0$. 
Thus we focus on those $\mvu \in \cS^{d-1}$ such that $|\mvu_1| < 1$, 
and denote by $p_{\mvu}$ the density of $(\mvu^{(-1)}/\|\mvu^{(-1)}\|)'\mvX^{(-1)}$. By Lemma \ref{app:logconcave_proj},  there exists some constant $\tilde{C} > 0$, depending only on $C$, such that $p_{\mvu}(\tau) \leq \tilde{C} \exp(-|\tau|/\tilde{C}) \leq \tilde{C}$ for any $\tau \in \bR$.  
Let $\epsilon \in (0,1/2)$ be a constant to be specified.

If $\|\mvu^{(-1)}\| \geq \epsilon$, then for any $\tau \geq 0$,
\begin{align*}
    \bP(|\mvu'\mvX| \leq \tau) &= \bP\left( \frac{-\tau -\mvu_1}{\|\mvu^{(-1)}\|}
     \leq 
    \frac{(\mvu^{(-1)})'\mvX^{(-1)}}{\|\mvu^{(-1)}\|} \leq \frac{\tau -\mvu_1}{\|\mvu^{(-1)}\|}\right) 
    \leq  2\tilde{C} \epsilon^{-1} \tau.
\end{align*} 
Now consider $0 < \|\mvu^{(-1)}\| < \epsilon$, which implies that $1/2 < 1 - \epsilon^2 < |\mvu_1| < 1$. If $\mvu_1 > 0$, then for any $\tau \in (0,1/4]$,
\begin{align*}
    \bP(|\mvu'\mvX| \leq \tau)  \leq  \int_{-\infty}^{(\tau-\mvu_1)/\|\mvu^{(-1)}\|} \tilde{C} \exp(-|x|/\tilde{C}) dx \leq \int_{-\infty}^{-4^{-1}\epsilon^{-1}} \tilde{C} \exp(-|x|/\tilde{C}) dx.
\end{align*} 
The same is true when $\mvu_1 < 0$. Thus   there exists some constant $\epsilon^* \in (0,1/2)$, depending only on $\tilde{C}$, such that $\bP(|\mvu'\mvX| \leq \tau) \leq 1/4$ for any $\tau \in (0,1/4]$ if $\|\mvu^{-1}\| < \epsilon^*$. Combining these two cases, we have \ref{cond:continuity} holds with $\ell_1 = \min\{\epsilon^*/(8\tilde{C}), 1/4\}$.
\end{proof}

\subsection{Proof of Lemma \ref{lemma:logconcave_example}}\label{proof:lemma:logconcave_example}

\begin{proof}
Fix any $\mvu,\mvv \in \cS^{d-1}$, and let $\gamma = \mvu'\mvv$. If $\gamma \in (-1,1)$, let $\mvw = (\mvv - \gamma \mvu)/\sqrt{1-\gamma^2}$. If $\gamma \in \{-1,1\}$, let $\mvw \in \cS^{d-1}$ be any unit vector such that $\mvu'\mvw = 0$. In either case, $\mvv = \gamma \mvu + \sqrt{1-\gamma^2} \mvw$  and $\mvu'\mvw=0$.  Denote by $f$ the joint density of $(\mvu'\mvX, \mvw'\mvX)$. By Lemma \ref{app:logconcave_proj},  there exists some constant $\tilde{C} > 0$, depending only on $C$, such that 
$$
\inf_{|\tau_1|^2 + |\tau_2|^2  \leq \tilde{C}^{-2}}f(\tau_1,\tau_2) \geq \tilde{C}^{-1}, \quad
f(\tau_1,\tau_2) \leq \tilde{C} \exp(-\sqrt{\tau_1^2 + \tau_2^2}/\tilde{C})\;\; \text{ for } \tau_1,\tau_2 \in \bR.
$$
Note that the first part requires $\mvX$ to be centered, while the second part does not. 
By a change-of-variable, i.e., from $(\tau_1,\tau_2)$ to $(r\sin(\theta), r\cos(\theta))$,  $\Exp[|\mvu'\mvX| I(\textup{sgn}(\mvu'\mvX) \neq \textup{sgn}(\mvv'\mvX))] 
$ is given by
\begin{align*}
    & \int_{0}^{\infty} \int_{0}^{\pi} r \sin(\theta) I( \gamma \sin(\theta) + \sqrt{1-\gamma^2} \cos(\theta) < 0) f(r\sin(\theta),r \cos(\theta))\ rdrd\theta\\
  +  \;\; &     \int_{0}^{\infty} \int_{\pi}^{2\pi} (-r \sin(\theta)) I( \gamma \sin(\theta) + \sqrt{1-\gamma^2} \cos(\theta) > 0) f(r\sin(\theta),r \cos(\theta))\ rdrd\theta,
\end{align*}
which, together with the lower and upper bound on $f$, implies that
\begin{align*}
     &\left(\int_{0}^{\tilde{C}^{-1}} \tilde{C}^{-1} r^2 dr\right) \left(\int_{0}^{\pi}  \sin(\theta) I( \gamma \sin(\theta) + \sqrt{1-\gamma^2} \cos(\theta) < 0) d\theta \right)\\
\leq \;\;& \Exp[|\mvu'\mvX| I(\textup{sgn}(\mvu'\mvX) \neq \textup{sgn}(\mvv'\mvX))]\\
\leq \;\;& 2\left(\int_{0}^{\infty} \tilde{C} \exp(-r/\tilde{C}) r^2  dr\right) \left(\int_{0}^{\pi}   \sin(\theta) I( \gamma \sin(\theta) + \sqrt{1-\gamma^2} \cos(\theta) < 0) d\theta \right).
\end{align*}
Now let $\alpha = \arccos(\gamma) \in [0,\pi]$. By elementary calculation, we have
\begin{align*}
    &\int_{0}^{\pi}   \sin(\theta) I( \gamma \sin(\theta) + \sqrt{1-\gamma^2} \cos(\theta) < 0) d\theta \\
=& \int_{\pi-\alpha}^{\pi} \sin(\theta) d\theta = 
1 - \cos(\alpha) = 1 - \mvu'\mvv = \|\mvu - \mvv\|^2/2,
\end{align*}
which completes the proof.
\end{proof}

\section{Proof of the upper bound part in Theorem \ref{theorem:lower_bound}}\label{proof:theorem:lower_bound_upper}
Here, we provide the proof for the upper bound part in Theorem \ref{theorem:lower_bound}, and 
the lower bound part is in Section \ref{sec:proof_lower_bound}.

In view of the part (i) of Corollary \ref{cor:regret_d} for the proposed Tr-LinUCB algorithm, it suffices to show that the problem instances in \ref{prob_instances_all_lower}  verify the conditions \ref{assumption_parameter_noise}-\ref{cond:unit_sphere}. We consider the case involving log-concave densities in Subsection \ref{app:log-concave_veri} and the case of $\textup{Unif}(\sqrt{d}\cS^{d-1})$ in Subsection \ref{app:spheres_veri}.

\subsection{Verification related to log-concave densities}\label{app:log-concave_veri}
In this subsection, the distribution $F$ of the context vector $\mvX$
has an isotropic log-concave density and $\|\mvX\| \leq \sqrt{d} m_{X}$ almost surely.
 
It is clear that the condition \ref{assumption_parameter_noise} holds with $m_{\theta} = 1$, $m_{R} = 1$, $\sigma^2 = 1$.

Since $\|\mvtheta_1-\mvtheta_2\| \in [1/2,1]$, by Lemma \ref{app:logconcave_proj}, the density of $(\mvtheta_1-\mvtheta_2)'\mvX$ is uniformly bounded by some absolute constant $\tilde{C} > 0$. Thus the condition \ref{cond:margin} holds with $L_0 = 2\tilde{C}$.

The conditions \ref{cond:continuity} and \ref{cond:unit_sphere} are verified  in Lemma \ref{lemma:continuity} and \ref{lemma:logconcave_example} respectively.

Now we focus on the verification of the condition \ref{cond:posdef}. Fix any $\mvu,\mvv \in \cS^{d-1}$, and $\gamma = \mvu'\mvv$. If $\gamma \in (-1,1)$, let $\mvw = (\mvv - \gamma \mvu)/\sqrt{1-\gamma^2}$. If $\gamma \in \{-1,1\}$, let $\mvw \in \cS^{d-1}$ be any unit vector such that $\mvu'\mvw = 0$. In either case, $\mvv = \gamma \mvu + \sqrt{1-\gamma^2} \mvw$  and $\mvu'\mvw=0$.  Denote by $f$ the joint density of $(\mvu'\mvX, \mvw'\mvX)$. Then
\begin{align*}
\Exp[(\mvv'\mvX)^2 I(\mvu'\mvX > \delta)] = \int_{\bR^2} (\gamma \tau_1 + \sqrt{1-\gamma^2} \tau_2)^2 I(\tau_1 > \delta) f(\tau_1,\tau_2)  d\tau_1 d\tau_2.
\end{align*}
 By Lemma \ref{app:logconcave_proj},  for some absolute constant $c > 0$,   $\inf_{\max\{|\tau_1|, |\tau_2|\}  \leq c} f(\tau_1,\tau_2) \geq c$. Thus for any $\delta \in (0,c/2)$,
 \begin{align*}
     \Exp[(\mvv'\mvX)^2 I(\mvu'\mvX > \delta)] &\geq \int_{c/2}^{c} \int_{-c}^{c}
 c (\gamma \tau_1 + \sqrt{1-\gamma^2} \tau_2)^2 d\tau_1 d\tau_2\\
 &= 7c^5 \gamma^2/12 + c^5 ( 1-\gamma^2)/3 \geq c^5/3.
 \end{align*}
 In particular, there exists some absolute constant $c^* > 0$ such that for any $\mvu,\mvv \in \cS^{d-1}$, $ \Exp[(\mvv'\mvX)^2 I(\mvu'\mvX > 2c^*)] \geq (c^*)^2$. Since $\|\mvtheta_2 - \mvtheta_1\| \in [1/2,1]$, the condition \ref{cond:posdef} holds with $\ell_1 = c^*$. Thus the verification of   conditions \ref{assumption_parameter_noise}-\ref{cond:unit_sphere} for 
 the problem instances in \ref{prob_instances_all_lower}, when $F$  
has an isotropic log-concave density and $\|\mvX\| \leq \sqrt{d} m_{X}$ almost surely, is complete.

\subsection{Verification related to spheres}\label{app:spheres_veri}
In this subsection, we verify  conditions \ref{assumption_parameter_noise}-\ref{cond:unit_sphere} for 
 the problem instances in \ref{prob_instances_all_lower}, with $F$ being  $\textup{Unif}(\sqrt{d}\cS^{d-1})$. Recall that $d \geq 3$.

Denote by $\mvPsi = (\mvPsi_1,\mvPsi_2,\ldots,\mvPsi_d)$  a random vector with the uniform distribution on the sphere with center at the origin and radius $\sqrt{d}$, i.e., $\textup{Unif}(\sqrt{d}\cS^{d-1})$; thus, $\mvPsi_j$ is the $j$-th component of $\mvPsi$ for $j \in [d]$. To avoid confusion, we use the notation $\mvPsi$ for the context $\mvX$. 
Since $\|\mvtheta_1 - \mvtheta_2\| \in [1/2,1]$, we may assume  $\|\mvtheta_1 - \mvtheta_2\|=1$ without loss of generality.

It is clear that the condition \ref{assumption_parameter_noise} holds with $m_{\theta} = 1$, $m_{R} = 1$, $m_X = 1$, and $\sigma^2 = 1$.


By Lemma \ref{aux:unit_sphere}, for any $\mvu \in \cS^{d-1}$, $\bP(|\mvu'\mvPsi| \leq \tau) \leq 2\tau$  for any $\tau \geq 0$ and $\Exp[|\mvu' \mvPsi|] \leq 1$, which   verifies the condition \ref{cond:margin} with $L_0=2$,  and the condition \ref{cond:continuity} with $\ell_1 = 1/8$.  The condition \ref{cond:unit_sphere} is verified in Lemma \ref{app:unit_sphere_lu} with $L_1 = \sqrt{2}$. 

By Lemma \ref{app:sphere_smallest_eigen} and due to symmetry, for any $\mvu,\mvv \in \cS^{d-1}$ and $\ell_0 \in (0,1/4)$, 
\begin{align*}
    \Exp[(\mvv'\mvPsi)^2 I(\mvu'\mvPsi \geq \ell_0)] = 2^{-1} (\Exp[(\mvv'\mvPsi)^2] - \Exp[(\mvv'\mvPsi)^2 I(|\mvu'\mvPsi| \leq \ell_0)]) \geq 2^{-1}(1-4\ell_0),
\end{align*}
which implies that  the  condition \ref{cond:posdef} holds with $\ell_0 = 1/8$. 

\begin{lemma}\label{aux:unit_sphere}
For each $\mvu \in \cS^{d-1}$, denote by $\phi_{\mvu}^{(d)}$ the Lebesgue density of $\mvu'\mvPsi$. Then for each $\mvu \in \cS^{d-1}$,   
 $\phi_{\mvu}^{(d)}$ is non-increasing on $(0,\sqrt{d})$, and for some absolute constant $C > 0$, $C^{-1}  \leq \phi_{\mvu}^{(d)}(1) \leq \phi_{\mvu}^{(d)}(0) \leq 1$, and $C^{-1} \leq \Exp[|\mvu'\mvPsi|] \leq 1$.
\end{lemma}
\begin{proof}
Due to rotation invariance, for each $\mvu \in \cS^{d-1}$, $\mvu'\mvPsi$ has the same distribution as $\mvPsi_1$, the first component of $\mvPsi$. Denote by $\phi_1^{(d)}$ the density of $\mvPsi_1$. It is elementary 
that for $\tau \in (-\sqrt{d},\sqrt{d})$, 
$$
\phi_1^{(d)}(\tau) = \frac{\Gamma(d/2)}{\sqrt{d}\ \Gamma((d-1)/2)\Gamma(1/2)} (1-\frac{\tau^2}{d})^{(d-3)/2},
$$
where $\Gamma(\cdot)$ is the gamma function. Thus $\phi_{1}^{(d)}$ is non-increasing on $(0,\sqrt{d})$ for $d \geq 3$. By the Gautschi's inequality, 
$\sqrt{d/2-1} \leq \Gamma(d/2)/\Gamma((d-1)/2) \leq \sqrt{d/2}$. It is elementary that $\inf_{d \geq 3} (1-1/d)^{(d-3)/2} > 0$, which implies that $C^{-1}  \leq \phi_{1}^{(d)}(1) \leq \phi_{1}^{(d)}(0) \leq 1$ for some absolute constant $C > 0$. Finally, since $\Exp[|\mvPsi_1|] \geq \phi_{1}^{(d)}(1)  \int_0^1 \tau d\tau$, the lower bound follows. The upper bound is since $\Exp[|\mvPsi_1|] \leq \sqrt{\Exp[\mvPsi_1^2]} = 1$.
\end{proof}

\begin{lemma}\label{app:unit_sphere_lu}
There exists an absolute constant $C>0$ such that for any $\mvu, \mvv \in \cS^{d-1}$, 
\begin{align*}
   C^{-1} \|\mvu-\mvv\|^2 \leq  \Exp[|\mvu'\mvPsi|\  I(\textup{sgn}(\mvu'\mvPsi) \neq \textup{sgn}(\mvv'\mvPsi))] \leq \sqrt{2}\|\mvu-\mvv\|^2.
\end{align*}
\end{lemma}
\begin{proof}
Let $\alpha = \arccos(\mvu'\mvv)\in [0,\pi]$. Due to rotation invariance,
$(\mvu'\mvPsi, \mvv'\mvPsi)$ has the same distribution as $(\mvPsi_1, \cos(\alpha) \mvPsi_1 + \sin(\alpha) \mvPsi_2)$, where $\mvPsi_1$ and $\mvPsi_2$ are the first and second component of $\mvPsi$ respectively.
Thus 
\begin{align*}
\Exp[|\mvu'\mvPsi|\  I(\textup{sgn}(\mvu'\mvPsi) \neq \textup{sgn}(\mvv'\mvPsi))] = 2\Exp[|\mvPsi_1|]I(\mvPsi_1 >0, \cos(\alpha) \mvPsi_1 + \sin(\alpha) \mvPsi_2 < 0).
\end{align*}


For $r \in (0,\sqrt{d})$, conditional on $\mvPsi_1^2 + \mvPsi_2^2 = r^2$, $(\mvPsi_1/r, \mvPsi_2/r)$ has the same distribution as $(\sin(\zeta), \cos(\zeta))$, where $\zeta$ has uniform distribution on $(0,2\pi)$. Thus  
\begin{align*}
    &\Exp\left[|\mvPsi_1|\ I(\mvPsi_1 >0, \cos(\alpha) \mvPsi_1 + \sin(\alpha) \mvPsi_2 < 0) \ \vert \mvPsi_1^2 + \mvPsi_2^2 = r^2\right] \\
    =&
    r \Exp[|\sin(\zeta)| I(\sin(\zeta) > 0,\ \sin(\zeta + \alpha) < 0)] \\
    =& r \int_{\pi-\alpha}^{\pi} \sin(\tau) d\tau = r(1 - \cos(\alpha)) =2^{-1}r\|\mvu-\mvv\|^2.
\end{align*}
As a result, $\Exp[|\mvu'\mvPsi|\  I(\textup{sgn}(\mvu'\mvPsi) \neq \textup{sgn}(\mvv'\mvPsi))] =  \Exp\left[(\mvPsi_1^2 + \mvPsi_2^2)^{1/2}\right]\|\mvu-\mvv\|^2$. Since
\begin{align*}
 \Exp[|\mvPsi_1|]\leq    \Exp\left[(\mvPsi_1^2 + \mvPsi_2^2)^{1/2}\right] \leq (\Exp[\mvPsi_1^2 + \mvPsi_2^2])^{1/2} =\sqrt{2},
\end{align*}
the proof is complete due to Lemma \ref{aux:unit_sphere}.
\end{proof}

\begin{lemma}\label{app:sphere_smallest_eigen}
For any   $\mvu, \mvv \in \cS^{d-1}$  and $\ell \in (0,4)$, $\Exp[(\mvv'\mvPsi)^2 I(|\mvu'\mvPsi| \leq \ell)] \leq 4\ell$.
\end{lemma}
\begin{proof}
Fix $\mvu,\mvv \in \cS^{d-1}$, and denote by $\gamma = \mvu'\mvv$. Due to rotation invariance, $(\mvu'\mvPsi, \mvv'\mvPsi)$ has the same distribution as $(\mvPsi_1, \gamma \mvPsi_1 + \sqrt{1-\gamma^2} \mvPsi_2)$. Since $\Exp[\mvPsi_2\ \vert \ \mvPsi_1] = 0$ and 
$\Exp[\mvPsi_2^2 \ \vert \ \mvPsi_1] = (d-\mvPsi_1^2)/(d-1)$, we have
\begin{align*}
    &\Exp\left[ (\mvv'\mvPsi)^2\ I\left(|\mvu'\mvPsi| \leq \ell \right)\right] 
    =  \Exp\left[ (\gamma \mvPsi_1 + \sqrt{1-\gamma^2} \mvPsi_2)^2 \ I\left(|\mvPsi_1| \leq \ell \right)\right]   \\
   \leq & \gamma^2 \ell^2 +  (1-\gamma^2) (d/(d-1)) \bP(|\mvPsi_1| \leq \ell).
\end{align*}
Since $d/(d-1) \leq 2$ for $d \geq 3$, and due to Lemma \ref{aux:unit_sphere}, we have
$\Exp\left[ (\mvv'\mvPsi)^2\ I\left(|\mvu'\mvPsi| \leq \ell \right)\right] \leq \gamma^2 \ell^2 + 4(1-\gamma^2) \ell$, which completes the proof.
\end{proof}

\section{Auxiliary Results}
In this section, we provide supporting results and calculations.

\subsection{An application of the Talagrand's concentration inequality}

Let $d \geq 1$ be an integer, and $h > 0$. For $\mvu,\mvv \in \cS^{d-1}, \mvz \in \bR^{d}$, define $\tilde{\phi}_{\mvu}(\mvz) = I(\mvu'\mvz \geq h)$ and
$\phi_{\mvu, \mvv}(\mvz) = I\left(|\mvu' \mvz| \geq h,\; |\mvv'\mvz| \geq h \right)$. Denote by $\cG= \{\phi_{\mvu,\mvv}: \mvu,\mvv \in \cS^{d-1}\}$, and 
$\tilde{\cG} = \{\tilde{\phi}_{\mvu}: \mvu \in \cS^{d-1}\}$. Since $\phi_{\mvu, \mvv}=\tilde{\phi}_{\mvu}\tilde{\phi}_{\mvv}+\tilde{\phi}_{\mvu}\tilde{\phi}_{-\mvv}+\tilde{\phi}_{-\mvu}\tilde{\phi}_{\mvv}+\tilde{\phi}_{-\mvu}\tilde{\phi}_{-\mvv}$, we have $\cG \subset \tilde{\cG}\cdot \tilde{\cG} + \tilde{\cG}\cdot \tilde{\cG}+\tilde{\cG}\cdot \tilde{\cG}+\tilde{\cG}\cdot \tilde{\cG}$, where for two families, $\cG_1,\cG_2$, of functions, $\cG_1\cdot\cG_2 = \{g_1g_2: g_1 \in \cG_1, g_2 \in \cG_{2}\}$ and 
$\cG_1 + \cG_2 = \{g_1+g_2: g_1 \in \cG_1, g_2 \in \cG_{2}\}$

For a probability measure $Q$ and a function $g$ on  $\bR^d$, denote by $\|g\|_{L_2(Q)} = (\int g^2 dQ)^{1/2}$ the $L_2$-norm of $g$ relative to $Q$. 
Let $\cG$ be a family of functions on $\bR^d$. A function $G:\bR^d \to \bR$ is said to be an envelope function for $\cG$ if
$\sup_{g \in \cG}|g(\cdot)| \leq G(\cdot)$.
For $\epsilon > 0$, denote by $N(\epsilon,\cG,L_2(Q))$ the $\epsilon$ covering number of the class $\cG$ under the $L_2(Q)$ semi-metric.



\begin{lemma}\label{lemma:Talagrand}
Let $\mvZ_1,\ldots,\mvZ_n$ be i.i.d.~$\bR^{d}$-random vectors. There exists an absolute constant $C > 0$ such that for any $\tau >0$,
$$
\bP\left( \Delta_n \leq  C (\sqrt{dn} + \sqrt{n \tau} + \tau) \right) \geq 1-e^{-\tau},
$$
where $\Delta_n = \sup_{\mvu,\mvv \in \cS^{d-1}} \left|\sum_{i=1}^{n}( \phi_{\mvu,\mvv}(\mvZ_i) - \Exp[\phi_{\mvu,\mvv}(\mvZ_i)]) \right|$.
\end{lemma}
\begin{proof}
In this proof, $C$ is an absolute constant that may differ from line to line.  Fix $\tau > 0$. 
By the Talagrand’s inequality \citep[Theorem 3.3.9]{gine2021mathematical} (with $U=\sigma^2=1$ therein),  $ \bP(\Delta_n \geq  \Exp[\Delta_n] + \sqrt{2(2\Exp[\Delta_n] + n)\tau} + \tau/3 ) \leq e^{-\tau}$.
Since $\sqrt{2(2\Exp[\Delta_n] + n)\tau}
\leq \sqrt{4\Exp[\Delta_n]\tau} + \sqrt{2n\tau} \leq \Exp[\Delta_n] + \tau + \sqrt{2n\tau}$, we have  
$$
\bP\left(\Delta_n \geq  2(\Exp[\Delta_n] + \sqrt{n\tau} + \tau) \right) \leq e^{-\tau}.
$$

Next, we bound $\Exp[\Delta_n]$. 
Recall the definition of VC-subgraph class in \citep[Chapter 9]{kosorok2008introduction}.  We use the constant function $1$ as the envelope function for both $\cG$ and $\tilde{\cG}$. By \citep[Lemma 9.8, 9.12, and Theorem 9.2]{kosorok2008introduction}, $\tilde{G}$ is a VC-subgraph class with dimension at most $d+2$, and thus  $\sup_{Q}N(\epsilon,\tilde{G},L_2(Q)) \leq (C/\epsilon)^{4d}$ for $\epsilon \in (0,1)$, where  the supremum is taken over all discrete probability measures $Q$ on $\bR^d$.  Since $\cG \subset \tilde{\cG}\cdot \tilde{\cG} + \tilde{\cG}\cdot \tilde{\cG}+\tilde{\cG}\cdot \tilde{\cG}+\tilde{\cG}\cdot \tilde{\cG}$, by \citep[Lemma A.6 and Corollary A.1]{chernozhukov2014gaussian},   $\sup_{Q} N(\epsilon,G,L_2(Q)) \leq (C/\epsilon)^{32d}$ for $\epsilon \in (0,1)$. Then by the entropy integral bound \citep[Theorem 2.14.1]{van1996weak}, 
$$
\Exp[\Delta_n] \leq \sqrt{n} \int_0^{1} \sup_Q\sqrt{1+\log N(\epsilon,\cG,L_2(Q))} d\ \epsilon \leq C \sqrt{n d},
$$
which completes the proof.
\end{proof}

\subsection{An application of van Trees' inequality for lower bounds}\label{app:van_tree}
Let $n \geq 1$, $d \geq 2$ and assume $\sigma^2 > 0$ is known. Let $\{\mvZ_{n}: n \in \bN_{+}\}$ be a sequence of i.i.d.~$\bR^d$ random vectors, with $\Exp[\|\mvZ_{1}\|^2] < \infty$, independent from  $\{\epsilon_n: n \in \bN_{+}\}$, which are i.i.d.~$N(0,\sigma^2)$ random variables. Let $\mvTheta$ be an $\bR^{d}$ random vector with a Lebesgue density  ${\rho}_d(\cdot)$ given in \eqref{LB:theta_density},  supported on $ \cB_d(1/2,1) = \{\mvx \in \bR^{d}: 2^{-1} \leq \|\mvx\| \leq 1\}$; in particular, $\|\mvTheta\|$ has a Lebesgue density given by $\rho(\cdot)$, and $\mvTheta/\|\mvTheta\|$ has the uniform distribution over $\cS^{d-1}$. Further, $\text{ for } n \in \bN_{+}$, define
\begin{align}
    Y_n = \mvTheta'\mvZ_{n} + \epsilon_{n}, \quad \text{ and } \quad 
    \cH_n = \sigma(\mvZ_m,Y_m: m \in [n]).
\end{align}
Thus, $\{(\mvZ_m,Y_m): m \in [n]\}$ are the first $n$ i.i.d.~data points, and the goal is to estimate $\mvTheta$, which is a random vector in this subsection. Further,  any (nonrandom) admissible estimator of $\mvTheta$ must be $\cH_n$ measurable.

\begin{theorem}\label{app:lower_bound_van_Tree}
Let $\xi$ be a $\textup{Unif}(0,1)$ random variable that is independent from $\cH_n$ and $\mvTheta$. There exists an absolute constant $C > 0$ such that for any $\bR^{d}$ random vector $\hat{\mvpsi}_n \in \sigma(\cH_{n}, \xi)$,
$$
\Exp\left[\left\|\hat \mvpsi_n - {\mvTheta}/{\|\mvTheta\|} \right\|^2 \right] \geq \sigma^2{(d-1)^2}/{(n\Exp[\|\mvZ_{1}\|^2] + C d^2\sigma^2)}.
$$
\end{theorem}

\begin{proof}
We follow the approach in \citep{gill1995applications}. 
The conditional density of $\mvD = (\mvZ_1,Y_1)$,  given $\mvTheta = \mvtheta$, is
$f(\mvD; \mvtheta) =  (2\pi \sigma^2)^{-1/2} \exp\left( -{(Y_1 - \mvtheta'\mvZ_1)^2}/{(2\sigma^2)}\right)$. The Fisher information matrix for $\mvTheta$ is
\begin{align*}
    \cI_{\mvtheta} = \Exp\left[ \left(\frac{\partial \log f(\mvD; \mvTheta)}{\partial \mvtheta}\right)' \left(\frac{\partial \log f(\mvD; \mvTheta)}{\partial \mvtheta}\right) \right] = \frac{1}{\sigma^4}\Exp\left[   (\mvZ_1 \mvZ_1')\epsilon_1^2 \right].
\end{align*}
In particular, $\text{trace}(\cI_{\mvtheta}) = \Exp[\|\mvZ_{1}\|^2]/\sigma^2 $. Further, the information for the prior  ${\rho}_d(\cdot)$  in \eqref{LB:theta_density} is 
\begin{align*}
    \tilde{\cI}_{{\rho}_d} =  \Exp\left[\sum_{i=1}^{d} \left(\frac{\partial \log \rho_d(\mvTheta)}{\partial \mvtheta_i}\right)^2 \right] = \Exp\left[\left(\frac{\tilde{\rho}'(\|\mvTheta\|)}{\tilde{\rho}(\|\mvTheta\|)} - \frac{d-1}{\|\mvTheta\|} \right)^2 \right].
\end{align*}
Since $\|\mvTheta\| \geq 2^{-1}$ and $\|\mvTheta\|$ has the Lebesgue density $\tilde \rho(\cdot)$, we have  $\tilde{\cI}_{\rho_d} \leq Cd^2$. Finally, let $\mvpsi(\mvtheta) = \mvtheta/\|\mvtheta\|$ for $\mvtheta \in \cB_d(1/2,1)$.  Then for  $\mvtheta \in \cB_d(1/2,1)$,
\[
\frac{\partial \mvpsi_i(\mvtheta)}{\partial \mvtheta_i} = \frac{1 }{\|\mvtheta\|} - \frac{\mvtheta_i^2}{\|\mvtheta\|^3} \;\; \Longrightarrow \;\; \sum_{i=1}^{d} \frac{\partial \mvpsi_i(\mvtheta)}{\partial \mvtheta_i} = \frac{d-1}{\|\mvtheta\|} \geq d-1.
\]
Now by \citep[Theorem 1]{gill1995applications} with $B(\cdot)= C(\cdot) = \bI_{d}$, and since there always exists a non-random Bayes rule, we have
$\Exp[\|\hat \mvpsi_n- \mvpsi(\mvTheta)\|^2] \geq {(d-1)^2}/{(n\Exp[\|\mvZ_{1}\|^2]/\sigma^2 + C d^2)}$.
\end{proof}

\begin{remark}
The random variable $\xi$ in the above theorem is used to model additional information that is independent from data. 
\end{remark}

\subsection{About log-concave densities}\label{app:logconcave}
Let $p:\bR^d \to [0,\infty)$ be a probability density function with respect to the Lebesgue measure on $\bR^d$. We say $p$ is log-concave if $\log(p): \bR^d \to [-\infty,\infty)$ is concave, and is isotropic if $\Exp[\mvZ] = \boldsymbol{0}_{d}$ and $\Cov(\mvZ) = \bI_d$ for a random vector $\mvZ$ with the density $p$. We consider upper semi-continuous log-concave densities, just to fix a particular version. 
In the main text, we apply the following lemmas for  $m = 1$ or $2$.

\begin{lemma}\label{app:lemma_isotropic}
Let $m \geq 1$ be an integer. There exists a constant $C > 0$, depending only on $m$, such that for any isotropic, log-concave densities $p$ on $\bR^m$, (i) $\sup_{\mvx \in \bR^m} p(\mvx) \leq C$; (ii) $p(\mvx) \geq C^{-1}$ for $\mvx \in \bR^{m}$ with $\|\mvx\| \leq C^{-1}$; (iii)  $p(\mvx) \leq C \exp(-\|\mvx\|/C)$ for $\mvx \in \bR^{m}$.
\end{lemma}
\begin{proof}
For (i) and (ii), see \citet[Theorem 5.14]{lovasz2007geometry}.  We focus on (iii) for $m \geq 2$, and note that the $m=1$ case follows from the same argument.  Let $\mvZ$ be an $m$-dimensional random vector with an \textit{arbitrary} isotropic, log-concave densities $p$. 

Fix any $\mvv \in \cS^{m-1}$. Let $\bU = [\mvu_1,\ldots,\mvu_{m-1}]$ be an $m$-by-$(m-1)$ matrix such that each column has length $1$ and is orthogonal to $\mvv$, and columns are orthogonal to each other, i.e., $\bU'\bU = \bI_{m-1}$ and $\bU'\mvv = \boldsymbol{0}_{m-1}$. 

Let $\tilde{\mvZ} = \bU'\mvZ$, and denote by $\tilde{p}$ its density.
Then $\tilde{p}$ is a log-concave density on $\bR^{m-1}$ \citep[see][Proposition 2.5]{samworth2018recent}. By definition, $\Exp[\tilde{\mvZ}] = \boldsymbol{0}_{m-1}$ and $\Cov(\tilde{\mvZ}) = \bI_{m-1}$; thus, $\tilde{p}$ is isotropic. 

By part (i) and (ii), there exists a constant $C_1 > 1$, depending only on $m$, such that
$\tilde{p}(\boldsymbol{0}_{m-1})  \leq C_1$, $C_1^{-1} \leq p(\boldsymbol{0}_{m}) \leq C_1$. Further, by a change-of-variable, for any $r > 0$,
\begin{align*}
    \tilde{p}(\boldsymbol{0}_{m-1}) = & \int_{\bR} p( \tau \mvv) d\tau \geq r \inf_{\tau \in [0,r]} p(\tau \mvv) \\
    \geq 
    & r \inf_{\tau \in [0,r]} p(r \mvv)^{\tau/r} p(\boldsymbol{0}_{m})^{1-\tau/r}
     \geq r \min\{  p(\boldsymbol{0}_{m}), p(r \mvv)\},
\end{align*}
where the second to the last inequality is due to the  log-concavity of $p$.
Thus for $r^* = 2 C_1^2$, $p(r^* \mvv) \leq 1/(2C_1) \leq p(\boldsymbol{0}_{m})/2$. Then again due to  the  log-concavity of $p$, for any $r > r^*$,
\begin{align*}
    p(r^* \mvv) \geq p(r \mvv)^{r^*/r} p(\boldsymbol{0}_{m})^{1-r^*/r} \;\; \Rightarrow \;\; p(r \mvv) \leq p(\boldsymbol{0}_{m}) (1/2)^{r/r^*}.
\end{align*}
Since $\mvv \in \cS^{d-1}$ is arbitrary, we have $p(\mvx) \leq C_1 2^{-\|\mvx\|/r^*}$ for $\|\mvx\| > r^*$. Since $p$ is also arbitrary, increasing $C$ if necessary, the proof is complete.
\end{proof}

Next, we consider projections of  ``high-dimensional" log-concave random vectors onto low dimensional spaces. 

\begin{lemma}\label{app:logconcave_proj}
Let $L > 1$ be a real number.
Let $d \geq 2$ be an integer, and $\mvZ$  an $\bR^d$ random vector with a log-concave density and the property that
$\|\Exp[\mvZ]\| \leq L$ and the eigenvalues of $\Cov(\mvZ)$ are between $[L^{-1},L]$. Let $\mvu, \mvw \in \cS^{d-1}$ be two unit vectors such that $\mvu'\mvw = 0$.  Denote by $p_{\mvu}$ the density of $\mvu'\mvZ$, and by $p_{\mvu,\mvw}$ the joint density of $(\mvu'\mvZ,\mvw'\mvZ)$.
There exists a constant $C > 0$, depending only on $L$ (in particular, not on $d$), such that    
\begin{enumerate}[label=(\roman*)]
    \item $p_{\mvu}(\tau) \leq C e^{-|\tau|/C}$ for $\tau \in \bR$;
    \item $p_{\mvu,\mvw}(\tau_1,\tau_2) \leq C e^{-\sqrt{\tau_1^2+\tau_2^2}/C}$ for $\tau_1,\tau_2 \in \bR$;
    \item if, in addition, $\Exp[\mvZ] = 0$, then $p_{\mvu,\mvw}(\tau_1,\tau_2) \geq C^{-1}$ if  $\max\{|\tau_1|,|\tau_2|\} \leq C^{-1}$.
\end{enumerate}
\end{lemma}
\begin{proof}
By \citet[Proposition 2.5]{samworth2018recent}, $p_{\mvu}$ and $p_{\mvu,\mvw}$ are log-concave densities on $\bR$ and $\bR^2$ respectively. Further, let $U = \mvu'\mvZ$ and $W = \mvw'\mvZ$. Then $U$ has density $p_{\mvu}$, and $(U,W)$ has the joint density $p_{\mvu,\mvw}$.

Since $\|\Exp[\mvZ]\| \leq L$ (resp. $=0$), $\max\{|\Exp[U]|, |\Exp[V]|\} \leq L$ (resp. $=0$). Further, since the eigenvalues of $\Cov(\mvZ)$ are between $[L^{-1},L]$, $\Var(U)$ and the eigenvalues of $\Cov(U,W)$ are between $[L^{-1},L]$. Then the proof is complete due to Lemma \ref{app:lemma_isotropic} and change-of-variables.
\end{proof}

\subsection{Elementary lemmas}\label{app:elementary_lemmas}
\begin{lemma}\label{lemma:aux_exp_inequal}
For any $a \geq 9$ and $b > 0$, if $t \geq a + 2b\log(a+b)$, then $a+b\log(t) \leq t$.
\end{lemma}
\begin{proof}
Define $t_0 = a + 2b\log(a+b)$, and $f(t) = t - a - b\log(t)$. Since $f'(t) = 1 - b/t$ and $f'(t_0) > 0$, it suffices to show that $f(t_0) \geq 0$. Note that $f(t_0) = 2b\log(a+b) - b\log(a+2b\log(a+b))
\geq 2b\log(a+b) - \max\{b\log(2a), b\log(4b\log(a+b))\}$. Since $a \geq 9$, we have $(a+b)^2 \geq \max\{2a, 4b\log(a+b)\}$, which completes the proof.
\end{proof}

\begin{lemma}\label{lemma:triangle}
Let $\mvu,\mvv \in \bR^{d}\setminus \{\boldsymbol{0}_{d}\}$. Then $   \left\| {\mvu}/{\|\mvu\|} -{\mvv}/{\|\mvv\|} \right\| \leq {2\|\mvu-\mvv\|}/{\|\mvu\|}$.

\end{lemma}
\begin{proof}
By the triangle inequality,
\begin{align*}
    \left\| \frac{\mvu}{\|\mvu\|} -\frac{\mvv}{\|\mvv\|} \right\| \leq    \left\| \frac{\mvu}{\|\mvu\|} -\frac{\mvv}{\|\mvu\|} \right\| +
     \left\| \frac{\mvv}{\|\mvu\|} -\frac{\mvv}{\|\mvv\|} \right\|
     \leq \frac{\|\mvu-\mvv\|}{\|\mvu\|} + \frac{|\|\mvu\|-\|\mvv\||}{\|\mvu\|}.
\end{align*}
Then the proof is complete by another application of the triangle inequality.
\end{proof}

\begin{lemma}\label{lemma:density}
Let $d \geq 1$ be an integer, and $C,m_{Z} > 0$.
Let $\mvZ\in \bR^d$ be a random vector that has a Lebesgue density $p$ such that  $\sup_{\mvz \in \bR^{d}} p(\mvz) \leq C$. Further, assume $\|\mvZ\| \leq m_{Z}$. Then there exists a constant $\tilde{C}>0$, depending only on $d, C, m_{Z}$, such that the Lebesgue density of $\mvu'\mvZ$ is bounded by $\tilde{C}$ for any $\mvu \in \cS^{d-1}$.
\end{lemma}
\begin{proof}
Fix $\mvu \in \cS^{d-1}$. There exist $\mvu_2,\ldots,\mvu_d$ in $\bR^d$ such that  
$\bU = [\mvu; \mvu_2;\ldots;\mvu_{d}]$
is an  orthonormal matrix. Then the density of $\mvu'\mvZ$ is: for $\tau \in \bR$, 
$
f_{\mvu}(\tau) =    \int_{\mvx \in \bR^{d-1}} p(\bU^{-1} [\tau,\mvx']') d\mvx.
$
Since $\|\mvZ\| \leq m_{Z}$ and $p(\cdot) \leq C$, we have
$$
    f_{\mvu}(\tau) \leq \int_{\mvx \in \cB_{d-1}(m_Z)} p(\bU^{-1} [\tau,\mvx']') d\mvx \leq C \text{Vol}(\cB_{d-1}(m_Z)),
$$
where $\cB_{d-1}(r) = \{\mvx \in \bR^{d-1}: \|\mvx\| \leq r\}$ is the Euclidean ball with radius $r$ in $\bR^{d-1}$, and
 $\text{Vol}(\cB_{d-1}(r))$ is its the Lebesgue volume. Since the upper bound does not depend on $\mvu$, the proof is complete.
\end{proof}

\section*{Funding}
Yanglei Song is supported by the Natural Sciences and Engineering Research Council of Canada (NSERC). This research is enabled in part by support provided by Compute Canada (\url{www.computecanada.ca}).

\bibliographystyle{imsart-nameyear} 
\bibliography{TrLinUCB}       
\end{document}